\documentclass[12pt]{article}
\usepackage{enumerate}
\usepackage[OT1]{fontenc}
\usepackage{amsmath,amssymb}
\usepackage{natbib}
\usepackage[usenames]{color}
\usepackage{bbm}
\usepackage{smile}
\usepackage{multirow}
\usepackage[breaklinks,colorlinks,
            linkcolor=red,
            anchorcolor=blue,
            citecolor=blue
            ]{hyperref}
\def\bA{{\mathbf A}}
\def\bX{{\mathbf X}}
\def\bx{{\mathbf x}}

\def\bv{{\mathbf v}}

\def\bZ{{\mathbf Z}}
\def\bTheta{{\boldsymbol \Theta}}
\def\bsigma{{\boldsymbol \sigma}}

\def\tr{\mathop{\text{tr}}\kern.2ex}

\def\tt{{\tilde t}}
\def\tbv{{\tilde {\mathbf v}}}

\def\tR{{\tilde R}}

\def\P{{\mathbb P}}
\def\E{{\mathbb E}}

\def\cA{{\mathcal{A}}}
\def\cB{{\mathcal{B}}}
\def\cC{{\mathcal{C}}}
\def\cD{{\mathcal{D}}}
\def\cE{{\mathcal{E}}}
\def\cK{{\mathcal{K}}}
\def\cL{{\mathcal{L}}}
\def\cN{{\mathcal{N}}}
\def\cO{{\mathcal{O}}}

\def\cR{{\mathcal{R}}}
\def\cS{{\mathcal{S}}}
\def\cT{{\mathcal{T}}}
\def\cU{{\mathcal{U}}}
\def\cV{{\mathcal{V}}}
\def\cW{{\mathcal{W}}}
\def\cX{{\mathcal{X}}}
\def\mfS{{\mathfrak{S}}}

\def\tr{{\text{Tr}}}

\def\mat{{\text{mat}}}
\def\th{{\text{th}}}
\def\true{\text{true}}

\def\sign{\mathop{\text{sign}}}
\def\supp{\mathop{\text{supp}}}

\def\md{\mathrm{d}}
\def\rank{\mathrm{rank}}

\long\def\comment#1{}

\def\vec{\mathop{\text{vec}}}
\def\tr{\mathop{\text{Tr}}}

\newcommand{\bel}{\begin{eqnarray}\label}
\newcommand{\eel}{\end{eqnarray}}
\newcommand{\bes}{\begin{eqnarray*}}
\newcommand{\ees}{\end{eqnarray*}}

\def\betah{{\widehat \beta}}
\def\bThetah{{\widehat \bTheta}}
\def\Sighat{{\widehat \Sigma}}
\def\Thetahat{{\widehat {\boldsymbol \Theta}}}
\def\Deltahat{{\widehat {\boldsymbol \Delta}}}

\newcommand{\la}{\langle}
\newcommand{\ra}{\rangle}

\let\emptyset\varnothing

\usepackage{mathrsfs}

\usepackage{fullpage}

\def \rg {{\rm rg}}
\def \RG {{\rm RG}}

\usepackage{hyperref}
\usepackage[english]{babel}
\usepackage{breakcites}
\usepackage{microtype}
\def\##1\#{\begin{align}#1\end{align}}
\def\$#1\${\begin{align*}#1\end{align*}}
\usepackage{enumitem}

\usepackage{setspace} 

\begin{document}

\title{\LARGE Provably Efficient High-Dimensional Bandit Learning with Batched Feedbacks}
\author
{Jianqing Fan \qquad Zhaoran Wang \qquad Zhuoran Yang \qquad Chenlu Ye
}
\date{}
\maketitle

\def\r#1{\textcolor{red}{\bf #1}}
\def\b#1{\textcolor{blue}{\bf #1}}

\setlength{\abovedisplayskip}{6pt}
\setlength{\belowdisplayskip}{6pt}

\begin{abstract}
\begin{spacing}{1.35}

We study high-dimensional multi-armed contextual bandits with batched feedback where the $T$ steps of online  interactions are divided into $L$ batches. 
In specific, each batch collects data according to a policy that depends on previous batches
and the rewards are revealed only at the end of the batch. 
Such a feedback structure is popular in applications such as personalized medicine and online advertisement, where the online data often do not arrive in a fully serial manner. 
We consider high-dimensional and linear settings where the reward function of the bandit model admits either a  sparse or low-rank structure and ask how small a number of batches are needed for a comparable performance with fully dynamic data in which $L = T$. For these  settings, we design a  provably sample-efficient algorithm  which achieves a $   \mathcal{\tilde O}(s_0^2 \log^2  T)$ regret in the sparse case and  $ \mathcal{\tilde O} ( r ^2  \log^2  T)$ regret in the low-rank case, using only
$L = \mathcal{O}( \log T)$ batches. Here $s_0$ and $r$ are  the sparsity and rank of the reward  parameter in sparse and low-rank cases, respectively, and 
$ \mathcal{\tilde O}(\cdot)$ omits logarithmic factors involving the feature  dimensions. 
In other words, our algorithm achieves regret bounds comparable to those in fully sequential setting with only  $\mathcal{O}( \log T)$ batches. 
Our algorithm features a novel batch allocation method that adjusts the batch sizes according to the estimation accuracy within each batch and cumulative regret.  Furthermore, we also conduct experiments with synthetic and real-world data to validate our theory.

\end{spacing}


\end{abstract}
\section{Introduction}
With the growing availability of user-specific data and the increasing demand for personalized service, high-dimensional contextual bandits have received tremendous attention in various industries, including healthcare \citep{ameko2020offline}, individual content recommendations \citep{li2010contextual}, dynamic pricing \citep{qiang2016dynamic,fan2023policy}, and talent searching \citep{geyik2018session}. Furthermore, data dimensions are rapidly increasing in the era of big data. For example, patients' information in healthcare includes thousands of components, such as medical history records, clinical test results, and biomarker profiles. Sometimes, this information has to be framed into a low-rank matrix model, like X-ray images and genetic biomarker profiles. These problems are generally high-dimensional but often represented by parameters with sparse structures, such as vectors with low effective dimensions or matrices with low ranks. 


In the face of high-dimensional data, traditional sequential bandits encounter a problem of inefficiency and unrealistic feedback acquisition. Because the feedback may not be revealed instantly, making optimizations under high dimensions consumes significant time and resources. For example, in clinical trials, it takes a period to judge the effectiveness of treatments. Still, new patients usually need to be treated urgently and cannot wait for the previous patient to finish the trial. Similarly, in individual music recommendations, a vast number of users log in simultaneously and are impatient, so the system must interact with them simultaneously. Additionally, there are some applications where feedback is naturally divided into batches, such as in vaccine treatment, where vaccines are allocated into batches according to demand, and each batch corresponds to the vaccine given to people in each period.

Motivated by these problems, we study high-dimensional contextual bandits with batched feedback, formulated as high-dimensional linear contextual bandits or low-rank matrix bandits with sparse parameters, respectively. ``Batched feedback" means that feedback is not revealed instantly; instead, they are divided into several batches, and rewards in the same batch are revealed simultaneously at the end of the batch. Hence, the agent's knowledge is only updated once per batch. For the batched bandit problem, the agent 
needs to decide how to allocate $T$ rounds of interactions into $L$ batches and how to take actions in each batch.

Our problem involves three coupled challenges – (a) the exploration-exploitation tradeoff, (b) batched feedback, and (c) estimation in high-dimensional statistical models. Challenge (a) arises from the need to balance gathering enough samples to estimate the model accurately with the need to choose greedily based on the limited information available, known as the exploration-exploitation tradeoff. More importantly, since our feedback is divided into batches, we have to balance exploration and exploitation even before observing the contexts (data in the next batch). Challenge (b) is unique to our problem since we need to design the batch allocation policy so that the order of the expected regret will not be much affected by the batched feedback, thus making the decision process more challenging than the standard LASSO bandit. Finally, learning the reward functions for high-dimensional covariates naturally involves Challenge (c).

To address these challenges, we propose a novel algorithm that incorporates three key components: (a) a batch allocation policy that uses gradually decelerated-increasing batches, (b) an exploration strategy that employs $\epsilon$-decay forced sampling and arm-elimination, and (c) high-dimensional estimation techniques such as LASSO and nuclear-norm regularization. Specifically, for (a), we propose a novel technique that guarantees sufficient i.i.d. samples for estimation convergence and develop an analysis to demonstrate that the regret bound under batched scenarios almost matches that under sequential ones. Our approach is based on the intuition that our estimations become more accurate as time passes so that the updating frequency can be reduced. Additionally, as the convergence rate of estimations slows down, the growth rate of batch sizes should be decelerated accordingly. For (b), we introduce $\epsilon$-decay forced sampling due to the lack of knowledge about the actions at the beginning, so more importance should be attached to exploration. As more knowledge is obtained, the proportion of exploration decays. Moreover, arm elimination excludes poorly performing (sub-optimal) actions and identifies the optimal policy. For (c), under high-dimensional linear contextual bandits, the estimator will converge if a constant fraction of the non-i.i.d. sample set is i.i.d., as proved in \cite{bastani2015online}. In the low-rank matrix bandit scenario, to the best of our knowledge, we first establish a tail inequality for a low-rank estimator with a subset of enough i.i.d. samples. Notably, the proof is significantly different and novel compared with the one for the LASSO estimator \citep{bastani2015online}, where we show matrix concentration inequalities.

We theoretically demonstrate that for the sparse cases with $T$ users and $d$ covariate dimensions, only $L=\Theta(\log T)$ batches are sufficient to achieve a $\cO(\log d (\log^2 T+\log d))$ regret. Notably, the batched bandit algorithm achieves nearly the same order of regret bounds as the sequential one (LASSO bandit with $L=T$), which means that our algorithm attains a desired bound in the batched version. For a low-rank case with $T$ users and $d_1\times d_2$ matrix dimensions, the algorithm can guarantee the regret bound on the order of $\cO(\log(d_1d_2)(\log^2 T+\log(d_1d_2)))$ with only $L=\Theta(\log T)$ batches. Notably, we believe that our work is the first in a low-rank matrix setting that bounds the expected reward by the logarithmic dependence both on the sample size and on the covariate matrix dimension.

\subsection{Main Contributions}

Our contributions are three-fold. 
\begin{itemize}
    \item For high-dimensional linear bandits, we propose a computationally efficient algorithm that is based on a novel batch allocation policy. Our algorithm adjusts each batch's length according to the estimation error's convergence rate. Additionally, we show that our policy ensures sufficient i.i.d. samples during the estimation process. Our algorithm guarantees a regret bound of $\cO(s_0^2\log d(\log^2 T + \log d))$, where $s_0 \ll d$ is the number of effective dimensions and $T$ is the total number of users. This bound almost matches that in the sequential case ($\cO(s_0^2(\log T + \log d)^2)$ \citep{bastani2015online}).
    \item When extending to the low-rank matrix bandit problem, we incorporate the batch allocation policy into the algorithm. Furthermore, we demonstrate that the low-rank estimator will converge if a constant fraction of the samples are independently and identically distributed. In terms of regret bounds, the algorithm guarantees an $\cO(r^2\log(d_1d_2)(\log^2 T + \log(d_1d_2)))$ regret bound, where $r$ is the upper bound of the effective matrix rank and $d_1\times d_2$ is the dimension of the matrix. To the best of our knowledge, this is the first bound that achieves a logarithmic square dependence on both the sample size and the matrix dimensions.
    \item We conduct experiments on synthetic and real datasets to validate the performance of our algorithms. Particularly, we compare our algorithm under certain batches with the sequential version ($L=T$). Our results show that the performance of our algorithms nearly approximates the sequential one.
\end{itemize}

\subsection{Related Work}
Our research on batched high-dimensional bandit problems is built upon the fields of contextual bandits, high-dimensional statistics, and batched literature. This section will provide a brief overview of these three areas.

The contextual bandit problem is characterized by the exploration-exploitation tradeoff, which arises from the bandit feedback setting where only the feedback from the chosen decision is accessible to the agent. There are two main approaches to address the tradeoff. The first approach is to explore and exploit simultaneously by comparing the confidence bounds for all the policies, which are represented by UCB-type algorithms \citep{auer2002using,dani2008stochastic,rusmevichientong2010linearly,abbasi2013online,deshpande2012linear}. The second approach is to arrange pure-exploration steps, which was first proposed by \citet{goldenshluger2013linear}. They introduce a forced sampling method to generate enough i.i.d. samples and prove that the OLS bandit algorithm attains a regret bound logarithmically dependent on the sample size. However, their algorithm does not efficiently apply to high-dimensional settings.

In terms of high-dimensional contextual linear bandits, following the i.i.d. covariate setting and the approaches in \citet{goldenshluger2013linear}, \citet{bastani2015online} propose the LASSO bandit under the sparsity condition $|\beta|_0\leq s_0$. With a tighter regret analysis and convergence on LASSO estimators, their algorithm achieves a poly-logarithmic dependence on the sample size and covariate dimension: $\cO(s_0^2(\log T+\log d)^2)$. Additionally, by adopting techniques from LASSO bandit and using the MCP method, \citet{wang2018online} propose a $\epsilon$-decay sampling method. The regret bound yielded by their G-MCP-Bandit algorithm is optimal on the sample size $\cO(\log T)$ and covariate dimension $\cO(\log d)$.

Following LASSO bandit, much bandit literature has focused on the exploration-free version of high-dimensional contextual linear bandits. By requiring more diversity for the covariate distribution (relaxed symmetry condition),  \citet{ariu2020thresholded,oh2021sparsity} neither proceed with exploration nor require knowledge of the sparsity parameter $s_0$ without the marginal condition. Thus, their algorithms are parameter-free. Furthermore, \citet{bastani2021mostly} propose the Greedy-First algorithm that determines whether the greedy policy fails with a hypothesis test on the covariates and rewards.  Additionally, there is an emerging body of literature on
high-dimensional contextual linear bandit problems that propose algorithms to tackle various scenarios \citep{kim2019doubly,hao2020high,oh2021sparsity, ariu2022thresholded}. Further effort has been made to investigate the challenges in sparse linear Markov decision process (MDP) by \citep{hao2021online}.

In low-rank matrix scenarios, significant progress has been made recently. In statistics, \citet{negahban2011estimation} derive estimation error bounds for the trace regression model under nuclear norm penalization. Then, \citet{fan2021shrinkage} analyze the robust low-rank matrix recovery for heavy-tailed data by developing a robust quadratic loss function. When utilizing the low-rank structure in the bandit problem, \citet{jun2019bilinear} and its generalization \citet{lu2021low} propose online computation algorithms that bound the expected regret with $\tilde{\cO}((d_1+d_2)^{3/2}\sqrt{rT})$, where $\tilde{\cO}$ omits poly-logarithmic factors of the covariate matrix dimension $d_1, d_2$, matrix rank $r$, and sample size $T$. Additionally, \citet{li2022simple} provide a novel analysis technique by applying the matrix Bernstein inequality. However, due to some gaps, their approach cannot lead to regret depending logarithmically on $d_1, d_2$. In this paper, we introduce the statistical analysis of low-rank estimators and prove the convergence result for non-i.i.d. samples. With more rigorous analysis, we achieve an expected regret of $\cO(r^2\log(d_1d_2)(\log^2 T+\log(d_1d_2)))$.

Ultimately, our work is related to batched settings for high-dimensional sparse linear contextual bandits \citep{ren2020dynamic}. The field of bandits and Markov decision processes with batched feedback has been rapidly developing in recent years \citep{perchet2016batched,gao2019batched,han2020sequential,wang2021provably,karbasi2021parallelizing} due to their broad applications in real-world problems. Specifically, \citet{wang2020online} study batched high-dimensional sparse bandits with similar settings to our work. However, they set the batch size as fixed, so the regret bound is multiplied by the batch size number. In contrast, we design a fine-grained grid according to the estimation error bound. Meanwhile, \cite{kalkanli2021batched} propose batched Thompson sampling for the multi-armed bandit (MAB) via a dynamic batch allocation. To the best of our knowledge, our algorithm is the first to attain a regret bound in almost the same order as the LASSO bandit algorithm.

\subsection{Notation}\label{notation}
For any integer $n$, let $[n]$ be the set $\{1, \ldots, n\}$. For any vector $\beta\in\RR^d$ and index subset $\mathcal I \subset [d]$, denote by $\beta_{\mathit{I}}\in\RR^d$ the vector composed of nonzero elements of $\beta$ and by $\supp(\beta)$ the index set of nonzero entries of $\beta$. For any data matrix $\mathbf{Z} \in \RR^{n \times d}$, let $\hat{\Sigma}(\mathbf{Z})=\mathbf{Z}^{\mathrm{T}}\mathbf{Z}/n$ refer to its sample covariance matrix. For any subset $\mathcal{A} \subset [n]$, let $\hat{\Sigma}(\mathcal{A})$ represent $\hat{\Sigma}(\mathbf{Z}(\mathcal{A}))$. For $x,y \in \RR$, denote $\max(x,y)$ and $\min(x,y)$ by $x \vee y$ and $x \wedge y$ respectively. Additionally, let $\RR^+$ and $\ZZ^+$ represent positive real numbers and positive integers, and $\RR_{\succeq0}^{d\times d}$ refer to the set of positive semidefinite matrices of size d by d. For any real-valued random variable $z$, we claim that it is $\sigma$-subgaussian if $\mathbb{E}(e^{tz}) \le e^{\sigma^2t^2/2}$ for every $t \in \RR$, which implies that $\mathbb{E}(z)=0$ and $\mathrm{Var}(z) \le \sigma^2$.

For matrix models, use $\RR^{d_1 \times d_2}$ to denote the space of $d_1$-by-$d_2$ real matrices. For any matrix $\bX \in \RR^{d_1 \times d_2}$, define $\| \bX \|_{\op},~\| \bX \|_N,~\| \bX \|_F,~\| \bX \|_1$ and $\| \bX \|_{\max}$ to be its $l_2$ operator norm, nuclear norm, Frobenius norm, the sum of the absolute matrix values and elementwise max norm respectively. The vectorized version of $\bX$ is denoted as $\Vec{\bX}= ( \bX_1^{\mathrm{T}}, \bX_2^{\mathrm{T}}, \ldots, \bX_{d_2}^{\mathrm{T}})^{\mathrm{T}}$, where $\bX_j$ is the $j^{\th}$ column of $\bX$. Let $\mat(\bX)$ denote the $d_1$-by-$d_2$ matrix constructed by $\bX$, where $\left( x_{(j-1)d_1+1}, \ldots, x_{jd_1}\right)^{\mathrm{T}}$ is the $j^{\th}$ column of $\mat(\bx)$. For any time index subset $\mathcal{S}^{'} \subset [T]$, let $\mathbf{X}(\mathcal{S}^{'})$ be the $\,|\, \mathcal{S}^{'}  \,|\, \times d $ submatrix whose rows are $X_t$ for all the $t \in \mathcal{S}^{'} $. Given two matrices $\mathbf A, \mathbf B \in \RR^{d_1 \times d_2}$, we use $ \tr(\bf A^{\mathrm{T}} \bf B)$ to denote the matrix inner product $\left< \bf A, \bf B\right>$, where $\tr$ is the trace operator. 

\section{Preliminaries}
This section will describe the standard problem formulation of contextual bandits for high-dimensional linear models with batched feedback. 
\subsection{Contextual Bandit}\label{s:ContextualBandit}
Consider an arrival process with $T$ time steps. At each step $t \in [T]$, a new user (such as a patient) arrives, and the agent observes the covariate vector $X_t$ containing all available personal information (such as medical records, clinical test reports, and other useful observations). The observed sequence of covariates $\{X_t\}_{t\in [T]}$ is drawn i.i.d. from an unknown distribution $\mathcal{P}_X$ over a deterministic set $\mathcal{X}\subseteq\RR^d$. The agent then has access to $K$ arms and aims to optimize the total reward produced during the process. Let $\pi_t$ be the chosen arm at time $t$. If the agent selects action $k_t = \pi_t$, she will receive the reward $R_{t, \pi_t}$.

The agent's policy is denoted as $\pi=\{\pi_t\}_{t\in [T]}$, where $\pi_t \in [K]$ depends on the user's covariate and the information that the agent attains before the current batch. To benchmark the performance of policy $\pi$, we introduce the optimal policy $\pi_t^{\ast}= \argmax_{k \in [K]}{R_{t,k}}$, which attains the highest reward given the true parameters. Performance of the policy $\pi$ is measured by the notion of \emph{regret}, which is defined as the cumulative expected suboptimality compared with the best policy:
\begin{equation}\label{regret}
\RG(T,L)= \sum_{t=1}^{T}\mathbb{E}\bigl( R_{t,\pi_t^{\ast}} - R_{t,\pi_t} \bigr).
\end{equation}
The goal is to find the policy $\pi$ that minimizes the cumulative regret up to time $T$.

\subsection{Linear Models}\label{s:LinearModels}
Suppose that each arm $k\in[K]$ has an unknown parameter $\beta_k^{\true} \in \RR^d$. If we choose arm $k\in[K]$ at time $t$, the reward $r_{t,k} \in \RR$ is generated as
$$
r_{t,k} = \la x_t, \beta_k^{\true} \ra+\varepsilon_{t},
$$
where the noise $\varepsilon_{t}$ is a $\sigma$-subgaussian variable. Then, we consider two linear models: the high-dimensional sparse model and the low-rank matrix model.

\vspace{4pt}
\noindent
{\bf High-Dimensional Sparse Model.} The reward for arm $k\in[K]$ pulled at time $t$ is formulated as
$$
r_{t,k} = x_t^{\mathrm{T}}\beta_k^{\true}+\varepsilon_{t}.
$$
Now we have the dataset $\cD = (\bX, \br)$, where the rows of design matrix $\bX \in \RR^{T \times (d+1)}$ are the vectors $x_t$ and  treatment arm and the components of response vector $\br$ are the rewards $r_{t,\pi_t}$ (for all $t \in [T]$). We also denote the noise vector $(\varepsilon_1,\ldots,\varepsilon_T)$ by $\be$. By putting $\beta_k$'s as a long vector $\beta$, and appropriately expanding $\bX$ by padding zero according to the treatment arms, and still denoting the rearranged matrix as $\bX$, we can write the model as $\br = \bX \bbeta + \be$.  Thus, for a regularization parameter $\lambda \ge 0$, the LASSO estimator \citep{tibshirani1996regression} is determined by solving
\#\label{LASSO}
\hat \beta(\cD,\lambda) = \arg\min_{\beta} \left\{ \frac{\| \br-\mathbf{X}\beta\|_2^2}{T} + \lambda \| \beta \|_1 \right\}.
\#

\vspace{4pt}
\noindent
{\bf Low-Rank Matrix Model.} The reward for arm $k\in[K]$ pulled at time $t$ satisfies
$$
r_{t,k} = \left< x_t, \bTheta_k^{\true}\right> + \varepsilon_{t},
$$
where $\bTheta_k^{\true} \in \RR^{d_1 \times d_2}$ is the true coefficient matrix. Given the sequence of covariate matrices $\{ \bX_t \in \RR^{d_1 \times d_2}\}_{t\in [T]} $ and the rewards $\{ r_{t,k} \in \RR\}_{t\in [T]}$, we can represent this model in a more compact form by defining the observation operator $\mathfrak X_T:\RR^{d_1 \times d_2}\rightarrow\RR^{\mathrm{T}}$ with elements $[\mathfrak X_T(\bTheta)]_t=\left< \bX_t, \bTheta\right>$ and writing
\$
\br=\mathfrak X_T(\bTheta_k^{\true})+\be,
\$
where $\br$ and $\be$ are the $T$-dimensional vectors with components $r_{t,\pi_t}$ and $\varepsilon_t$, respectively. Generally, the parameter matrices $\bTheta_k^{\true}$ are low-rank or well approximated by a low-rank matrix. Hence, we estimate $\bTheta^{\true}$ by solving
\#\label{eq0504}
\Thetahat(\cD, \lambda) = \arg\min_{\bTheta \in \RR^{d_1 \times d_2}} \left\{\frac{1}{T}\| r-\mathfrak X_T(\bTheta)\|_{2}^2 + \lambda \| \bTheta \|_N\right\},
\#
where $\lambda \ge 0$ is a regularization parameter.
The sum of its singular values gives the nuclear norm of $\bTheta$:
\$
\| \bTheta \|_N=\sum_{j=1}^{d_1\wedge d_2}\sigma_j(\bTheta).
\$

\subsection{Batched Feedback Structure}\label{s:BatchedFeedbackStructure}
\begin{figure}[t]
    \centering
    \includegraphics[width=7in]{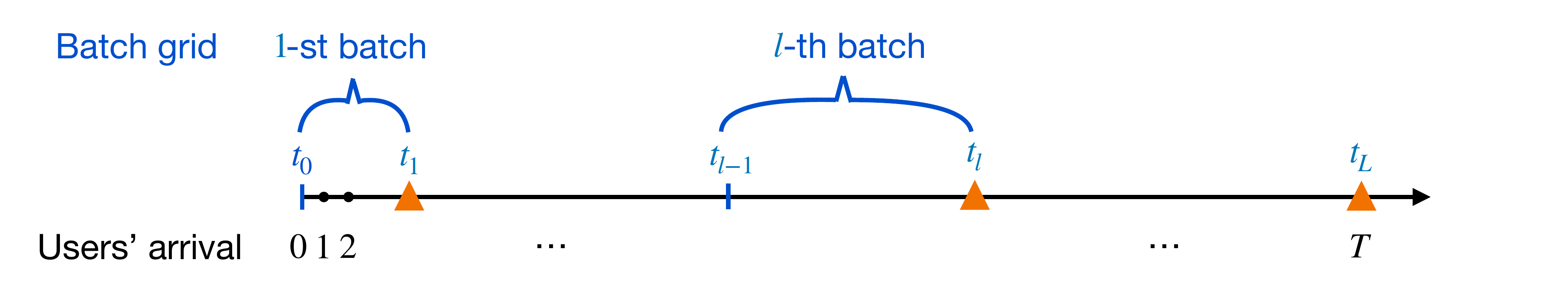}
    \caption{A batched online learning and decision-making process, where $T$ arriving users are divided into $L$ batches. We use blue color to denote the grids. For each batch, the rewards are only revealed at the end of the batch, which is represented by the orange triangle.}
    \label{f:batch_bandit}
\end{figure}
The concept of batching refers to dividing time sequences into groups. During each batch, the agent interacts with a group of users simultaneously. Once all the decisions in this group have been made, the feedback belonging to the same batch is revealed. To illustrate this interaction protocol, we present Figure \ref{f:batch_bandit}, where the total time steps $T$ are divided into $L$ batches by the agent before the interaction starts, and the rewards are revealed at each division point $t_1,t_2,\ldots,t_L$. The division points are called grid $\mathcal{T}=\{ t_1,t_2,\ldots, t_L\},$ with $0<t_1<t_2<\ldots<t_L=T$ (batch $l$ is composed of time steps $t_{l-1}+1$ to $t_l$). Note that if $L=T$, the size of each batch is $1$, and the bandit becomes fully sequential. If $L < T$, at the beginning of each batch, the agent needs to specify policies for the batch without observing contexts. The agent makes decisions only according to the feedback from previous batches. To achieve the final goal of maximizing the expected cumulative reward $\sum_{l=1}^L\sum_{t=t_{l-1}+1}^{t_l}\mathbb{E}(R_{t,\pi_t})$, the agent faces two tasks before the process. One task is to decide how to divide the feedback into batches, called the batch allocation policy. The other task is to specify how to take action in each batch.

\section{Algorithm}\label{algo}
In this section, we propose algorithms to overcome three online learning and decision-making challenges: the exploration-exploitation tradeoff, batched feedback, and high-dimensional estimation with non-i.i.d. data. These challenges are coupled together for the following reasons.

The first challenge, the exploration-exploitation tradeoff, arises from the fact that the agent can only observe the reward $r_{t,\pi_t}$ corresponding to the chosen arm $\pi_t$. Hence, it is likely to mistake suboptimal arms for the optimal choice with limited samples. Because the greedy decision-making can lead to suboptimal decisions, we need enough i.i.d. samples to guarantee the convergence of estimators for each arm $k\in[K]$. On the other hand, excessive forced samples (collected with a random exploration) can also sacrifice cumulative regret as they don't use collected data to make decisions, particularly when the estimation becomes more accurate. Additionally, at the beginning of each batch $l\in[L]$, the agent must allocate the exploration and exploitation for time steps $t_{l-1}+1,\ldots,t_l$ before observing the covariates $\{x_t\}_{t=t_{l-1}+1}^{t_l}$, making the tradeoff more difficult. The second challenge, batched feedbacks, occurs because in batch $l$, the rewards $\{r_{t,\pi_t}\}_{t=t_{l-1}+1}^{t_l}$ are not available to the agent until the decision is made for time step $t_l$. Therefore, for the covariates $\{x_t\}_{t=t_{l-1}+1}^{t_l}$, the agent have to make decisions according to the estimates $\betah(\cD_{k,l-1},\lambda_{l-1})$, where $\cD_{k,l-1}$ represents the data $(\bX,\br)$ collected in the first $l-1$ batches for arms $k\in[K]$. The third challenge is high-dimensional estimations based on non-i.i.d. data, which requires a carefully designed algorithm and analysis to achieve the desired convergence result.

We employ several techniques to address the challenges faced in this problem. The first challenge is overcome by using forced sampling and arm elimination. To balance the exploration and exploitation, we introduce the $\epsilon$-decay forced sampling method to obtain i.i.d. samples via random explorations \citep{wang2018online}, which not only ensures sufficient i.i.d. samples but also gradually reduces the proportion of forced sampling.  For arm elimination, we use a two-stage sampling procedure that leverages the gap between optimal and suboptimal arms (as defined in Assumption \ref{armoptimality}), thus increasing the chances of selecting the optimal arm. 
The second challenge is addressed by using a fixed batch allocation policy where the batch sizes satisfy a specific equation. This policy is formulated based on two intuitions: firstly, we want the batch size $t_{l}-t_{l-1}$ to adapt to changes in estimation accuracy; secondly, we aim to roughly balance the regret incurred in each batch and achieve a logarithmic dependence on the batch number $L$, as we will illustrate below. Finally, the third challenge is addressed by using LASSO or nuclear-norm regularized regression. We defer the analysis to $\S$\ref{theory}.


\subsection{Batched High-Dimensional Sparse Bandit Algorithm}\label{batchLASSOalgo}
For batched high-dimensional linear bandits, we propose the batched high-dimensional sparse bandit algorithm, which is called the batched sparse bandit for simplicity. For convenience, we denote whole-sample and forced-sample (obtained from random explorations) sets for arm $k$ up to the end of batch $l$ by $\cW_{k,l}$ and $\cR_{k,l}$, respectively. Let $\betah(\cW_{k,l}, \lambda_{2,l})$ and $\betah(\cR_{k,l}, \lambda_1)$ be estimators trained on $\cW_{k,l}$ and $\cR_{k,l}$.

\vspace{4pt}
\noindent
{\bf The execution of the batched sparse bandit:} Before the arrival process, given the number of users $T$ and batches $L$, the agent designs the grid and initializes the parameters. Then, in each batch $l$, the decisions are made only based on the estimators in the previous $l-1$ batches. For each time step $t$ in the current batch, a user comes with an observable covariate vector $x_t$. After recording the user's covariate $x_t$, the agent draws a binary random variable $\mathcal{D}_t$, where $\mathcal{D}_t=1$ with probability $\min\{1, t_0/t\}$, for a given $t_0 > 0$. There are two situations:

\begin{itemize}
\item If $\mathcal{D}_t=1$, the agent will randomly implement a decision $\pi_t$ with equal probability. Then, she will update the forced-sample dataset $\mathcal{R}_{\pi_t,l}$ by including $\{x_t, 0\}$ (as the reward is temporarily unobserved but will be updated when available).
\item If $\mathcal{D}_t=0$, the agent will execute the two-stage decision procedure in Algorithm \ref{algo101}. In the first stage, she will construct a decision candidate set containing the arms yielding rewards (with the forced-sample estimators $\{\betah(\cR_{k,l-1}, \lambda_1)\}_{k\in[K]}$) within $h/2$ of the maximum possible value. If the set $\Pi_t$ has only one component, this component will become the optimal decision; otherwise, the precise stage will proceed based on the whole-sample estimators $\{\betah(\cW_{k,l-1}, \lambda_{2,l-1})\}_{k\in \Pi_t}$. The agent will select the arm that generates the highest reward in the decision candidate set $\Pi_t$.
\end{itemize}
Then, the agent updates the whole-sample dataset $\mathcal{W}_{\pi_t,l}$ by appending the datapoint $\{x_t, 0\}$. When all the selections in batch $l$ are completed, she will observe the user's response to the decision $\pi_t$ and fill the reward $r_t$ in the corresponding datapoints for all $t$ in batch $l$. Meanwhile, the agent will update the regularization parameter $\lambda_{2,l}$, forced-sample estimators $\betah(\cR_{k,l}, \lambda_1)$ and whole-sample estimators $\betah(\cW_{k,l}, \lambda_{2,l})$ via LASSO for $k\in[K]$, based on the sample sets $\mathcal{R}_{k,l}$ and $\mathcal{W}_{k,l}$. The detailed procedure is presented in Algorithm \ref{algo1}.

\begin{algorithm}[t]
\caption{Batched High-dimensional Sparse Bandit}\label{algo1}
\begin{spacing}{1.35}
\begin{algorithmic}[1]
\STATE \textbf{Require:} Input parameters time $T$, number of batches $L$, $t_0$, $h$, $\lambda_1$, $\lambda_{2,0}$ and $a$ 
\STATE \hspace{0.15in} Choose grid $\mathcal{T}=\{ t_1,\ldots,t_L\}$ according to \eqref{eqgrid}

\STATE \hspace{0.15in} Initialize $\betah(\cR_{k,0}, \lambda_1) = \betah(\cW_{k,0}, \lambda_{2,0}) = \mathbf{0}$, and $\mathcal{R}_{k,0}=\mathcal{W}_{k,0}= \emptyset$ for all $k \in [K]$
\STATE \hspace{0.15in} \textbf{For} batch $l=1,\ldots,L$,~ \textbf{do}
\STATE \hspace{0.3in} $\mathcal{R}_{k,l}=\mathcal{R}_{k,l-1}$, $\mathcal{W}_{k,l}=\mathcal{W}_{k,l-1}$ for all $k \in [K]$
\STATE \hspace{0.3in} \textbf{For} time $t=t_{l-1}+1,\ldots,t_l$ $\textbf{do}$
\STATE \hspace{0.45in} Observe $x_t$
\STATE \hspace{0.45in} Draw a binary random variable $\mathcal{D}_t$, where $\mathcal{D}_t=1$ with probability $\min\{1, t_0/t\}$
\STATE \hspace{0.45in} $\mathbf{If}$ $\mathcal{D}_t=1$
\STATE \hspace{0.6in} Assign $\pi_t$ to a random decision $k \in [K]$ with probability $\mathbb{P}(\pi_t=k)=1/K$
\STATE \hspace{0.6in} Update $\mathcal{R}_{\pi_t,l}=\mathcal{R}_{\pi_t,l} \cup \{ X_t,0\}$
\STATE \hspace{0.45in} $\mathbf{Else}$
\STATE \hspace{0.6in} Execute two-stage sampling procedure in Algorithm \ref{algo101}
\STATE \hspace{0.45in} $\mathbf{End~If}$
\STATE \hspace{0.45in} Execute decision $\pi_t$, and update $\mathcal{W}_{\pi_t,l}=\mathcal{W}_{k,l} \cup \{ X_t,0\}$  
\STATE \hspace{0.3in} $\mathbf{End~For}$
\STATE \hspace{0.3in} Replace temporary rewards $0$ in $\mathcal{R}_{k,l}, \mathcal{W}_{k,l}$ by the actual rewards for $k\in[K]$.
\STATE \hspace{0.3in} Compute regularization parameter $\lambda_{2,l}=\lambda_{2,0}\sqrt{(\log t_l+\log d)/ t_l}$
\STATE \hspace{0.3in} Observe rewards and update LASSO estimators $\betah(\cR_{k,l}, \lambda_1)$ and $\betah(\cW_{k,l}, \lambda_{2,l})$
\STATE \hspace{0.15in} $\mathbf{End~For}$
\end{algorithmic}
\end{spacing}
\end{algorithm}

\begin{algorithm}[t]\caption{Two-Stage Sampling Procedure For Linear Model}\label{algo101}
\begin{spacing}{1.35}
\begin{algorithmic}[1]
\STATE \textbf{Screening~stage:}
\STATE \hspace{0.15in} Construct the decision candidate set
\STATE \hspace{0.3in} $\Pi_t=\bigl\{ k: X_t^{\mathrm{T}} \betah(\cR_{k,l-1}, \lambda_1) \ge \max_{j \in [K]} X_t^{\mathrm{T}} \betah(\cR_{j,l-1}, \lambda_1) -h/2\bigr\}$
\STATE \textbf{Selection~stage:}
\STATE \hspace{0.15in}$\mathbf{If}$ $\Pi_t$ is a singleton 
\STATE \hspace{0.3in} Set $\pi_t=\Pi_t$
\STATE \hspace{0.15in} $\mathbf{Else}$
\STATE \hspace{0.3in} Set $\pi_t=\argmax_{k \in \Pi_t} \{X_t^{\mathrm{T}} \betah(\cW_{k,l-1}, \lambda_{2,l-1})\}$
\STATE \hspace{0.15in} $\mathbf{End~If}$
\end{algorithmic}
\end{spacing}
\end{algorithm}

\vspace{4pt}
\noindent
{\bf Batch Allocation Policy:} Our grid design is motivated by two intuitions - (a) the batch size should be adjusted according to the estimation accuracy; (b) the batch size can adjust the regret of each batch so that the cumulative regret has a logarithmic dependence on the batch number. From (a), we will find from Proposition \ref{prop5} that the estimation error at the end of batch $l$ is bounded by $\sqrt{\log t_l/t_l}$, so the regret at time step $t$ in batch $l+1$ attains the $\log t_l/t_l$ bound (as we will show in appendix). To balance the regret for each batch, we incorporate the term $t_l/\log t_l$ into the batch size. From (b), only if the expected regret in batch $l$ is bounded by an order of $1/l$ can the cumulative regret $\sum_{l=1}^L \E(R_l)$ (controlled by an integral $\int_1^L \E(R_l) \md l$) be bounded by a polynomial of $\log L$. Hence, the regret will be finally bounded by a polynomial of $\log T$, which almost equals the regret bound under the sequential setting \citep{bastani2015online}. Suppose the number of batches $L=\Omega(\log T)$. The grid is as follows:



\begin{equation}\label{eqgrid}
\begin{aligned}
&t_l = \Big\lfloor \Big(\frac{a}{(l-1)\log t_{l-1}}+1\Big)t_{l-1}\Big\rfloor,l=2,\ldots,L-1,\\
&t_1= 2,~t_L=T,
\end{aligned}
\end{equation}
where $a=c\log^2 T/\log L$ is a predetermined parameter satisfying that $t_L=T$, and $c$ is an absolue constant. The setting of $t_1$ is to get feedback more frequently at the beginning.


\vspace{4pt}
\noindent
{\bf $\varepsilon$-decay Forced Sampling Method:} To ensure enough i.i.d. samples by the online learning and arm-selection process, we adopt the $\varepsilon$-decay forced sampling method from \citet{wang2018online}, where random selections are made with decreasing probability $\min\{1, t_0/t\}$. This decreasing probability is applied to avoid exploring too much to control the cumulative regret. Furthermore,
it will be presented in Lemma \ref{boundsrandom} that the $\varepsilon$-decay forced sampling method guarantees sufficient forced samples (on the logarithmic dependence on sample sizes) to ensure the performance of estimation.

\vspace{4pt}
\noindent
{\bf Two-Stage Sampling Procedure:} The procedure in Algorithm \ref{algo101} is introduced from \citet{bastani2015online}. 
The screening stage determines a preliminary set of decisions by eliminating the sub-optimal arms based on the forced-sample estimators. The $h/2$ radius of the preliminary set is to ensure that the sub-optimal arms are excluded and the optimal arm is included. Then, in the precise stage, we determine the arm that yields the best-estimated reward according to the whole-sample estimators.

\subsection{Batched Low-Rank Bandit Algorithm}\label{batchlralgo}
Now we propose the batched low-rank bandit algorithm for low-rank matrix models defined in $\S$\ref{s:LinearModels}. Accordingly, we define the whole-sample estimator as $\bThetah(\cW_{k,l}, \lambda_{2,l})$ and the forced-sample estimator as $\bThetah(\cR_{k,l}, \lambda_1)$, and choose the grid identically to \eqref{eqgrid}. Then we can obtain the batched low-bank bandit for trace regression by only changing the parameters into matrix forms in Algorithm \ref{algo1} and updating the regularization parameter by $\lambda_{2,l}=\lambda_{2,0}\sqrt{(\log t_l+\log(d_1+d_2))/ t_l}$. Thus, for conciseness, we do not repeat the pseudocode of the algorithm here. Additionally, since the grid choice remains the same, the upper and lower bounds of forced-sample sizes up to time step $t$ are still on the order of $\log t$ (by Lemma \ref{boundsrandom}). The performance of the policy will be discussed in the following section. 

\section{Theoretical Results}\label{theory}


\subsection{High-Dimensional Sparse Bandit}\label{s:High-DimensionalSparseBandit}
Before illustrating the theorem, we will show four technical assumptions necessary for the theoretical analysis of the expected cumulative regret. We adapt the standard assumptions in the high-dimensional linear bandit literature such as \citet{bastani2015online}, \citet{wang2018online} and \citet{wang2020online}.
\begin{assumption}[Parameter set]\label{parameterset} There exist positive constants $x_{\max}$, $b$ and $s_0$, such that for any $t$ and $k \in [K]$, we have $\| X_t\|_{\max} \le x_{\max}$, $\| \beta_k^{\true}\|_1 \le b$ and $\| \beta_k^{\true} \|_0 \le s_0$.
\end{assumption}

The first assumption is proposed by \citet{rusmevichientong2010linearly} to guarantee that the observed covariate vector $X_t$ is in the bounded subspace $\mathcal X = \{ x \in \RR^d: \| x \|_{\max} \le x_{\max} \}$, and the arm parameters $\beta_k^{\true}$ are bounded and sparse. According to the H\"older inequality 
, we have $\,|\, X_t^{\mathrm{T}}\beta_k^{\true} \,|\, \le bx_{\max}$. Thus, the expected regret at any time step is at most $R_{\max} = bx_{\max}$.

\begin{assumption}[Margin condition]\label{margincondition} There exists a constant $C_m \ge 0$ such that for $k,j \in [K]$ and $k \ne j$, we have $\mathbb{P}\left( 0 \le \,|\, X^{\mathrm{T}}(\beta_k^{\true} - \beta_j^{\true})\,|\, \le \gamma \right) \le C_m R_{\max}\gamma$.
\end{assumption}

This assumption is known as the margin condition for the $K$-class classification problem and is adapted to linear bandit models in \citet{bastani2015online}. In this assumption, the probability that the covariate vector approximates a decision boundary hyperplane $\{ x^{\mathrm{T}}\beta_k=x^{\mathrm{T}}\beta_j\}$ is bounded. Therefore, given a user's covariate vector, the decisions are appropriately separated based on their rewards. Otherwise, it is likely to pull the wrong arm since there may exist two arms whose rewards are too close.

\begin{assumption}[Arm optimality]\label{armoptimality} There exists two mutually exclusive sets $\mathcal{K}_s$ and $\mathcal{K}_o$ which make up the arm set $[K]$, with the suboptimal set $\mathcal{K}_s = \{k \in [K] \,|\, x^{\mathrm{T}}\beta_k < \max_{j \ne k} x^{\mathrm{T}}\beta_j - h ,~\text{a.e.}~ x \in \mathcal{X} \}$ and the optimal set $\mathcal{K}_o = \{ k \in [K] \,|\, \exists ~\cU_k \subseteq \mathcal X ~\text{s.t.}~\P(x \in \cU_k) > p_{\ast}~\text{and}~ x^{\mathrm{T}}\beta_k > \max_{j \ne k} x^{\mathrm{T}}\beta_j + h ~\text{for}~ x \in \cU_k \}$ , where parameter $h > 0$ and $p_{\ast}$ is a positive constant satisfying that $\min_{k \in \mathcal{K}_o} \mathbb P(x \in \cU_k) \ge p_{\ast}$.
\end{assumption}
\begin{figure}
    \centering
    \includegraphics[width=0.5\textwidth]{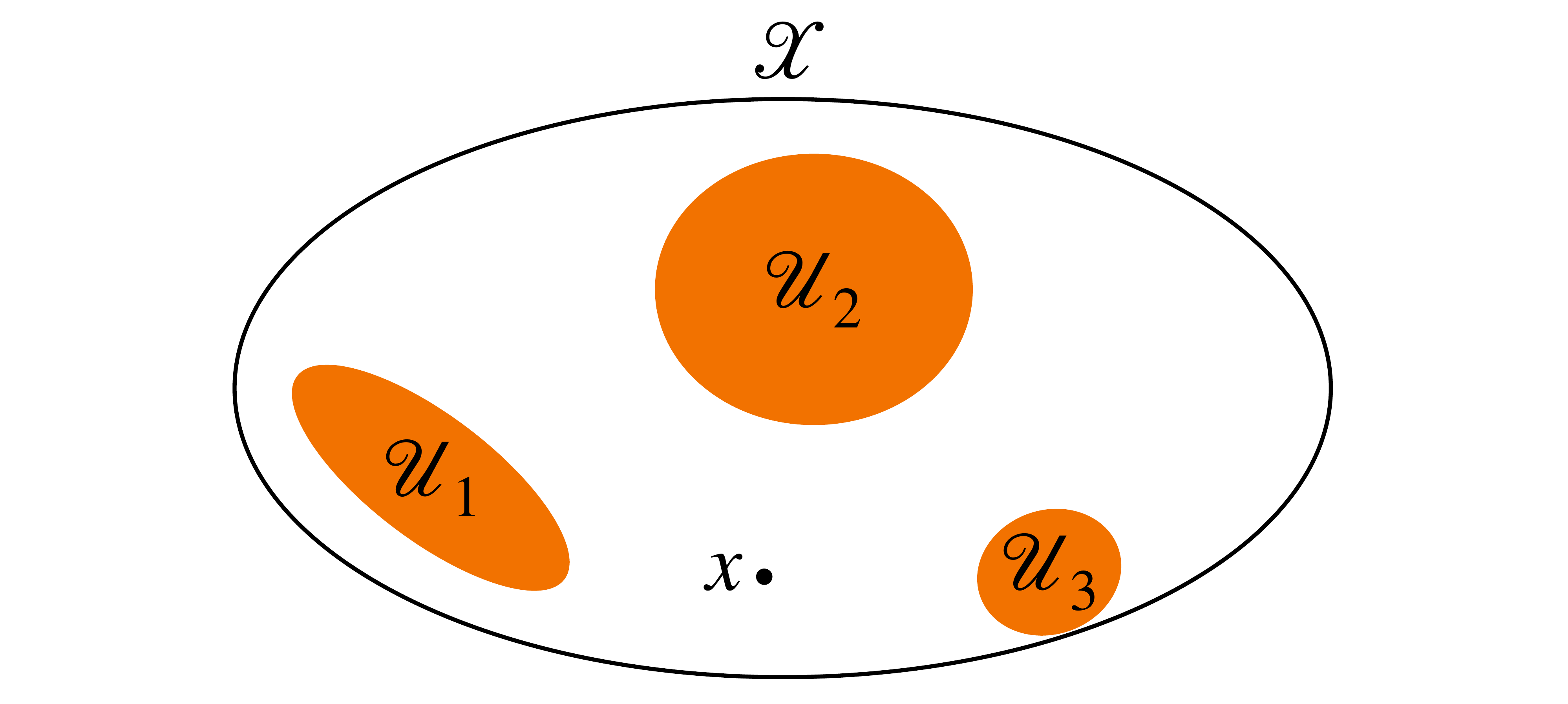}
    \caption{An illustration of the arm optimality condition, where the $5$ arms are divided into the optimal set $\cK_o=\{1,2,3\}$ and the suboptimal set $\cK_s=\{4,5\}$, the $\mathcal{U}_1,~\mathcal{U}_2,~\mathcal{U}_3$ are the optimal regions for the arms $1,~2,~3$, and the $\cX\setminus\{\cU_1\cup\cU_2\cup\cU_3\}$ is the region where no arm is strictly optimal.}
    \label{fig:Arm optimality}
\end{figure}

Following from \citet{goldenshluger2013linear} and \citet{bastani2015online}, this assumption requires that the suboptimal and optimal sets make up $[K]$. For any optimal arm $k \in \mathcal{K}_o$, there exists a region $\cU_k$ with positive measure such that for any covariate $x \in \cU_k$, the corresponding reward is strictly optimal. 
For the suboptimal arm $k\in\mathcal{K}_s$, for any covariate $x\in\cX$, the expected reward $x^\top\beta_k$ is strictly smaller than the best possible expected reward.
Additionally, Assumption \ref{margincondition} and Assumption \ref{armoptimality} are related in the way that when the margin condition holds and the parameter $C_m$ is small enough, the arm optimality holds (i.e., for any arm $k\in[K]$, it belongs to either $\cK_o$ or $\cK_s$). Particularly, we can demonstrate that this assumption holds whenever Assumption \ref{margincondition} holds and the probability of falling into the margin is small enough. For details, we refer readers to Lemma \ref{lm:pfoptimality}. 

\begin{assumption}[Compatibility condition]\label{compatibilitycondition} There exists a constant $\phi_0 > 0$ such that, for each 
$k \in [K]$, we have $\Sigma_k \in \mathcal{C}(\supp(\beta_k), \phi_0)$, where we define
\$
\mathcal{C}(\cI,\phi) = \bigl\{M \in \RR_{\succeq 0}^{d \times d}\,\big|\, \forall v \in \RR^d ~\text{s.t.} \| v_{\cI^c}\|_1 \le 3\| v_\cI\|_1, \text{we have}~ \| v_\cI\|_1^2 \le \,|\, \cI\,|\, \cdot (v^{\mathrm{T}}Mv)/\phi^2\bigr\},
\$
and for each $k\in[K]$, $\Sigma_k = \mathbb{E}(XX^{\mathrm{T}} \,|\, X \in \cV_k)$, where the region $\cV_k\subseteq\mathcal{X}$ and $\P(x\in\cV_k)>p_{\ast}$.
\end{assumption}

Note that for arms $k\in\cK_o$, the region $\cV_k$ is exactly $\cU_k$ defined in Assumption \ref{armoptimality}. The compatibility condition (or restricted eigenvalue condition) is necessary for the consistency of high-dimensional estimators \citep{candes2007dantzig,bickel2009simultaneous,negahban2012unified,buhlmann2011statistics,wang2018online,fan2020statistical}. As the covariance matrix $\Sigma_k$ is the conditional expectation of the Hessian matrix for the loss function in \eqref{LASSO}, this condition requires that the loss function is locally strongly convex in a cone subspace. In low-dimensional settings, this condition means that loss functions are strongly convex near the true parameter. Nevertheless, the strong convexity is inappropriate for high-dimension settings since the covariate dimensions significantly outweigh sample sizes. To illustrate the above four assumptions, a simple example is provided in \citet{bastani2015online}.

Following the batched sparse bandit algorithm, the expected cumulative regret upper bound can be established in the theorem below.
\begin{theorem}[Total Regret]\label{th:linear}
Suppose that Assumptions \ref{parameterset}-\ref{compatibilitycondition} hold. When $K \ge 2$, $d > 2$, $T \ge C_5$, $L\ge 2$, under Algorithm \ref{algo1} with the grid parameter $a \ge \log t_1$, $\lambda_1 = \phi_0^2 p_{\ast}h/(64s_0 x_{\max})$, $\lambda_{2,l}=(\phi_0^2/s_0)\cdot \sqrt{(\log t_l + \log d)/(2 C_1p_{\ast} t_l)}$ and $\lambda_{2,0}=(\phi_0^2/s_0)\cdot \sqrt{1/(2p_{\ast} C_1 )}$, the cumulative regret up to time $T$ is upper bounded:
$$
\RG(T,L)\le\mathcal{O} \left( K s_0^2 \log d (\log^2 T+\log d) \right),
$$
where the constants $C_0$, $C_1$, $C_2$, $C_3$, $C_4$ and $C_5$ are independent of $T$ (the specific values are provided in Appendix \ref{s:Proof_of_Theorem_1}), and $t_0 = 2C_0K$.
\end{theorem}
Theorem \ref{th:linear} first demonstrates that the expected cumulative regret of the batched sparse bandit over $T$ time steps is upper-bounded by $\mathcal{O} (s_0^2\log^2 T)$ within a factor of $\log d$, thus matching the regret bound achieved by the sequential version \citep{bastani2015online}, within a factor of $\log d$. Therefore, the batched sparse bandit successfully balances the regrets in each batch with the grid choice.  Nevertheless, the batched sparse bandit and LASSO bandit both do not meet the lower bound $\cO(s_0\log T)$, which is a natural extension from the $\cO(d\log T)$ lower bound in the low-dimensional setting \citep{goldenshluger2013linear}.

Now we briefly introduce the proof techniques. To begin with, since our algorithm ensures enough i.i.d. samples in both forced-sample and whole-sample sets, we will obtain the convergence result for the forced-sample and whole-sample estimators by the LASSO tail inequality for non-i.i.d. data and the batch allocation policy. Then, we divide the total sample size $T$ into three groups and bound the regret in each group. 
Eventually, with our batch allocation policy, the cumulative expected regret is upper bounded by $\mathcal{O} (s_0^2\log^2 T)$. For more details, see $\S$~\ref{sketch} for a proof sketch and Appendix \ref{ss:Bounding the Cumulative Regret} for the detailed proof.

\subsection{Batched Low-Rank Bandit}\label{s:BatchedLow-RankBandit}
Before introducing the assumptions for the low-rank bandit, we adopt the concept of decomposability and subspaces from \citet{wainwright2019high}, which is crucial for the compatibility condition in the matrix form. Consider any matrix $\bTheta\in\RR^{d_1 \times d_2}$ and denote its row and column spaces as $\text{rowspan}(\bTheta)\subseteq\RR^{d_2}$ and $\text{colspan}(\bTheta)\subseteq\RR^{d_1}$. Given a positive integer $r\le d'=d_1\wedge d_2$, we let $\UU$ and $\VV$ denote the $r$-dimensional subspaces of vectors. Then the two subspaces of matrices can be defined as
\#
\cM(\UU,\VV)&=\left\{ \bTheta\in\RR^{d_1 \times d_2}\mid \text{rowspan}(\bTheta)\subseteq\VV,~\text{colspan}(\bTheta)\subseteq\UU\right\},\label{eq:subspacemuv}\\
\barcM^{\bot}(\UU,\VV)&=\left\{ \bTheta\in\RR^{d_1 \times d_2}\mid \text{rowspan}(\bTheta)\subseteq\VV^{\bot},~\text{colspan}(\bTheta)\subseteq\UU^{\bot}\right\},\label{eq:subspacebmuv}
\#
where $\UU^{\bot}$ and $\VV^{\bot}$ denote the subspaces orthogonal to $\UU$ and $\VV$. By taking the orthogonal complement of \eqref{eq:subspacebmuv}, we can define the subspace $\barcM(\UU,\VV)$, which is a strict superset of $\cM^{\bot}(\UU,\VV)$. Additionally, it is easy to verify that the nuclear norm is decomposable with respect to the given pair of subspaces $(\cM,\barcM^{\bot})$ defined in \eqref{eq:subspacemuv} and \eqref{eq:subspacebmuv}:
\$
\|\bA+\bB\|_N=\|\bA\|_N+\|\bB\|_N,
\$
for any pair of matrices $\bA\in\cM$ and $\bB\in\barcM^{\bot}$.

For arm $k\in[K]$, suppose that the target matrix $\bTheta_k^{\true}$ has a low-rank structure $\max_{k\in[K]}\rank(\bTheta_k^{\true})\le r$, and write it in the SVD factored form $\bTheta_k^{\true}=\bU_k\bD_k\bV_k^{\mathrm{T}}$, where the first $r$ entries of the diagonal matrix $\bD_k\in\RR^{d'\times d'}$ are the $r$ nonzero singular values of $\bTheta$, and the first $r$ columns of the orthonormal matrices $\bU_k\in\RR^{d_1\times d'}$ and $\bV_k\in\RR^{d_2\times d'}$ are the left and right singular vectors of $\bTheta$. Let $\UU_k$ and $\VV_k$ be the $r$-dimensional subspaces spanned by the column vectors of $\bU_k$ and $\bV_k$ respectively, thus yielding the pair of subspaces $(\cM_k,\barcM_k^{\bot})$. Now we state the assumptions for the low-rank matrix bandit, which are adapted from the high-dimensional linear bandit and \citet{wainwright2019high}.
\begin{assumption}[Parameter set]\label{as0501} The dimensions $d_1, d_2 \ge 2$ and the rank constraint parameter $\rho>1$. There exist positive constants $x_{\max}$, $b$, $\kappa_0$, such that for any $t$ and $k \in [K]$, we have $\| \bX_t\|_{\max} \le x_{\max}$, $\| \bX_t \|_{F} \le \kappa_0$
, $\| \bTheta^{\true} \|_N \le b$. 
\end{assumption}

This assumption ensures that the covariate $\bX_t$ and the parameter $\bTheta^{\true}$ are in bounded subspaces. The bounds on $\bX_t$ and $\bTheta^{\true}$ in \cite{li2022simple} imply this assumption. According to the \chenlu{matrix H\"older inequality} \citep{baumgartner2011inequality}, for any time $t$, $\,|\, \la \bX_t^{\mathrm{T}}, \bTheta_k^{true} \ra \,|\, \le \| \bX_t \|_{\op} \cdot \| \bTheta^{\true} \|_N \le \| \bX_t \|_{F} \cdot \| \bTheta^{\true} \|_N \le \kappa_0 b$ holds. However, the regret at any time step is upper bounded by $R_{\max} = \kappa_0b$.

\begin{assumption}[Margin condition]\label{as0502} There exists a $C_m \ge 0$ such that for $k \ne j$ and $k,j \in [K]$, we have $\mathbb{P}\left( 0 < \,|\, \la X^{\mathrm{T}}, \bTheta_k^{true} - \bTheta_j^{true} \ra \,|\, \le \gamma \right) \le C_mR_{\max}\gamma$.
\end{assumption}

\begin{assumption}[Arm optimality]\label{as0503} There exists two mutually exclusive sets $\mathcal{K}_s$ and $\mathcal{K}_o$ that include all $K$ arms for some positive constant $h$, with suboptimal set $\mathcal{K}_s = \{k \in [K] \,|\, \la \bx^{\mathrm{T}}, \bTheta_k \ra < \max_{j \ne k} \la \bx^{\mathrm{T}}, \bTheta_j \ra - h,~\text{a.e.}~ \bx \in \mathcal{X} \}$ and optimal set $\mathcal{K}_o = \{ k \in [K] \,|\, \exists ~\cU_k \subseteq \mathcal X, ~\text{s.t.}~P(\bx \in \cU_k) > 0 ~\text{and}~ \la \bx^{\mathrm{T}}, \bTheta_k \ra > \max_{j \ne k} \la \bx^{\mathrm{T}}, \bTheta_j \ra + h ~\text{for}~ \bx \in \cU_k \}$, where parameter $h >0$ and $p_{\ast}$ is a positive constant satisfying that $\min_{k \in \mathcal{K}_o} \mathbb P(\bx \in \cU_k) \ge p_{\ast}$. In other words, for optimal arms $K \in \mathcal{K}_o$, we define
\#
\cU_k = \bigl\{ \bx \in \mathcal{X} \,|\, \la \bx^{\mathrm{T}}, \bTheta_k \ra > \max_{j \ne k} \la \bx^{\mathrm{T}}, \bTheta_j \ra + h \bigr\}.
\#
\end{assumption}

Assumptions \ref{as0502} and \ref{as0503} are directly extended from the corresponding assumptions for the high-dimensional linear setting.

\begin{assumption}[Compatibility condition]\label{as0504}
Consider the pair of subspaces $(\cM_k,\barcM_k^{\bot})$ defined above. There exists a constant $\phi_0 > 0$ such that, for each
$k \in [K]$, we have $\Sigma_k \in \mathcal{C}(\cM_k,\cM_k^{\bot}, \phi_0)$, where we define
\#\label{eq:cC}
\mathcal{C}(\cM_k,\barcM_k^{\bot},\phi) = \{M \in \RR_{\succeq 0}^{{d_1d_2 \times d_1d_2}}\,|\,& \forall~ \bDelta \in \RR^{d_1 \times d_2} ~\text{s.t.}~ \| \bDelta_{\barcM_k^{\bot}}\|_N \le 3\| \bDelta_{\barcM_k}\|_N,\nonumber\\
&\text{we have}~ \| \bDelta_{\barcM_k}\|_F^2 \le \vec(\bDelta)^{\mathrm{T}} M \vec(\bDelta)/(4\phi)^2\}
\#
and $\Sigma_k = \mathbb{E}(\vec(\bX)\vec(\bX)^{\mathrm{T}} \,|\, \bX \in \cV_k)$, where the region $\cV_k\subseteq\mathcal{X}$ and $\P(\bx\in\cV_k)>p_{\ast}$ for each $k\in[K]$.

\end{assumption}

For arms $k\in\cK_o$, the region $\cV_k$ is exactly $\cU_k$ defined in Assumption \ref{as0503}. This assumption, adapted from the compatibility condition for linear estimation, is closely related to the restricted strong convexity for matrix estimation (Equation (10.17) in \citet{wainwright2019high}). This condition also ensures the diversity of training samples so that the low-rank estimation will converge. 

We also have the regret upper bound for the batched low-rank bandit in the following theorem. Recall that the agent's decision and the optimal policy at time $t$ are $\pi_t$ and $\pi_t^{\ast}$, respectively. The expected cumulative regret at the time $T$
\$
\RG(T,L)= \sum_{l=1}^L\sum_{t=t_{l-1}+1}^{t_l}\mathbb{E}\big( \big\la X_t^{\mathrm{T}}, \bTheta_{\pi_t^{\ast}}^{true} - \bTheta_{\pi_t}^{true} \big\ra \big)
\$
is bounded by the following theorem.

\begin{theorem}[Total regret]\label{th:matirx}
Suppose that Assumptions \ref{as0501}-\ref{as0504} hold. Under the batched low-rank bandit, we take $\lambda_1 = \phi_0^2 p_{\ast}h/(48r x_{\max})$, $\lambda_{2,l}=2\phi_0^2/(3r)\cdot \sqrt{2(\log t_l + \log(d_1+d_2))/(p_{\ast} C_1 t_l)}$ and $\lambda_{2,0}=2 \phi_0^2/(3r) \sqrt{2/(p_{\ast} C_1)}$. When $K \ge 2$, $d_1,d_2 > 1$, $T \ge C_5$ and $L\ge 2$, the cumulative regret up to time $T$ is upper bounded:  
\#
\RG(T,L)\le\mathcal{O}\big( K r^2 \log(d_1d_2) (\log^2 T+\log(d_1d_2))  \big),
\notag
\#
where the constants $C_0$, $C_1$, $C_2$, $C_3$, $C_4$ and $C_5$ are constants independent of $T$ (the specific values are provided in Appendix \ref{s:Proof_of_Theorem_2}), and we take $t_0 = 2C_0K$.
\end{theorem}
Theorem \ref{th:matirx} indicates that the expected cumulative regret of the batched low-rank bandit over $T$ time steps is upper-bounded by $\mathcal{O} \left( K r^2 \log(d_1d_2) (\log^2 T+\log(d_1d_2) \right)$. To the best of our knowledge, this is the first $\mathcal{O} \left(\log(d_1d_2) (\log^2 T+\log(d_1d_2))\right)$ regret bound for low-rank linear bandits, outperforming previous algorithms for low-rank bandit problems in both $T$ and $d_1,d_2$. The analysis of this theorem directly follows the proof of Theorem \ref{th:linear}. We present the details of proving the bound in Appendix \ref{s:Proof_of_Theorem_2}.
\section{Proof Sketch}\label{sketch}
Although we adopt the main structure of the proof in \citet{bastani2015online}, we highlight our technical novelty from three perspectives: first, since the batched setting significantly decreases the frequency of updating parameters, we design the grid iteration \eqref{eqgrid} to counteract the influence of batched feedbacks and bound the regret in batch $l$ on the order of $a/l$; second, we demonstrate that for general batch allocation policies (not restricted to our grid selection), there are sufficient i.i.d. samples in the whole-sample set. Notably, we demonstrate in Proposition \ref{prop4general} that this conclusion holds for a general grid choice by utilizing the tricks of keeping the second half of batches and bounding the summations with integrations; third, for the low-rank bandit, we prove a convergence result for non-i.i.d. data in Lemma \ref{lem0501}. The difficulty lies in coping with the matrices and operator norms. Hence, we prove a convergence result for Matrix Sub-Gaussian Series and adopt matrix Bernstein inequality from \citet{tropp2011freedman}.

The contents of this section are organized as follows. To begin with, we present abridged technical proofs for the batched sparse bandit in four parts: in $\S$\ref{LASSOtailineq}, we adopt the LASSO tail inequality for non-i.i.d. data from \citet{bastani2015online} to provide a basic convergence guarantee; in $\S$\ref{sub402}, we obtain the convergence result for forced-sample estimators; in $\S$\ref{sub403}, we construct an i.i.d. subset and demonstrate the convergence for whole-sample estimators; in $\S$\ref{sub404}, we ultimately get the regret upper bound.


\subsection{A LASSO Tail Inequality For Non-i.i.d. Data}\label{LASSOtailineq}
Since general convergence guarantees do not apply to this setting, we adopt the tail inequality for non-i.i.d. data from \citet{bastani2015online}, which facilitates the proof of the following convergence properties. First of all, a general conclusion for the LASSO estimator for the non-i.i.d. sample will be proved. Specifically, we consider the design matrix $\mathbf{Z}_{n \times d}$ whose rows are bounded random vectors (i.e., $\| Z_t\|_{\max} \le x_{\max}$ for all $t \in [T]$), the response vector $W_{n \times 1}$ and the noise vector $\varepsilon_{n \times 1}$, and describe their relationships with the linear model
$$
W = \mathbf{Z}\beta+\varepsilon.
$$
It is assumed that $\| \beta\|_0=s_0$. For any subset $\mathcal{S}' \subset [T]$, let $\mathbf{X}(\mathcal{S}' )$ be the $\,|\, \mathcal{S}' \,|\, \times d$ submatrix of $\mathbf{X}$ whose rows are $X_t$ for each $t \in \mathcal{S}'$. Recalling the notation in $\S$\ref{notation}, we define the terms $\mathbf{Z}(\mathcal{A})$, $W(\mathcal{A})$, and $\hat{\Sigma}(\mathcal{A})$ similarly for any subset $\mathcal{A} \subset [T]$. Trained on samples in $\mathcal{A}$ (which are not required to be i.i.d.), a LASSO estimator is obtained for any $\lambda \ge 0$:
$$
\hat{\beta}(\mathcal{A},\lambda) = \argmin_{\beta '} \Big\{ \frac{\| W(\mathcal{A})-\mathbf{Z}(\mathcal{A}) \beta '\|_2^2}{\,|\, \mathcal{A}\,|\,} + \lambda \| \beta ' \|_1 \Big\}.
$$
Now suppose that there exists some unknown subset $\mathcal{A}' \subset \mathcal{A}$ comprised of i.i.d. samples which are subject to a distribution $\mathcal{P}_Z$. We define $\Sigma = \mathbb{E}_{Z \sim \mathcal{P}_Z}(ZZ^{\mathrm{T}})$ and assume that $\Sigma \in \mathcal{C}(\text{supp}(\beta), \phi_1)$ for a constant $\phi_1 \in \RR^+$. \citet{bastani2015online} demonstrates in the following lemma that if the size of an i.i.d. sample $\mathcal{A}'$ is large enough to make up a constant fraction of the samples in $\mathcal{A}$, a convergence guarantee will be proved for the LASSO estimator $\hat{\beta}(\mathcal{A},\lambda)$ trained on non-i.i.d. samples in $\mathcal{A}$. The result is shown as follows.

\begin{lemma}[LASSO Tail Inequality For Non-i.i.d. Data]\label{LASSOtailniid}
For any $\chi>0$, if $d>1$, $\,|\, \mathcal{A}'\,|\,/\,|\, \mathcal{A}\,|\, \ge p/2$, $\,|\, \mathcal{A}\,|\, \ge 6\log d/(pC_2^2(\phi_1))$, and $\lambda =\lambda(\chi, \phi_1\sqrt{p}/2)=\chi \phi_1^2p/(16s_0)$, then the following tail inequality holds:
$$
\mathbb{P}\bigl(\| \hat{\beta}(\mathcal{A},\lambda)-\beta\|_1>\chi\bigr)\le2\exp\bigl(-C_1\bigl(\phi_1\sqrt{p}/2\bigr)|\mathcal{A}|\chi^2+\log d\bigr)+\exp \bigl(-pC_2^2(\phi_1)|\mathcal{A}|/2\bigr),
$$
where $C_1(\phi_1) = \phi_1^4/(512 s_0^2 \sigma^2 x_{\max}^2)$ and $C_2(\phi_1) = \min ( 1/2,\phi_1^2/(256 s_0 x_{\max}^2))$.
\end{lemma}
The proof of this lemma can be seen in Appendix EC.2 of \citet{bastani2015online}, from which we know that generating sufficient i.i.d samples is necessary for the convergence of the LASSO estimator. Then, we will show in the following two subsections that the $\varepsilon$-decay and arm elimination methods guarantee the number of i.i.d. samples. 


\subsection{Estimator From Forced Samples Up To Batch \texorpdfstring{$l$}{l}}\label{sub402}
\begin{proposition}\label{prop3}
Suppose that Assumptions \ref{parameterset}, \ref{armoptimality} and \ref{compatibilitycondition} hold. Under Algorithm \ref{algo1}, for all arms $k \in [K]$, if $d>2$ and $t_l \ge (t_0 +1)^2/e^2 - 1$, the forced-sample estimator $\betah(\cR_{k,l}, \lambda_1)$ satisfies
$$
\mathbb{P}\Big( \| \betah(\cR_{k,l}, \lambda_1)-\beta_k^{\true}\|_1 \ge \frac{h}{4x_{\max}}\Big)\le \frac{7}{t_l+1},
$$
where $t_0 = 2C_0K$, $\lambda_1 = (\phi_0^2 p_{\ast}h)/(64s_0 x_{\max})$, $C_0 = \max \{10,8/p_{\ast}, 48\log d/(p_{\ast}^2C_2^2),
4\log d/C_1\cdot(96x_{\max}/(p_{\ast}h))^2 \}$, and the parameter $h$ is defined in Assumption \ref{armoptimality}.
\end{proposition}
Although the forced samples are i.i.d., we cannot directly utilize the LASSO tail inequality for i.i.d. data. Since the compatibility condition in Assumption \ref{compatibilitycondition} holds for the conditional covariance $\Sigma_k = \EE(XX^{\mathrm{T}}\mid X\in\cV_k)$ instead of $\EE(XX^{\mathrm{T}})$, we define $\cA^{'}=\{ \tau\in\cR_{k,l}\mid X_{\tau}\in\cV_k\}$ as the i.i.d. sample set from $\mathcal{P}_{X|X\in\cV_k}$. By invoking Lemma \ref{lem0501} with $\cA=\cR_{k,l}$ and $\cA^{'}$ defined above, the convergence result follows. The detailed proof is given in Appendix \ref{ss:Convergence of Random-Sample Estimators}.

\subsection{Estimator From Whole Samples Up To Batch \texorpdfstring{$l$}{l}}\label{sub403}
Consider a general grid choice $\cT'=\{ t_1,\ldots,t_l\}$:
\$
t_l=g(t_{l-1}),\quad\text{for}~l=2,\ldots,L,
\$
where $g:\ZZ\rightarrow\ZZ$ is a predetermined function. In this subsection, we will explore what kind of conditions for a general grid choice will ensure the tail inequality for the whole-sample estimators of optimal arms $k\in \mathcal{K}_o$. 

The challenge lies in the fact that simply invoking Lemma \ref{LASSOtailniid} with $\mathcal{A}= \mathcal{W}_{k,l}$ and $\mathcal{A}'= \mathcal{R}_{k,l}$ is not applicable due to the limited $\mathcal{O}(\log t_l)$ number of i.i.d. samples in $\mathcal{R}_{k,l}$, which is far from sufficient for the requirement of Lemma \ref{LASSOtailniid}. To resolve this, we will construct an i.i.d. sample set with the sample size in the same order of $t_l$.

In Algorithm \ref{algo1}, when random sampling is not conducted, the agent follows a two-step decision procedure to determine the optimal decision. The first step involves determining whether an arm is in the decision candidate set based on the performance of its forced-sample estimator. According to Proposition \ref{prop3}, this estimator will not be far from its true parameter value. We denote $\mathcal{A}_l$ as the event in which the forced-sample estimators in batch $l$ are within a given distance of their true parameters:
$$
\mathcal{A}_l=\left\{\| \betah(\cR_{k,l}, \lambda_1)-\beta_k^{\true}\|_1 \le \frac{h}{4x_{\max}},~\forall k\in[K]\right\}.
$$
Using Proposition \ref{prop3}, if $t_l\ge(t_0+1)^2/e^2-1$, we have $\mathbb{P}(\mathcal{A}_l)\ge 1-7K/(t_l+1)$. Moreover, conditioning on event $\mathcal{A}_l$, for any $x\in \cU_k, k\in[K]$, we can verify the following inequality:
$$
x^{\mathrm{T}}\betah(\cR_{k,l}, \lambda_1) \ge\max_{j\ne k}\left(x^{\mathrm{T}}\betah(\cR_{j,l}, \lambda_1)\right)+\frac{h}{2}.
$$
In other words, if event $\mathcal{A}_l$ happens for $x\in \cU_k$, in the first step, the agent will only choose arm $k$, which is the optimal arm. To bound the total number of time steps up to batch $l$ when $\mathcal A_l$ holds, and the agent selects the optimal arm via the two-step sampling procedure, we define for $t\in\{1,\ldots,t_{l}\}$ and $k\in\cK_o$:
$$
M_k(t)=\mathbb{E}\Big(\sum_{m=1}^{l}\sum_{\tau=t_{m-1}+1}^{t_m}\mathbbm{1}(X_{\tau}\in \cU_k,\mathcal A_{m-1},X_{\tau}\notin \mathcal{R}_{k,m}) \,\Big|\, \mathcal{F}_t\Big),
$$
where $\mathcal{F}_t=\{(x_{\tau},r_{\tau}) ,~\text{for}~ \tau\le t\}$. The difference between $M_k(t)$ defined in this paper and in \cite{bastani2015online} is that we need to divide the summation by batches. The $M_k(t)$ represents the size of a subset where the forced sampling is not executed, the covariate $X_{\tau}$ belongs to optimal set $\cU_k$, and the forced-sample estimator for the past batches approaches the true one with a given distance. We also denote the sample set $\{X_{\tau}\mid X_{\tau}\in \cU_k,\mathcal A_{m-1},X_{\tau}\notin \mathcal{R}_{k,m},\tau\in\{t_{m-1}+1,\ldots,t_m\}\}$ as $\mathcal{B}_{k,m}$, and we will show in Appendix \ref{sss:Proof of Proposition 5.3} that the samples in $\mathcal{B}_{k,m}$ are i.i.d. (given distribution $\mathcal{P}_{X|X\in \cU_k}$). Thus, $M_k(t_{l})$ denotes the number of a type of i.i.d. samples in the whole-sample set up to batch $l+1$. Eventually, knowing that $\{M_k(t_l)\}_{l\in[L]}$ is a martingale, we can bound the value of $M_k(t_{l})$ for each $l\in[L]$ in the following proposition.
\begin{proposition}\label{prop4general}
Suppose $l\ge 2$ and the grid $\cT'$ satisfies that $g$ is non-decreasing, $g(t_l)\le \alpha t_l$ for all $l\in[L]$ for some constant $\alpha>1$ and $t_2\ge 2t_1$. When $t_{l} \ge C_5^{'}=\min_{t\in\mathbb{Z}}\{ t\ge8K(C_0+3\alpha)\log(t/\tt_0), t\ge2g(\tt_0)\}$, where $\tt_0=\max \left\{ 7K-1, 2C_0 K-1, \left\lceil (t_0+1)^2/e^2 \right\rceil-1 \right\}$, for any $k\in\cK_o$, we have
$$
\mathbb{P}\left(M_k(t_{l})\ge\frac{p_{\ast}}{8}t_{l}\right)\ge 1 - \exp\left(-\frac{p_{\ast}^2}{128}(t_{l}-1)\right).
$$
\end{proposition}
This proposition implies that it suffices to require that the recursive grid function $g$ is non-decreasing and the growth rate does not exceed exponential growth ($t_l/t_{l-1}\le\alpha$). The proof of Proposition \ref{prop4general} is provided in Appendix \ref{sss:Proof of prop4general}. This proposition indicates that, with high probability, the actual i.i.d. sample size for decision $k$ in $\cU_k$ will be of the order of $\Theta(T)$ instead of $\Theta(\log T)$. Ensuring the number of $M_k(t_{l})$ is more challenging than in the case of Lemma EC.14 in \citet{bastani2015online} for two reasons. Firstly, our estimators are not updated until the end of each batch, and secondly, the forced-sampling time steps are not predetermined. Forced sampling may occur at each time step with a decreasing probability. Additionally, there is a discrepancy between the proof of Proposition \ref{prop4} and Proposition 5 in \citet{wang2018online}. When lower-bounding the probability of $\{X_{\tau}\in \cU_k,~\mathcal A_{m-1},~X_{\tau}\notin \mathcal{R}_{k,m}\}$ for each time step $\tau$ in batch $m$, the latter directly lower-bounds $\P(A_{m-1})$ and $\P(X_{\tau}\notin \mathcal{R}_{k,m})$ by their values of $1-7K/(T+1)$ and $1-2C_0K/(T+1)$ at the final time step $T$. This step is questionable as $1-C/\tau$ cannot be lower-bounded by $1-C/T$ for $\tau\le T$. To overcome this challenge, we introduce a new half-control technique that bound the summations of the remaining part of batches through integration.

Coming back to our specific grid in \eqref{eqgrid}, we demonstrate in Appendix \ref{sss:Proof of Proposition 5.3} that this grid satisfies the conditions in Proposition \ref{prop4} that show the following result.
\begin{proposition}\label{prop4}
Under Algorithm \ref{algo1}, if $t_{l} \ge C_5$, $l\ge 2$ 
, then for $k\in\cK_o$, we have
$$
\mathbb{P}\Big(M_k(t_{l})\ge\frac{p_{\ast}}{8}t_{l}\Big)\ge 1 - \exp\Big(-\frac{p_{\ast}^2}{128}(t_{l}-1)\Big).
$$
\end{proposition}
Now we can apply Lemma \ref{LASSOtailniid} to derive the following proposition. The proof is given in Appendix \ref{sss:Proof of Proposition 5.4}.
\begin{proposition}\label{prop5}
If Assumptions \ref{parameterset}, \ref{armoptimality} and \ref{compatibilitycondition} hold, Under Algorithm \ref{algo1}, when $t_l\ge C_5$ and $l\ge 2$, the whole-sample estimator $\betah(\cW_{k,l}, \lambda_{2,l})$ for $k \in \cK_o$ will satisfy the following inequality:
$$
\mathbb{P}\bigg(\| \betah(\cW_{k,l}, \lambda_{2,l}) - \beta_k^{\true}\|_1 \ge 32\sqrt{\frac{2(\log t_l +\log d)}{t_lp_{\ast}^3C_1(\phi_0)}}\bigg)\le\frac{2}{t_l}+2\exp\bigg(-\frac{t_lp_{\ast}^2C_2^2(\phi_0)}{64}\bigg),
$$
where
$\lambda_{2,l}=\phi_0^2/s_0 \cdot \sqrt{(\log t_l +\log d)/(2t_l p_{\ast}C_1(\phi_0))}$.
\end{proposition}

\subsection{Cumulative Regret up to Time \texorpdfstring{$T$}{T}} \label{cumuregret}\label{sub404}
Finally, with our convergence results, we divide the time steps, up to time $T$, into three groups to provide an upper bound for each group. Specifically, we define $l_0=\min\{ l'\ge1 \mid t_{l'}\ge C_5\}$ and partition the three groups are as follows.
\begin{enumerate}
\item Time steps in batch $l\le l_0$ (initialization) or forced sample $X_t\in\mathcal{R}_{k,L}~ \text{for}~ k \in [K]$.
\item Time steps out of forced-sample set $\mathcal{R}_{k,L}$ in batch $l>l_0$ when event $\cA_{l-1}$ does not hold.
\item Time steps out of forced-sample set $\mathcal{R}_{k,L}$ in batches $l>l_0$ when event $\cA_{l-1}$ holds.
\end{enumerate}

Note that the three groups collectively cover all $T$ time steps. The cumulative expected regret at time $T$ in the first group is bounded by $R_{\max}\cdot\left(t_{l_0}+K(2+6C_0\log T)\right)$, as the reward cannot be controlled in the initial stage (the first $l_0$ batches) and forced-sampling time steps. The regrets in these groups are bounded by $R_{\max}$ (as stated in Assumption \ref{parameterset}).

In the second group, the cumulative expected regret at time $T$ is upper-bounded by $7KR_{\max}\log^2 T$, as event $\cA_{l-1}$ occurs with a high probability according to Proposition \ref{prop3}. Therefore, we can assume the worst cases for the time steps in this group.

In the third group, the cumulative expected regret is bounded by $K ( C_3(\phi_0,p_{\ast})\log d+C_4(\phi_0, p_{\ast})) \log^2 T$. As event $\cA_{l-1}$ holds, the agent uses the forced-sampling estimator to choose the estimated best arm, and we can use Proposition \ref{prop4} without worrying that the chosen arm may not be the true optimal arm. Thus, the agent will choose the optimal arm when the agent does not sample randomly. Using Proposition \ref{prop5} for all time steps in this group, we can bound the expected regret in this group. By combining the cumulative regret for all three groups, Theorem \ref{th:linear} follows. The detailed proof is presented in Appendix \ref{ss:Bounding the Cumulative Regret}.

\section{Empirical Results}\label{s:EmpiricalResults}
In this section, we will evaluate the performance of the batched sparse bandit and the batched low-rank bandit by examining how their cumulative regret is affected by different factors such as time steps, batch number, data dimensions, and the size of the decision set.

In the high-dimensional sparse setting ($\S$\ref{ssec:hdbandit}), we will compare the performance of our batched bandit algorithm to the LASSO bandit algorithm using both synthetic data and real-world datasets. This evaluation is based on the formulations outlined in \citet{bastani2015online}. In the low-rank setting ($\S$\ref{ssec:lrbandit}), we will conduct experiments on synthetic data to compare the performance of our batched bandits.

\subsection{High-Dimensional Bandit}\label{ssec:hdbandit}

\subsubsection{Synthetic Data}

\begin{figure}[tp]
\centering
\subfigure[$K=2,d=100,s_0=5$]{
\centering
\includegraphics[height=1.6in]{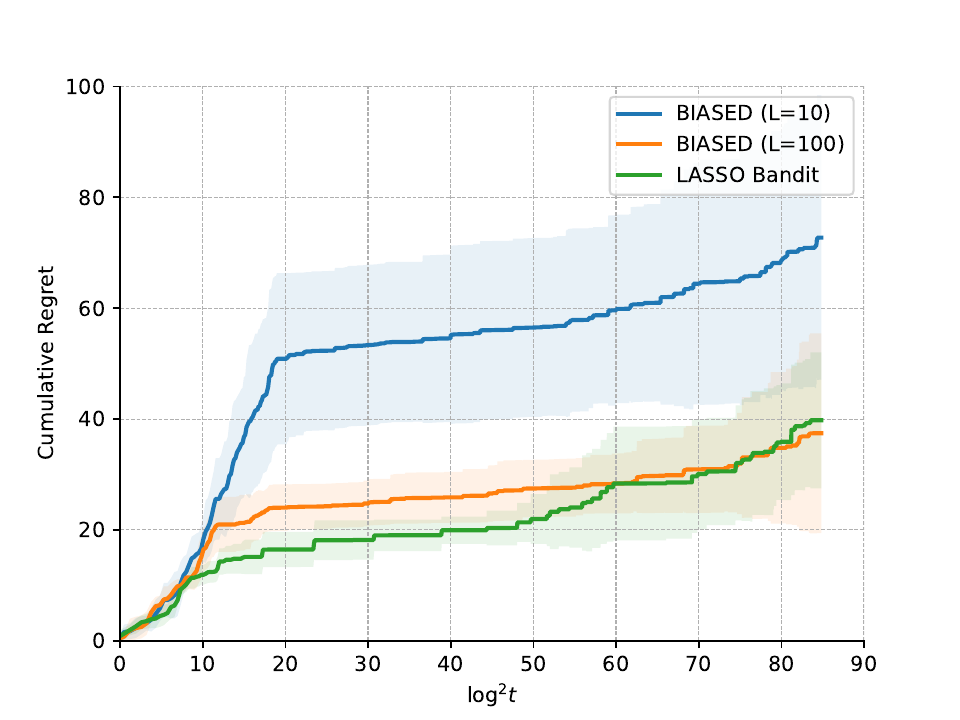}
\label{f:reg_toy1}
}
\subfigure[$K=10,d=1000,s_0=2$]{
\centering
\includegraphics[height=1.6in]{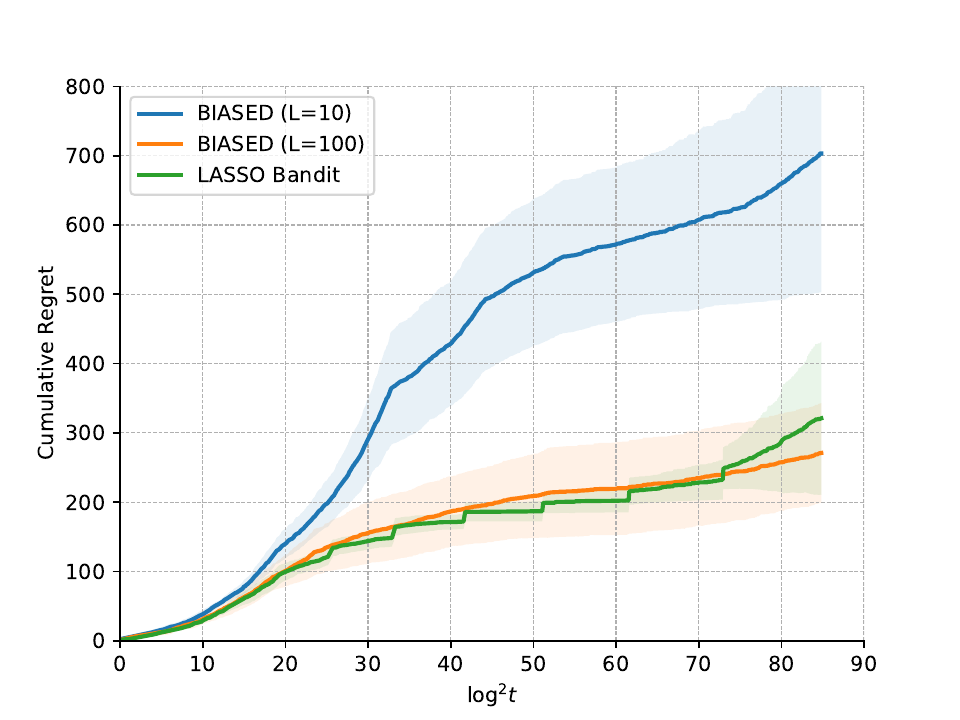}
\label{f:reg_toy2}
}  
\subfigure[$K=20,d=1000,s_0=2$]{
\centering
\includegraphics[height=1.6in]{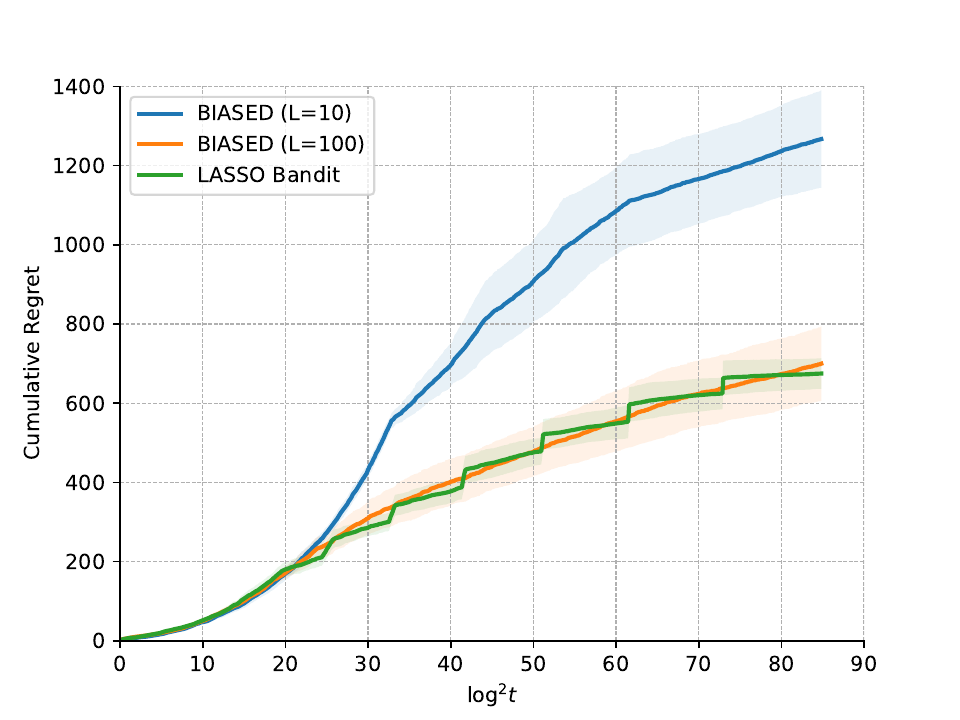}
\label{f:reg_toy3}
}
\subfigure[$K=2,d=100\sim1000,s_0=5$]{
\centering
\includegraphics[height=1.6in]{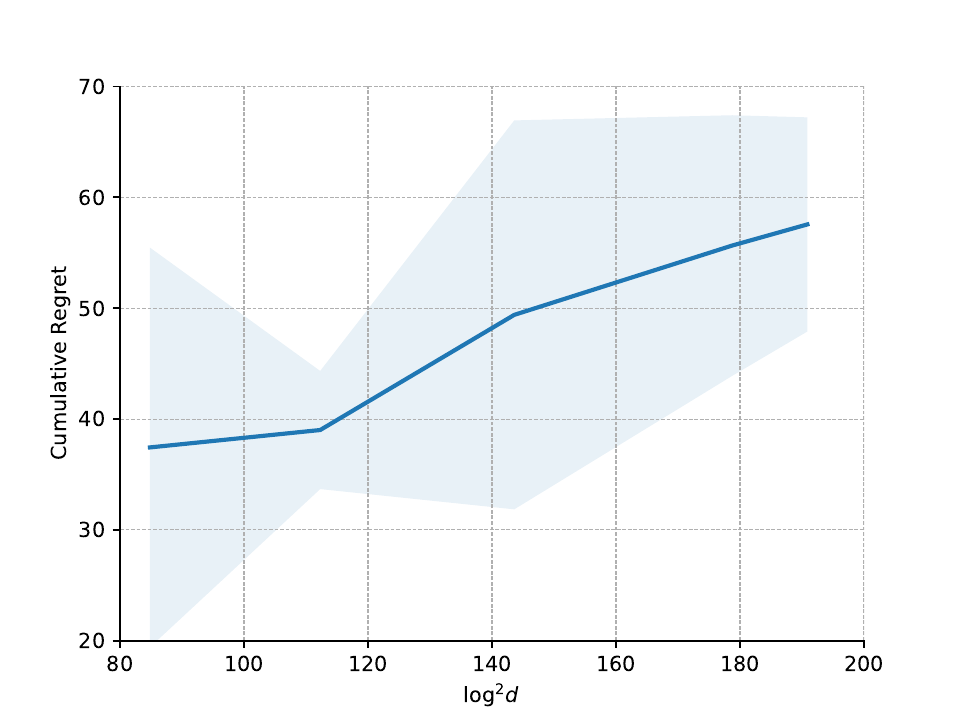}
\label{f:reg_toy4}
}
\centering
\caption{The performance on synthetic data for BIASED under batch $L=10$ and $L=100$ compared with LASSO bandit. The trajectory is the average of $5$ random 
and independent trials (shaded regions depict the $95\%$ confidence intervals). We examine the algorithms for different values of $K$, $d$, and $s_0$. The results align with our theoretical bounds: the cumulative regret has $\cO(\log^2 t)$ and $\cO(\log^2 d)$ dependence on time step $t$ and dimension $d$, respectively.}
\label{f:toy_high}
\end{figure}

{\bf Dataset.} Consider three cases for arm number $K$, data dimension $d$ and sparsity parameter $s_0$: (1) $K=2, ~d=100, ~s_0=5$; (2) $K=10, ~d=1000, ~s_0=2$; (3) $K=20, ~d=1000, ~s_0=2$. In these cases, we benchmark the performance of the Batched hIgh-dimensionAl SparsE banDit (BIASED) under batches $L=10$ and $L=100$ to the LASSO bandit  \citep{bastani2015online} (with $L=T=10000$). Moreover, to validate the influence of the covariate dimension $d$ on the cumulative regret for BIASED, we fix $K=2,~s_0=5, L=100$ and plot the regret with $d=100,200,400,800, 1000$. In each scenario, the true parameter $\beta_k^{\true}$ for each $k\in[K]$ is a sparse vector, where only $s_0$ randomly chosen components are nonzero. These $s_0$ values are sampled from a uniform distribution on $[0,1]$. At each time $t$, a user covariate $X_t$ is independently drawn from a Gaussian distribution $\cN(\bm 0_d, \bI_d)$ and truncated within $[-1,1]$. Then, the noise variance is $\sigma^2 = 0.01^2$.

\vspace{4pt}
\noindent
{\bf Algorithm Inputs.} Given the time periods $T=10000$, to validate the robustness of our algorithm with respect to different batch numbers $L$, we choose the size of the first batch $t_1=4$ and consider two batch numbers $L$ by using two different grid parameters $a$ according to \eqref{eqgrid}: 1) To obtain $L=100$, we choose $a=9.7$; 2) To obtain $L=10$, we choose $a=25.8$. For the other input parameters, we select the initialized regularization parameters $\lambda_1=0.05,~\lambda_{2,0}=0.2$, the $\epsilon$-decay forced sampling parameter $t_0=K$ (consistent with the value $t_0=2C_0K$ in Theorem \ref{th:linear}). For LASSO bandit, we choose the forced sampling parameter $q=1$. For both bandits, the arm optimality parameter $h=5$ in scenarios (1) and (2), and $h=10$ in scenario (3). Conclusively, the choice for the tuning parameters is basically consistent with \citet{bastani2015online}.

\vspace{4pt}
\noindent
{\bf Results.}
The performance of BIASED is compared with the LASSO bandit in Figure \ref{f:toy_high} under $10$ and $100$ batches, respectively. The horizontal axis is set to $\log^2 t$ to examine the order of the upper bound with respect to $T$. The results, averaged over $5$ trials
and have $95\%$ confidence intervals, validate that the regret bound is of the order $\log^2 t$.  Additionally, Figures \ref{f:reg_toy1}-\ref{f:reg_toy3} demonstrate that BIASED is robust to the choice of parameters $L$, $K$, and $s_0$. Although a smaller batch size of $L=10$ yields a higher regret than $L=100$, both regret bounds are in the same order. Finally, Figure \ref{f:reg_toy4} shows that the cumulative regret at time $T=10000$ depends on the dimension $d$ with $\log^2 d$.

\subsubsection{Warfarin Dosing}
{\bf Dataset.} In the first experiment on real data, we apply our algorithm to a precision medicine problem called warfarin dosage \citep{international2009estimation}, where physicians decide the optimal warfarin dosage for arriving patients. In the dataset \protect\footnote{\url{https://github.com/chuchro3/Warfarin/tree/master/data}}, $5528$ patients carry individual covariates of dimension $d=93$, including demographic, diagnosis, previous diagnoses, medications, and genetic information.

\vspace{4pt}
\noindent
{\bf Bandit Formulation and Hyperparameters.} According to \citet{bastani2015online}, the problem is formulated as follows. The correct dosage is divided into three levels: low, medium and high. Since the correct dose is given but concealed, the reward is $0$ if the chosen arm is the patient's correct dose. Otherwise, the reward is $-1$. For the hyperparameters, we set the initialized regularization parameters $\lambda_1=1.0,~\lambda_{2,0}=0.5$, and the arm optimality parameter $h=2.0$. For BIASED, the batch number $L=50$, the grid parameter $a=10.2$, the size of the first batch $t_1=2$ and $\epsilon$-decay forced sampling parameter $t_0=15$. For LASSO bandit, the forced sampling parameter $q=5$.

\vspace{4pt}
\noindent
{\bf Results.} In Figure \ref{f:wfr}, we compare the fraction of incorrect dosages between BIASED and LASSO bandit averaging over $5$ trails. We note that the batched bandit nearly exhibits the same performance as the sequential bandit (LASSO bandit). 

\begin{figure}[tp]
    \centering
    \includegraphics[width=0.5\linewidth]{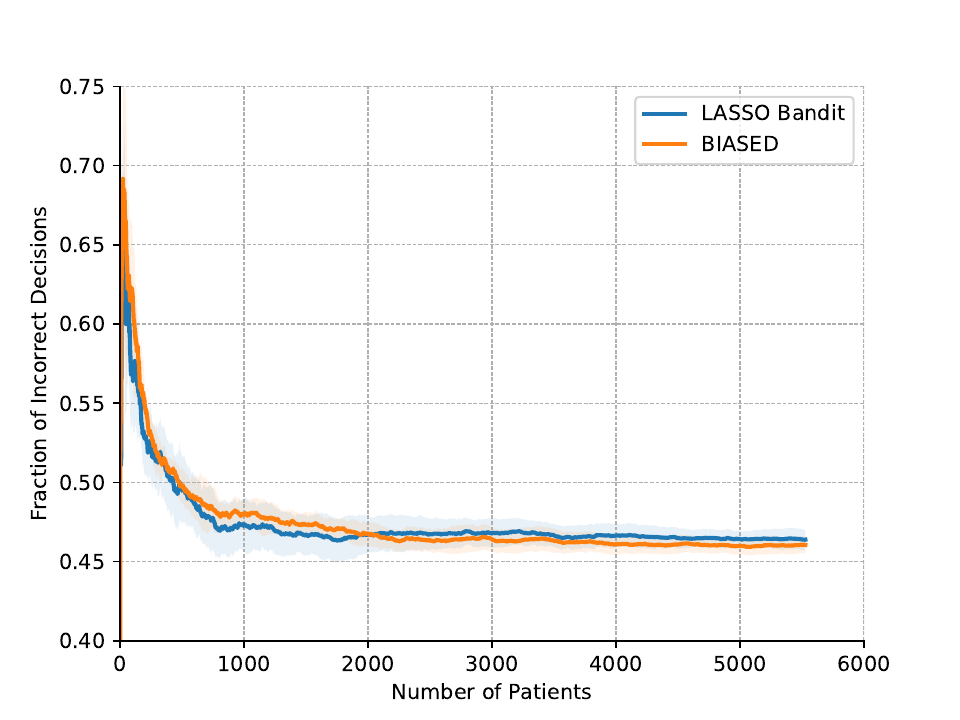}
    \caption{The numerical performance on health-care data for BIASED in comparison to LASSO bandit. The $x$-axis represents the number of patients, while the $y$-axis represents the fraction of incorrectly dosed patients. The presented curves are an average of $5$ independent and random trials and shaded by the $95\%$ confidence intervals. BIASED exhibits a performance that is comparable to the benchmark.}
    \label{f:wfr}
\end{figure}

\subsubsection{Dynamic Pricing in Retail Data}
Dynamic pricing problems with unknown demands \citep{besbes2009dynamic} can be modeled with a batched linear contextual bandit. Following \citet{xu2021learning}, we use a publicly available dataset of customized orders from meal delivery companies \protect\footnote{\protect\url{https://datahack.analyticsvidhya.com/contest/genpact-machine-learning-hackathon-1/}}. We select a subset containing $3584$ orders from the same city during the same period. The information on each order includes meal categories, cuisines, base prices and promotions. For each order $t$, the decision is the price $p_t$ (a continuous variable falling in a range of $[p_{\min},p_{\max}]$). The demand is modeled as a linear function of the contexts (of dimension $22$) and the price \citep{ban2021personalized}:
$$
Y_t = X_t^\top \beta_0 + p_t\cdot X_t^\top \beta_1 + \epsilon_t.
$$
The goal is to maximize the revenue $R_t=Y_t\cdot p_t$, which is known as the reward function. In the algorithm, we make parameter estimations $\hat \beta_0$, $\hat \beta_1$ and determine the price $\{p_t\}_{t=1}^T$. When selecting the price greedily, we let $p_t$ be the truncated maximizer of the estimated revenue function, i.e., $[X_t^\top \hat\beta_0/(-2X_t^\top\hat\beta_1)]_{[p_{\min},p_{\max}]}$. Then, we measure the performance by the regret $\sum_{t=1}^T (R_t^*-\hat Y_t\cdot p_t)$, where we define $R_t^*=Y_t^*\cdot p_t^*$ (where the $Y_t^*$ and $p_t^*$ are the true demand and price given in the dataset), and $\hat Y_t=X_t^\top \hat \beta_0 + p_t\cdot X_t^\top \hat \beta_1$.
Note that since the true parameters are unknown,  the demand $Y_t$ corresponding to the price $p_t$ was predicted using the estimation of the whole training set. We adapt the exploration strategy for dynamic pricing in \citet{ban2021personalized} to the batched sparse bandit, which is named BIASEX and shown in Algorithm \ref{algo:price}. The details can be found in Appendix \ref{s:Experimental_Details}.

\begin{figure}[tp]
    \centering
    \includegraphics[width=0.5\linewidth]{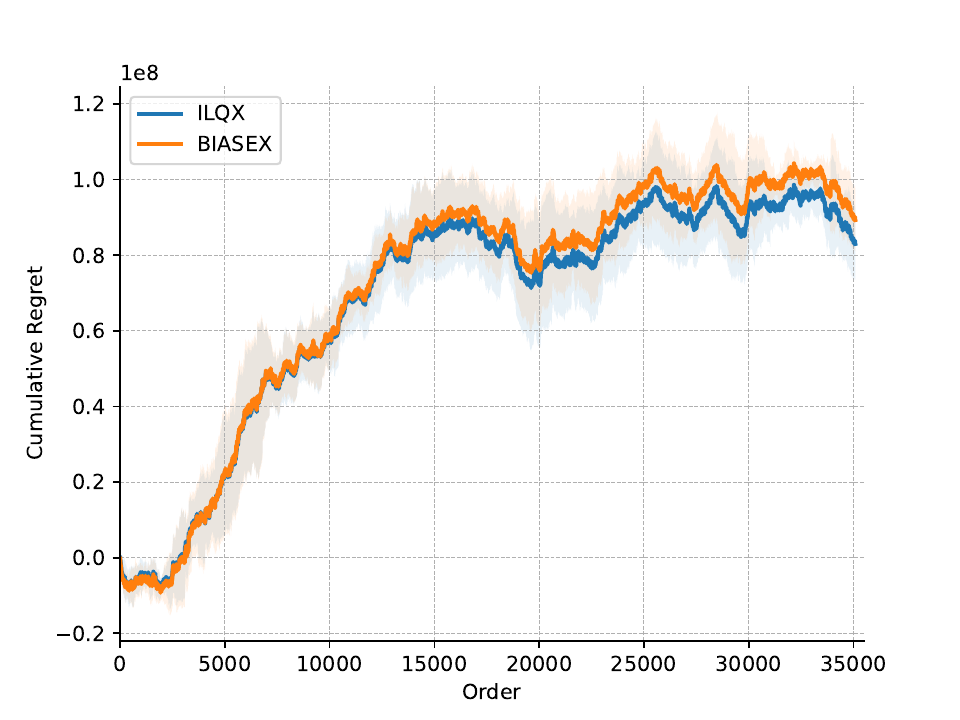}
    \caption{The numerical performance on retail data for BIASEX compared to ILQX (the LASSO-based pricing algorithm). The lines, averaged over $10$ trails and shaded by $95\%$ confidence intervals, indicate that the cumulative regret of BIASEX matches the benchmark.}
    \label{f:price}
\end{figure}

To compare the performance of BIASEX, we ran the algorithm with $100$ batches and compared it with ILQX, the LASSO-based pricing algorithm used in \citet{ban2021personalized}. Each curve was averaged over $10$ trials, and the results are shown in Figure \ref{f:price}. Our batched algorithm performed comparably to the baseline algorithm, demonstrating its effectiveness in solving batched dynamic pricing problems with unknown demands.

\subsection{Low-Rank Bandit}\label{ssec:lrbandit}
\vspace{4pt}
\noindent
{\bf Dataset.} We run the simulation for different arm numbers $K$, matrix row numbers and column numbers $d_1,d_2$ (assume that row and column numbers are the same $d_1=d_2=d$) and rank values $r$: (1) $k=2, d=10, r=2$; (2) $k=10, d=10, r=2$; (3) $k=5, d=20, r=5$. In the three scenarios, we compare the performance of the Batched LOw-rank Bandit (BLOB) among different batches: $L=10$, $L=100$ and $L=T$ (i.e., the sequential version). Then, by fixing $k=2, r=2, L=100$, we plot $d$ as $\{10,20,30,40,50\}$ in the graph to validate the order of regret bound on matrix dimensions. The true matrix parameter is a diagonal matrix with only $r$ nonzero components. The user covariate $\bx_t$ is a $d_1$ by $d_2$ matrix independently drawn from a Gaussian distribution $\cN(\bm 0_{d_1\times d_2}, \bI_{d_1\times d_2})$ and truncated between $[-1,1]$.

\vspace{4pt}
\noindent
{\bf Algorithm Inputs.} Given time steps $T=10000$, we choose $a=9.7$ to obtain $L=100$, and choose $a=26.5$ to derive $L=10$. In all scenarios, the number of the first batch is $t_1=2$. The initialized regularization parameters are chosen as $\lambda_1=\lambda_{2,0}=0.05$. Since the $\epsilon$-decay sampling parameter $t_0=2C_0K$, let $t_0=5r^2K/8$. 
For all the cases, the arm optimality parameter is set as $h=5$. The parameter selections are according to Theorem \ref{th:matirx} and follow the choices for the synthetic data in high-dimensional bandits.

\vspace{4pt}
\noindent
{\bf Results.} The curves of the regret bound are shown in Figure \ref{f:toy_lr}. From Figures \ref{f:reg_lr1}-\ref{f:reg_lr3}, we validate that the upper bound is on the order of $(\log T)^2$ and robust with respect to $L,K$ and $r$. Furthermore, Figure \ref{f:reg_lr4} verifies that the dependence of the cumulative regret at time $T=10000$ on the dimension $d_1,d_2$ is $\log^2(d_1d_2)$.

\begin{figure}[tp]
\centering
\subfigure[$K=2,~d=10,~r=2$]{
\centering
\includegraphics[height=1.6in]{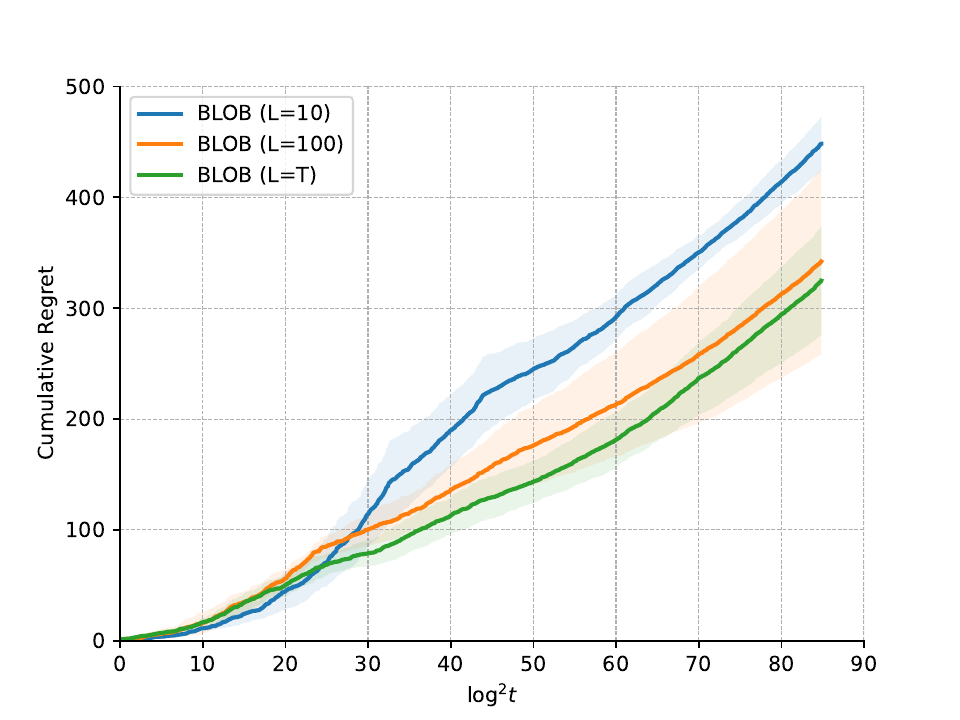}
\label{f:reg_lr1}
}
\subfigure[$K=10,~d=10,~r=2$]{
\centering
\includegraphics[height=1.6in]{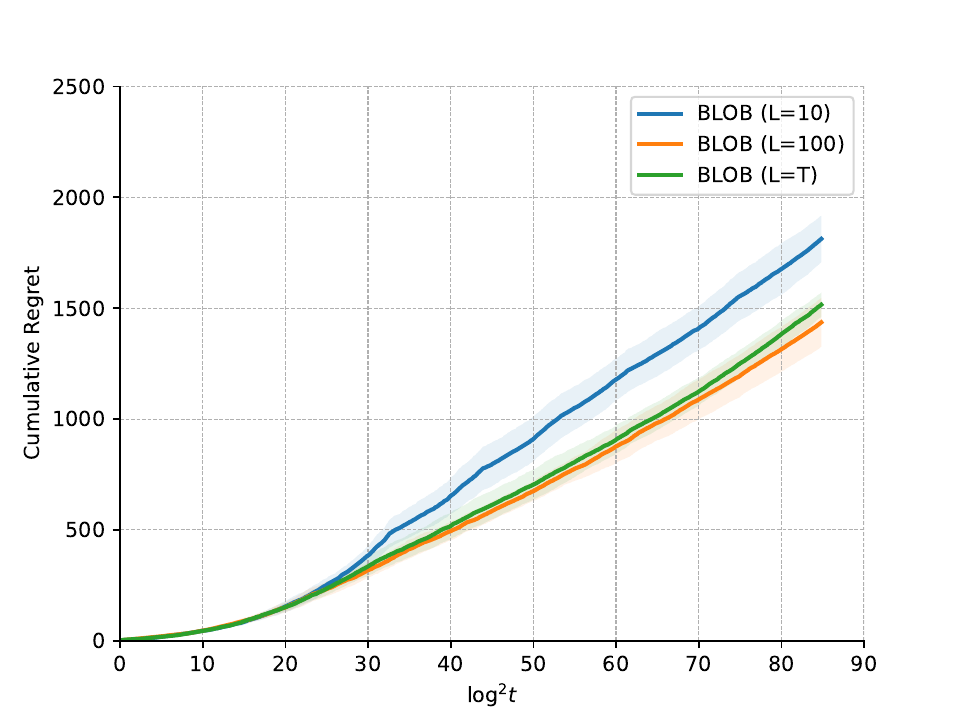}
\label{f:reg_lr2}
}
                                  
\subfigure[$K=5,~d=20,~r=5$]{
\centering
\includegraphics[height=1.6in]{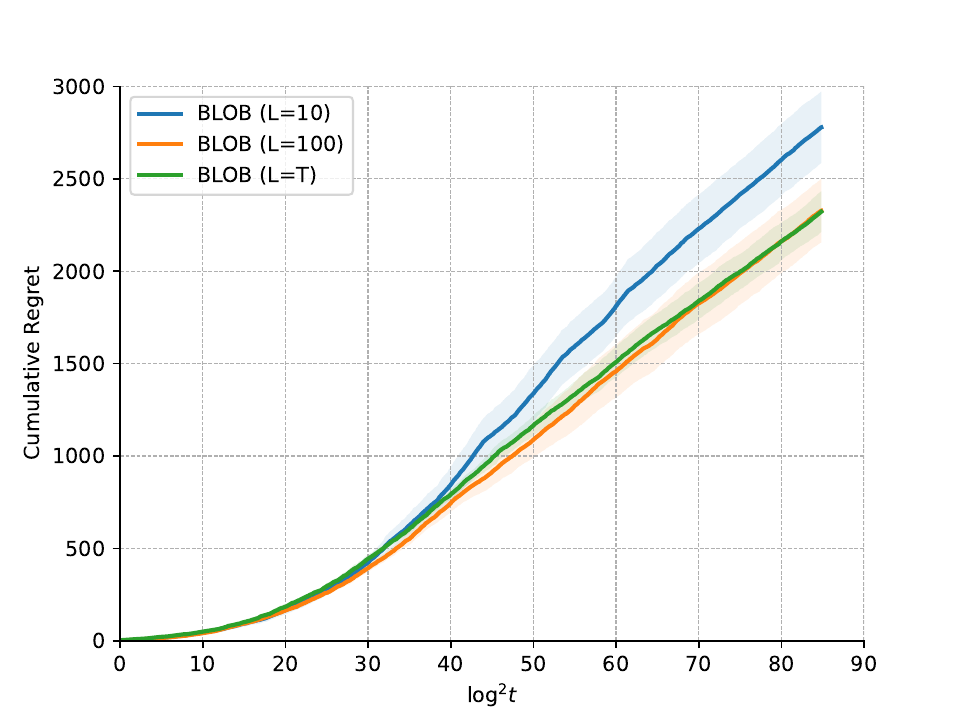}
\label{f:reg_lr3}
}
\subfigure[$K=2,~d=10\sim50,~r=2$]{
\centering
\includegraphics[height=1.6in]{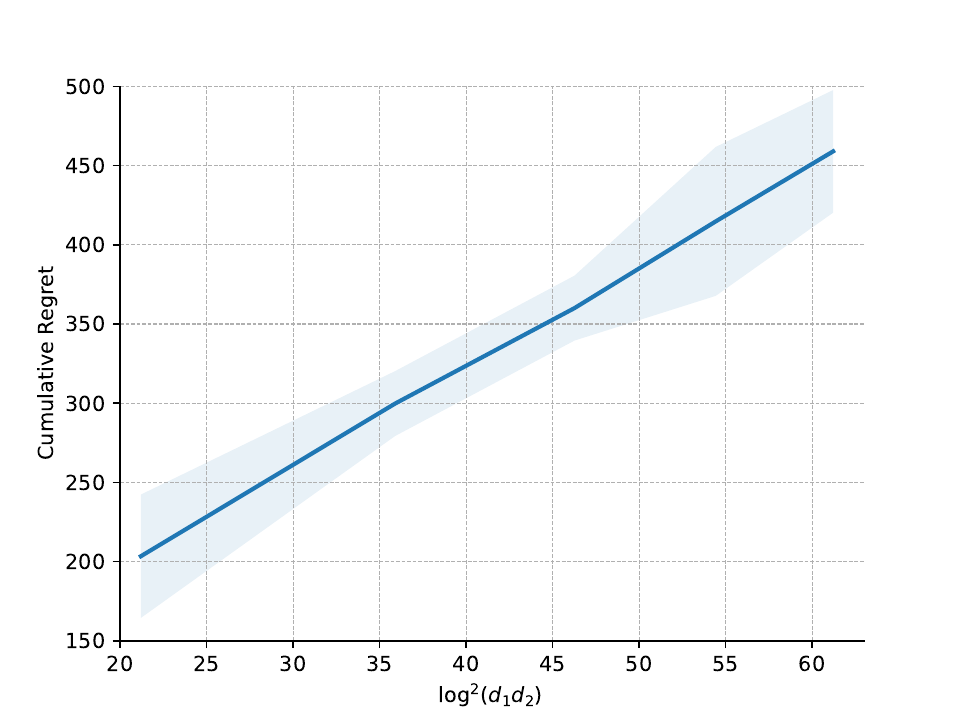}
\label{f:reg_lr4}
}
\centering
\caption{The numerical performance on synthetic data of BLOB under batch $L=10$, $L=100$, and the sequential case ($L=T$). We average the trajectories under $5$ trails and plot the $95\%$ confidence intervals. The results under different values of $K$, $d_1,d_2$, and $r$ show that the cumulative regret of BLOB has an $\cO(\log^2 t)$ dependence on time step $t$ and an $\cO(\log^2(d_1d_2))$ dependence on dimension $d_1,d_2$.}
\label{f:toy_lr}
\end{figure}

\section{Conclusions}
This paper presents two algorithms dealing with the batched version of online learning and the decision-making process in high-dimensional and low-rank matrix data, respectively. We propose a predetermined grid selection to reduce the influence of the batched setting and successfully approximate the regret bound in batched versions to the one in sequential versions. The regret bounds for both proposed algorithms are proved to be $\log^2 T$ in the number of sample sizes $T$ via at least $\Theta(\log T)$ batches. Particularly for low-rank versions, our algorithm first achieves the regret bound with the polynomial logarithmic dependence on both sample sizes and matrix dimensions. In conclusion, the proposed algorithms both maintain high time efficiency and achieve regret bounds close to sequential ones. Finally, the experiments based on synthetic and real data validate that the batched sparse bandit and the batched low-rank bandit perform favorably.

We end by discussing some limitations of the batched sparse bandit and the batched low-rank bandit. First of all, our grid selections are not adaptive and flexible enough to adjust to the performance in previous batches. Second, as mentioned in \citet{bastani2015online}, the forced sampling method applied by our algorithms may bring irreversible consequences. Take medical decision-making as an example, where making decisions randomly sacrifices the patients' health and causes medical tangle. In those cases, it is more appropriate to avoid pure exploration. For example, UCB \citet{auer2002using} explores within the range of confidence sets. For another example, \citet{bastani2021mostly} proposes a greedy-first algorithm to avoid exploration as much as possible.

\bibliographystyle{ims}
\bibliography{graphbib.bib}

\newpage
\begin{appendix}

\section{Proof of Theorem \ref{th:linear}}\label{s:Proof_of_Theorem_1}

We restate Theorem \ref{th:linear} and list the specific choices of the parameters as follows.

\begin{theorem}[Restatement of Theorem \ref{th:linear}]
Suppose that Assumptions \ref{parameterset}-\ref{compatibilitycondition} hold. When $K \ge 2$, $d > 2$, $T \ge C_5$, $L\ge 2$, under Algorithm \ref{algo1} with the grid parameter $a \ge \log t_1$, $\lambda_1 = \phi_0^2 p_{\ast}h/(64s_0 x_{\max})$, $\lambda_{2,l}=(\phi_0^2/s_0)\cdot \sqrt{(\log t_l + \log d)/(2 C_1p_{\ast} t_l)}$ and $\lambda_{2,0}=(\phi_0^2/s_0)\cdot \sqrt{1/(2p_{\ast} C_1 )}$, the cumulative regret up to time $T$ is upper bounded:
$$
\RG(T,L)\le\mathcal{O} \left( K s_0^2 \log d (\log^2 T+\log d) \right),
$$
where the constants $C_1(\phi_0)$, $C_2(\phi_0)$, $C_3(\phi_0, p_{\ast})$, $C_4(\phi_0,p_{\ast})$ and $C_5$ are given by
\begin{gather*}
C_1(\phi_0) = \frac{\phi_0^4}{512 s_0^2 \sigma^2 x_{\max}^2},\quad C_2(\phi_0) = \min \Big( \frac{1}{2},\frac{\phi_0^2}{256 s_0 x_{\max}^2}\Big), \\
C_3(\phi_0, p_{\ast}) = \frac{2C_m R_{\max}(64x_{\max})^2}{p_{\ast}^3 C_1},\quad C_4(\phi_0,p_{\ast})=C_3(\phi_0,p_{\ast})+8\Big(1+\frac{64}{p_{\ast}^2C_2^2}\Big)R_{\max},\\
C_5=\min_{t\in\mathbb{Z}}\left\{ t\ge4K(2C_0+7a+7)\log(t/\tt_0), t\ge2\tt_0(a/\log\tt_0+1)\right\},
\end{gather*}
and we take $C_0 = \max \{10,8/p_{\ast}, 48\log d/(p_{\ast}^2C_2^2),1024\log d x_{\max}^2/(C_1(p_{\ast}^2h^2\}$,
$\tt_0 = \max \{ 7K-1, 2C_0 K-1, \left\lceil (t_0+1)^2/e^2 \right\rceil-1 \}$ and $t_0 = 2C_0K$.
\end{theorem}

In this section, to demonstrate Theorem \ref{th:linear}, we first show the convergence of random-sample estimators and whole-sample estimators in $\S$\ref{ss:Convergence of Random-Sample Estimators} and $\S$\ref{ss:Convergence of Whole-Sample Estimators} respectively. Then, we bound the cumulative regret with the convergence results in $\S$\ref{ss:Bounding the Cumulative Regret}.

\subsection{Convergence of Random-Sample Estimators}\label{ss:Convergence of Random-Sample Estimators}

The convergence inequality of random-sample estimators is stated in Proposition \ref{prop3}. The proof does not follow from Lemma \ref{LASSOtailineq} directly, since according to Assumption \ref{compatibilitycondition}, we have only assumed that the compatibility condition holds for $\Sigma_k=\mathbb{E}_{X\sim\mathcal{P}_X}(XX^{\mathrm{T}}|X\in \cV_k)$ rather than $\mathbb{E}_{X\sim\mathcal{P}_X}(XX^{\mathrm{T}}
)$. We denote the forced-sample size for arm $k\in[K]$ up to batch $l$ by $n_k = \left| \mathcal R_{k,l} \right|$ and the sub-sample set $\{t'\in\mathcal{R}_{k,l}|X_{t'}\in \cV_k\}$ as $\cA^{'}$. We will solve this problem by showing that $\cA^{'}$ is a set of i.i.d. samples from $\mathcal{P}_{X|X\in \cV_k}$, and applying Lemma \ref{LASSOtailniid}.

\begin{proof}[Proof of Proposition \ref{prop3}]
To begin with, we will prove that $\cA^{'}$ is a set of i.i.d. samples from $\mathcal{P}_{X|X\in \cV_k}$. For each $t\in\mathcal{R}_{k,l}$, $X_t$ is drawn randomly from $\mathcal{P}_X$ and therefore with probability at least $p_{\ast}$, $X_t \in \cV_k$, i.e., $t\in\cA^{'}$. Additionally, $\{X_t\in \cV_k\}$ are independent for different values of $t\in\mathcal{R}_{k,l}$, since the original sequence $\{X_t\}_{t\in\mathcal{R}_{k,l}}$ is i.i.d., each $t\in\cA^{'}$, $X_t$ is an i.i.d. sample of $\mathcal{P}_{X|X\in \cU_k}$.

Furthermore, define the event 
$$
\cE_1 =  \big\{|\cA^{'}|/|\cA|\ge p_{\ast}/2\big\} \cap \big\{n_k\ge C_0\log(t_l+1)/2\big\}.
$$
Then, by invoking Lemma \ref{lem:ap005} we have
\#
\PP( \cE_1)&=\PP\big( n_k\ge C_0\log(t_l+1)/2\big)\cdot\PP\big( |\cA^{'}|/|\cA|\ge p_{\ast}/2 \mid n_k\ge C_0\log(t_l+1)/2\big)\nonumber\\
&\le(1-2/(t_l+1))^2\le 1-4/(t_l+1).\label{eq:d124}
\#
Suppose that $\cE_1$ holds, since $C_0 \ge \max \{10,8/p_{\ast}, 48\log d/(p_{\ast}^2C_2^2)\}$ and $48\log d/(p_{\ast}^2C_2^2)>12\log d/(p_{\ast}C_2^2)$, we have $n_k \ge C_0\log(t_l+1)/2 \ge C_0p_{\ast}/2 > 6\log d/p_{\ast}C_2^2$. Thus by combining Lemma \ref{LASSOtailniid} and \ref{ap004}, and by $t_0 = 2C_0K$, $t_l \ge (t_0 +1)^2/e^2 - 1$, $p=p_{\ast}$, $\chi = h/(4x_{\max})$ and $\lambda_1 = (\phi_0^2 p_{\ast}h)/(64 s_0 x_{\max})$, we know that
\$
&\mathbb{P}\Big( \| \betah(\cR_{k,l}, \lambda_1) - \beta_k^{\true} \|_1 \ge \frac{h}{4x_{\max}} \,\Big|\, \cE_1 \Big)\\
&\qquad\le 2\exp\Big(-p_{\ast}^2 C_1(\phi_0)C_0\frac{h^2}{512x_{\max}^2}\log(t_l+1)+\log d\Big)+\exp\Big(-\frac{p_{\ast} C_2(\phi_0)^2C_0\log(t_l+1}{4}\Big)\\
&\qquad\le 2\exp\big( -2\log d\log(t_l+1)+\log d \big) + \frac{1}{t_l+1}
\le\frac{3}{t_l+1},
\$
where the first inequality is deduced by using $C_0 \ge \max \{ 4/(p_{\ast}C_2^2),\log d/C_1\cdot(32x_{\max}/(p_{\ast}h))^2 \}$, and the second inequality is due to $\log d\le\log(t_l+1)(2\log d-1)$. Finally, by invoking \eqref{eq:d124}, we obtain the following result
\$
&\mathbb{P}\Big( \| \betah(\cR_{k,l}, \lambda_1) - \beta_k^{\true} \|_1 \ge \frac{h}{4x_{\max}}\Big)\\
&\qquad\le \PP(\cE_1^c) \PP\Big( \| \betah(\cR_{k,l}, \lambda_1) - \beta_k^{\true} \|_1 \ge \frac{h}{4x_{\max}} \,\Big|\, \cE_1^c \Big) + \PP(\cE_1) \PP\Big( \| \betah(\cR_{k,l}, \lambda_1) - \beta_k^{\true} \|_1 \ge \frac{h}{4x_{\max}} \,\Big|\, \cE_1 \Big)\\
&\qquad\le \PP(\cE_1^c) + \PP\Big( \| \betah(\cR_{k,l}, \lambda_1) - \beta_k^{\true} \|_1 \ge \frac{h}{4x_{\max}} \,\Big|\, \cE_1 \Big) \le \frac{4}{t_l+1} + \frac{3}{t_l+1} = \frac{7}{t_l+1}.
\$
Therefore, we conclude the proof of Proposition \ref{prop3}.
\end{proof}

\subsection{Convergence of Whole-Sample Estimators}\label{ss:Convergence of Whole-Sample Estimators}
To demonstrate the convergence of whole-sample estimators, we first show in $\S$\ref{sss:Proof of Proposition 5.3} that sufficient i.i.d. samples are guaranteed in the sample set for each optimal arm $k\in\cK_o$. Then, we invoke Lemma \ref{LASSOtailniid} to attain the tail inequality for whole-sample estimators in $\S$\ref{sss:Proof of Proposition 5.4}.

\subsubsection{A General Form}\label{sss:Proof of prop4general}
\begin{proof}[Proof of Proposition \ref{prop4general}]
Recall in $\S$\ref{sub403} that we define $M_k(t)$ as
$$
M_k(t)= \mathbb{E}\Big(\sum_{m=1}^{l}\sum_{\tau=t_{m-1}+1}^{t_m}\mathbbm{1}(X_{\tau}\in \cU_k,\mathcal A_{m-1},X_{\tau}\notin \mathcal{R}_{k,m}) \,\Big|\, \mathcal{F}_t\Big).
$$
To begin with, for each optimal arm $k\in\cK_o$, we will prove that if $X_t\in\mathcal{B}_{k,m}$, the arm $k$ will definitely be chosen, and furthermore the components of $\mathcal{B}_{k,m}$ are i.i.d.(with distribution $\mathcal{P}_{X|X\in \cU_k}$). For any $j\ne k$, since $X_t\in\cU_k$ and $\cA_{m-1}$ holds, we have
\$
&X_t^{\top}\big( \betah(\cR_{k,l}, \lambda_1) -\betah(\cR_{j,l}, \lambda_1)\big)\\ &\qquad=X_t^{\top}\big( \betah(\cR_{k,l}, \lambda_1)-\beta_k^{\true}\big) - X_t^{\top}\big( \betah(\cR_{j,l}, \lambda_1)-\beta_j^{\true}\big) + X_t^{\top}\big( \beta_i^{\true}-\beta_j^{\true}\big)\\
&\qquad\ge -x_{\max}\cdot\frac{h}{4x_{\max}} -x_{\max}\cdot\frac{h}{4x_{\max}} +h\ge h/2,
\$
which means that
$$
X^{\mathrm{T}}\betah(\cR_{k,l}, \lambda_1) \ge\max_{j\ne k} \big( X^{\mathrm{T}}\betah(\cR_{j,l}, \lambda_1) \big)+\frac{h}{2}.
$$
As a result, the agent will only select arm $k$ at the first step. Additionally, since $\mathcal A_{m-1}$ only relies on samples in $\mathcal{R}_{k,m-1}$, the $\mathcal A_{m-1}$ is independent of $X_{\tau}$. Therefore, random variables $\{X_{\tau}|\mathcal A_{m-1}\}$ are i.i.d. samples from $\mathcal{P}_X$. Now, since the presence of each $X_{\tau}$ in $\cU_k$ or not in $\mathcal{R}_{k,m}$ is simply rejection sampling, $X_{\tau}$ in $\mathcal{B}_m$ is distributed i.i.d. from $\mathcal{P}_{X|X\in \cU_k}$. Therefore, we have $M_k(t_{l})=\,|\,\bigcup_{m=1}^{l}\mathcal{B}_m\,|\,= \sum_{m=1}^{l}\sum_{\tau=t_{m-1}+1}^{t_m}\mathbbm{1}(X_{\tau}\in \cU_k,\mathcal A_{m-1},X_{\tau}\notin \mathcal{R}_{k,l})$ which counts a part of i.i.d. samples of the whole-sample set up to batch $l$. Moreover, for the reason that 
\$
\mathbb{E}\big( M_k(t) \,|\, \mathcal{F}_{t-1} \big) &= \mathbb{E}\Big( \big( \sum_{m=1}^{l}\sum_{\tau=t_{m-1}+1}^{t_m}\mathbbm{1}(X_{\tau}\in \cU_k,\mathcal A_{m-1},X_{\tau}\notin \mathcal{R}_{k,m}) \,\Big|\, \mathcal{F}_{t} \big) \,\big|\, \mathcal{F}_{t-1} \Big)\\
&= \mathbb{E}\Big(\sum_{m=1}^{l}\sum_{\tau=t_{m-1}+1}^{t_m}\mathbbm{1}(X_{\tau}\in \cU_k,\mathcal A_{m-1},X_{\tau}\notin \mathcal{R}_{k,m}) \,\Big|\, \mathcal{F}_{t-1}\Big)\\
&= M_k(t-1),
\$
and that covariates $X_t$ are generated randomly, we know that $\{M_k(t)\}$ is a martingale with $\,|\, M_k(t+1)-M_k(t)\,|\,\le1$, we can use $M(0)$ to bound the value of $M_k(t_{l+1})$ with Azuma’s inequality:
\$
\mathbb{P}\Big(|M_k(t_{l})-M_k(0)| \geq \frac{1}{2} M_k(0)\Big) \leq \exp \Big(\frac{-M_k(0)^{2} / 4}{2(t_{l}+1)}\Big).
\$
So we further get:
\#
 \mathbb{P}\Big(M_k(t_{l}) \leq \frac{1}{2} M_k(0)\Big) \leq \exp \Big(\frac{-M_k(0)^{2} / 4}{2(t_{l}+1)}\Big).\label{eq:mktail}
\#
Then, we can express $M_k(0)$ as follows:
\$
M_k(0)&=\mathbb{E}\Big(\sum_{m=1}^{l}\sum_{\tau=t_{m-1}+1}^{t_m}\mathbbm{1}(X_{\tau}\in \cU_k,A_{m-1},X_{\tau}\notin \mathcal{R}_{k,l})\Big)\nonumber\\
&=\sum_{m=1}^{l}\sum_{\tau=t_{m-1}+1}^{t_m}\mathbb{P}(X_{\tau}\in \cU_k,A_{m-1},X_{\tau}\notin \mathcal{R}_{k,l})\nonumber\\
&=\sum_{m=1}^{l}\sum_{\tau=t_{m-1}+1}^{t_m}\mathbb{P}(X_{\tau}\in \cU_k)\mathbb{P}(A_{m-1})\mathbb{P}(X_{\tau}\notin \mathcal{R}_{k,l}),
\$
where the last equality comes from the fact that $\{X\in \cU_k\}$ is independent of $\{\cA_{m-1},X_{\tau}\notin \mathcal{R}_{k,l}\}$ and that $\{X_{\tau}\notin \mathcal{R}_{k,l}\}$ is independent of $\{\cA_{m-1}\}$. Additionally, by invoking Proposition \ref{prop3} with $t_m\ge(t_0+1)^2/e^2-1$ for all arms in $[K]$, we have $\PP(\cA_m)\ge1-7K/(t_m+1)$.

Then, let $l_0=\min\{ l'\ge1 \mid t_{l'}\ge\tt_0\}$, where $\tt_0=\max \left\{ 7K-1, 2C_0 K-1, \left\lceil (t_0+1)^2/e^2 \right\rceil-1 \right\}$. If $l_0\ge 2$, we have 
\$
t_{l_0}=g(t_{l_0-1})\le g(\tt_0)\le t_l/2,
\$
where the first inequality hold since $g$ is an increasing function and $t_{l_0-1}\le\tt_0$, and the second inequality uses $t_l\ge C_5'\ge2g(\tt_0)$.

If $l_0=1$, we write that 
\$
t_l\ge t_2=g(t_1)\ge 2t_1\ge 2t_0.
\$
Then, the $M_k(0)$ can be lower bounded by
\#
M_k(0)\ge p_{\ast}\sum_{m=l_0+1}^{l}\Big(1-\frac{7K}{t_{m-1}+1}\Big)\sum_{\tau=t_{m-1}+1}^{t_m}\Big(1-\frac{2C_0K}{\tau}\Big).\label{eq:mk0bound0}
\#
By invoking Lemma \ref{lem:monint}, it follows that
\$
\sum_{\tau=t_{m-1}+1}^{t_m}\Big(1-\frac{2C_0K}{\tau}\Big)\ge\int_{t_{m-1}}^{t_m}\Big(1-\frac{2C_0K}{\tau}\Big)\md\tau =t_m-t_{m-1}-2C_0K\log\Big(\frac{t_m}{t_{m-1}}\Big),
\$
which is taking back into \eqref{eq:mk0bound0} to obtain that
\#
M_k(0)&\ge p_{\ast}\sum_{m=l_0+1}^{l}\Big(1-\frac{7K}{t_{m-1}+1}\Big)\cdot\Big( t_m-t_{m-1}-2C_0K\log\Big(\frac{t_m}{t_{m-1}}\Big)\Big)\nonumber\\
&\ge p_{\ast}\sum_{m=l_0+1}^{l}\Big(1-\frac{7K}{t_{m-1}+1}\Big)\cdot\left( t_m-t_{m-1}\right)-2p_{\ast}C_0K\log\Big(\frac{t_l}{t_{l_0}}\Big)\nonumber\\
&\ge p_{\ast}\sum_{m=l_0+1}^{l}\Big(1-\frac{7K}{t_{m}/\alpha+1}\Big)\cdot\left( t_m-t_{m-1}\right)-2p_{\ast}C_0K\log\Big(\frac{t_l}{t_{l_0}}\Big),\label{eq:gmk0bound1}
\#
where the last inequality is due to $t_m=g(t_{m-1})\le\alpha t_{m-1}$. Similar to Lemma \ref{lem:monint}, we have
\$
\sum_{m=l_0+1}^{l}\Big(1-\frac{7K}{t_{m}/\alpha+1}\Big)\left( t_m-t_{m-1}\right) &\ge\int_{t_{l_0}}^{t_l}\Big(1-\frac{7K}{t/\alpha+1}\Big)\md t\\
&= t_l-t_{l_0}-7K\alpha\log\Big(\frac{t_l+\alpha}{t_{l_0}+\alpha}\Big).
\$
By taking the result above back into \eqref{eq:gmk0bound1}, we get
\#
M_k(0)&\ge p_{\ast}\Big((t_l-t_{l_0})-7K\alpha\log\Big(\frac{t_l+\alpha}{t_{l_0}+\alpha}\Big)-2C_0K\log\Big(\frac{t_l}{t_{l_0}}\Big)\Big)\nonumber\\
&\ge p_{\ast}\Big(\frac{t_l}{2}-2K(C_0+3\alpha)\log\Big(\frac{t_l}{t_{l_0}}\Big)\Big) \ge \frac{p_{\ast}t_l}{4},\label{eq:gmk0bound2}
\#
where the second inequality is by applying $t_{l_0}\le t_l/2$, and the last inequality is due to $t_l\ge C_5^{'}$. Eventually, we obtain the desired inequality by combining \eqref{eq:mktail} and \eqref{eq:gmk0bound2}.
\end{proof}

\subsubsection{Proof of Corollary \ref{prop4}}\label{sss:Proof of Proposition 5.3}
\begin{proof}[Proof of Corollary \ref{prop4}]
It is obvious that
$$
g(t_l)=\Big\lfloor \Big(\frac{a}{l\log t_l}+1\Big)t_l\Big\rfloor
$$
is non-decreasing. Additionally, we have
\$
t_{l-1} \ge \frac{t_l}{\alpha/((l-1)\log t_{l-1})+1} \ge \frac{t_l}{\alpha+1}.
\$
Moreover, since $a \ge \log t_1$, we have 
\$
t_l \ge t_2 = \Big\lfloor \big( \frac{a}{\log t_1} +1 \big) t_1 \Big\rfloor \ge 2t_1.
\$
Hence, by invoking Proposition \ref{prop4} with $\alpha'=\alpha+1$, we derive the result.
\end{proof}

\subsubsection{Proof of Proposition \ref{prop5}}\label{sss:Proof of Proposition 5.4}
\begin{proof}[Proof of Proposition \ref{prop5}]
Since 
$$
\Big|\bigcup_{m=1}^{l}\mathcal{B}_m\Big| =M_k(t_{l})=\sum_{m=1}^{l}\sum_{\tau=t_{m-1}+1}^{t_m}\mathbbm{1}(X_{\tau}\in \cU_k,\mathcal A_{m-1},X_{\tau}\notin \mathcal{R}_{k,l}),
$$
Applying Proposition \ref{prop4}, we have
\$
\mathbb{P}\Big(\Big|\bigcup_{m=1}^{l}\mathcal{B}_m\Big|\ge\frac{p_{\ast}t_l}{8}\Big)\ge 1-\exp\Big(-\frac{p_{\ast}^2(t_{l}-1)}{128}\Big).
\$
Let $\mathcal{A}=\mathcal{W}_{k,l}$, $\cA^{'}=\cup_{m=1}^{l}\mathcal{B}_m$, $p=p_{\ast}/4$. Since $|\cA|\ge |\cA'| \ge p_*t_l/8$, $t_l\ge C_5\ge 4C_0$ and $C_0\ge 48\log d/(p_*^2C_2^2)$, we have $|\cA|$ satisfies $|\cA|\ge 6\log d/(pC_2^2)$. Therefore, we can apply Lemma \ref{LASSOtailniid} with $\lambda=\chi\phi_0^2p_{\ast}/(16s_0)$ to obtain the following result:
\$
\mathbb{P}\Big(\| \betah(\cW_{k,l}, \lambda) - \beta_k^{\true}\|_1 > \chi\Big)\le& 2\exp\Big(-C_1\Big(\frac{\phi_0\sqrt{p_{\ast}/4}}{2}\Big)\frac{t_lp_{\ast}}{8}\chi^2+\log d\Big)+\exp \Big(-\frac{p_{\ast}^2C_2^2(\phi_0)t_l}{64}\Big)\\
&\qquad +\exp \Big(-\frac{p_{\ast}^2(t_{l}-1)}{128}\Big)\\
\le& 2\exp\Big(-\frac{t_lp_{\ast}^3C_1(\phi_0)}{2048}\chi^2+\log d\Big)+2\exp \Big(-\frac{p_{\ast}^2C_2^2(\phi_0)t_l}{64}\Big),
\$                                                                           
where the last inequality is deduced by using $C_2(\phi_0)\le 1/2$ and $t_l-1 \ge t_1 /2$. Taking 
$$
\chi=32\sqrt{\frac{2(\log t_l+\log d)}{t_lp_{\ast}^3C_1(\phi_0)}},
$$
we get
$$
\mathbb{P}\Big(\| \betah(\cW_{k,l}, \lambda_{2,l}) - \beta_k^{\true}\|_1 \ge 32\sqrt{\frac{2(\log t_l +\log d)}{t_lp_{\ast}^3C_1(\phi_0)}}\Big)\le\frac{2}{t_l}+2\exp\Big(-\frac{t_lp_{\ast}^2C_2^2(\phi_0)}{64}\Big),
$$
where the choice of $\chi$ implies $\lambda=\lambda_{2,l}$. Hence, we conclude the proof of Proposition \ref{prop5}.
\end{proof}

\subsection{Bounding the Cumulative Regret}\label{ss:Bounding the Cumulative Regret}
\begin{proof}[Proof of Theorem \ref{th:linear}]
Recall that the cumulative regret up to time $T$, batch $L$ is divided into three groups:
\#\label{eq:regret decompose 3 parts}
\RG(T,L) =& \underbrace{\sum_{l=1}^L\sum_{t=t_{l-1}+1}^{t_l} \rg_{t,l}\mathbbm{1}\big(l\le l_0~\text{or}~X_t\in\mathcal{R}_{k,L}:~ k \in [K]\big)}_{\RG_1(T,L)}\notag\\
&\qquad +\underbrace{\sum_{l=1}^L\sum_{t=t_{l-1}+1}^{t_l} \rg_{t,l}\mathbbm{1}\big(l> l_0~\text{and}~\mathcal A_{l-1}^c\big)}_{\RG_2(T,L)} + \underbrace{\sum_{l=1}^L\sum_{t=t_{l-1}+1}^{t_l} \rg_{t,l}\mathbbm{1}\big(l> l_0~\text{and}~\mathcal A_{l-1}\big)}_{\RG_3(T,L)}.
\#
In the sequel, we will bound the regrets in each group separately.

\vspace{4pt}
\noindent
\textbf{Regrets in Group 1.}\\
We bound $RG_1(T,L)$ as
\$
\RG_1(T,L)\le R_{\max}\Big(t_{l_0}+\sum_{t=t_{1}+1}^{T}\mathbbm{1}(X_t\in\mathcal{R}_{k,L},k\in[K])\Big)\le R_{\max}\Big(t_{l_0}+\sum_{k=1}^{K}n_k\Big).
\$
From Lemma \ref{ap004}, we find that if $t_0 = 2C_0K$, $C_0 =\max\{10, 8/p_{\ast}\}$, and $T \ge (t_0 +1)^2/e^2 - 1$, the following inequality can be obtained:
\$
\mathbb{P}\left(n_k>6C_0\log T\right)\le\frac{2}{T+1},
\$
which indicates that
\#\label{eq:P(sum n_k > log T)}
\mathbb{P}\big(\sum_{k=1}^{K}n_k>6C_0K\log T\big)\le\sum_{k=1}^{K}\mathbb{P}\big(n_k>6C_0K\log T\big)\le\frac{2K}{T+1}.
\#
Thus we have
\$
\RG_{1}(T,L) \leq& R_{\max }\big(t_{l_0}+\sum_{k=1}^{K}n_k\big)\nonumber\\
=&R_{\max }\Big(\sum_{k=1}^{K} n_{k}\,\Big|\,\sum_{k=1}^{K} n_{k}>6 C_{0}K\log T\Big) \mathbb{P}\Big(\sum_{k=1}^{K} n_{k}>6 C_{0}K \log T\Big) \nonumber\\
&\qquad+R_{\max }\Big(\sum_{k=1}^{K} n_{k}\,\Big|\,\sum_{k \in \mathcal{K}} n_{k} \leq 6 C_{0}K \log T\Big) \mathbb{P}\Big(\sum_{k=1}^{K} n_{k} \leq 6 C_{0}K \log T\Big)+R_{\max } t_{l_0}.
\$
By combining the inequality above and \eqref{eq:P(sum n_k > log T)}, we get
\#\label{eq:RG_1 linear}
\RG_{1}(T,L) \leq& R_{\max } T \frac{2K}{T+1}+R_{\max } 6 C_{0}K\log T\big(1-\frac{2K}{T+1}\big)+R_{\max }t_{l_0} \nonumber\\
\leq& R_{\max }K\left(2+6 C_{0} \log T\right)+R_{\max} t_{l_0}.
\#

\vspace{4pt}
\noindent
\textbf{Regrets in Group 2.}\\
From Proposition \ref{prop3}, we have
\$
\mathbb{P}\Big( \| \hat{\beta}(\mathcal{R}_{k,l},\lambda_1)-\beta_k\|_1 \le \frac{h}{4x_{\max}}\Big)\ge 1-\frac{7}{t_l+1},~\forall k\in[K],
\$
which implies that
\$
\mathbb{P}\left( \mathcal A_l\right)\ge1-\frac{7K}{t_l+1}.
\$
Therefore, $\RG_2(T,L)$ can be bounded as follows:
\#\label{eq:RG_2 linear}
\RG_2(T,L)&\le\mathbb{E}\Big(\sum_{l=l_0}^{L}(t_l-t_{l-1})\mathbbm{1}(\mathcal A_{l-1}^c)R_{\max}\Big) \le \sum_{l=l_0}^{L}(t_l-t_{l-1})\frac{7K}{t_{l-1}+1}R_{\max}\notag\\
&\le 7KR_{\max}\sum_{l=l_0}^L\frac{(a t_{l-1})/((l-1)\log t_{l-1})}{t_{l-1}+1} \le 7KR_{\max}a\sum_{l=l_0}^L\frac{1}{l-1}\notag\\
&\le 7KR_{\max}a(\log L+1),
\#
where the third equality comes from the grid structure:
$$
t_l=\Big\lfloor (\frac{a}{(l-1)\log t_{l-1}}+1)t_{l-1}\Big\rfloor,~l=2,\ldots,L-1.
$$

\vspace{4pt}
\noindent
\textbf{Regrets in Group 3.}\\
Without loss of generality, at time $t$ in batch $l+1$ ($t=t_l+1,\ldots,t_{l+1}$), it is assumed that arm $i$ is true optimal arm. Then, we can bound the regret $\rg_t$ as follows:
\#\label{apeg04}
\rg_{t,l} &=\mathbb{E}\Big(\sum_{j=1}^K\mathbbm{1}\big(j=\arg \max _{k \in [K]} \big(X_{t}^{\mathrm{T}}\betah(\cW_{k,l}, \lambda_{2,l})\big)\big)\cdot\big(X^{\mathrm{T}}_{t}\beta_{i}^{\text {true}}-X^{\mathrm{T}}_{t}\beta_{j}^{\text {true}}\big)\Big) \nonumber\\
& \leq \mathbb{E}\Big(\sum_{j \neq i} \mathbbm{1}\big(X^{\mathrm{T}}_{t}\betah(\cW_{j,l}, \lambda_{2,l}) > X^{\mathrm{T}}_{t}\betah(\cW_{i,l}, \lambda_{2,l})\big)\cdot\big(X^{\mathrm{T}}_{t}\beta_{i}^{\text {true}}-X^{\mathrm{T}}_{t}\beta_{j}^{\text {true}}\big)\Big).
\#
Let $\mathcal{E}(t, \delta)_{1, k}$ denote the event for $X_t$ that at time $t$ the true reward of arm $k$ is at least $\delta$ less than that of arm $i$, which means that $\mathcal{E}(t, \delta)_{1, k}=\left\{X^{\mathrm{T}}_{t}\beta_{i}^{\text {true}}>X^{\mathrm{T}}_{t}\beta_{k}^{\text {true}}+\delta\right\}$, for $k \neq i, k \in [K]$. Then we have the following bound:
\begin{align}
\rg_{t,l} \le& \mathbb{E}\Big(\sum_{j \neq i} \mathbbm{1}\big(\big\{X^{\mathrm{T}}_{t}\betah(\cW_{i,l}, \lambda_{2,l}) > X^{\mathrm{T}}_{t}\betah(\cW_{j,l}, \lambda_{2,l})\big\}\cap \mathcal{E}(t, \delta)_{1, j}\big) \cdot \big(X^{\mathrm{T}}_{t}\beta_{i}^{\text {true}}-X^{\mathrm{T}}_{t}\beta_{j}^{\text {true}}\big)\Big) \label{apeg05}\\
&\qquad + \mathbb{E}\Big(\sum_{j \neq i} \mathbbm{1}\big(\big\{X^{\mathrm{T}}_{t}\betah(\cW_{j,l}, \lambda_{2,l})>X^{\mathrm{T}}_{t}\betah(\cW_{i,l}, \lambda_{2,l})\big\}\cap \mathcal{E}(t, \delta)_{1, j}^c\big) \cdot \big(X^{\mathrm{T}}_{t}\beta_{i}^{\text {true}}-X^{\mathrm{T}}_{t}\beta_{j}^{\text {true}}\big)\Big). \label{apeg06}
\end{align}
We bound the term in \eqref{apeg06} as follows:
\$
&\mathbb{E}\Big(\sum_{j \neq i} \mathbbm{1}\big(\big\{X^{\mathrm{T}}_{t}\betah(\cW_{j,l}, \lambda_{2,l})>X^{\mathrm{T}}_{t}\betah(\cW_{i,l}, \lambda_{2,l})\big\}\cap \mathcal{E}(t, \delta)_{1, j}^c\big)\Big) \cdot \delta\\
&\qquad \le\mathbb{E}\Big(\sum_{j \neq i} \mathbbm{1}\left(\mathcal{E}(t, \delta)_{1, j}^c\right)\Big) \cdot \delta=\sum_{j \neq i}\mathbb{P}\left(\mathcal{E}(t, \delta)_{1, j}^c\right) \cdot \delta\\
&\qquad =(K-1)C_mR_{\max}\delta^2\le C_mR_{\max}K\delta^2,
\$
where the last inequality comes from Assumption \ref{margincondition}.

Now we consider the term in \eqref{apeg05}, which can be bounded as follows:
\begin{align*}
&\mathbb{E}\Big(\sum_{j \neq i} \mathbbm{1}\big(\left\{X^{\mathrm{T}}_{t}\betah(\cW_{j,l}, \lambda_{2,l})>X^{\mathrm{T}}_{t}\betah(\cW_{i,l}, \lambda_{2,l})\right\}\cap \mathcal{E}(t, \delta)_{1, j}\big)\Big) \cdot(2R_{\max}) \nonumber\\
&\qquad\le \mathbb{E}\Big(\sum_{j \neq i} \mathbbm{1}\big(X^{\mathrm{T}}_{t}\betah(\cW_{j,l}, \lambda_{2,l})-X^{\mathrm{T}}_{t}\beta_{j}^{\text {true}}>X^{\mathrm{T}}_{t}\betah(\cW_{i,l}, \lambda_{2,l})-X^{\mathrm{T}}_{t}\beta_{i}^{\text {true}}+\delta\big)\Big) \cdot(2R_{\max}) \nonumber\\
&\qquad\le \mathbb{E}\Big(\sum_{j \neq i} \mathbbm{1}\big(\left|X^{\mathrm{T}}_{t}\betah(\cW_{j,l}, \lambda_{2,l})-X^{\mathrm{T}}_{t}\beta_{j}^{\text {true}}\right|>-\left|X^{\mathrm{T}}_{t}\betah(\cW_{i,l}, \lambda_{2,l})-X^{\mathrm{T}}_{t}\beta_{i}^{\text {true}}\right|+\delta\big)\Big) \cdot(2R_{\max}).
\end{align*}
By applying H\"older's inequality to the right hand side of the last inequality above, we get
\begin{align}
&\mathbb{E}\Big(\sum_{j \neq i} \mathbbm{1}\big(\left\{X^{\mathrm{T}}_{t}\betah(\cW_{j,l}, \lambda_{2,l})>X^{\mathrm{T}}_{t}\betah(\cW_{i,l}, \lambda_{2,l})\right\}\cap \mathcal{E}(t, \delta)_{1, j}\big)\Big) \cdot(2R_{\max}) \nonumber\\
&\qquad\le \mathbb{E}\Big(\sum_{j \neq i} \mathbbm{1}\big(x_{\max}\|\betah(\cW_{j,l}, \lambda_{2,l})-\beta_{j}^{\text {true}}\|_1>-x_{\max}\|\betah(\cW_{i,l}, \lambda_{2,l})-\beta_{i}^{\text {true}}\|_1+\delta\big)\Big) \cdot(2R_{\max}) \nonumber\\
&\qquad\le \mathbb{E}\Big(\sum_{j \neq i} \mathbbm{1}\big(\|\betah(\cW_{j,l}, \lambda_{2,l})-\beta_{j}^{\text {true}}\|_1+\|\betah(\cW_{i,l}, \lambda_{2,l})-\beta_{i}^{\text {true}}\|_1>\frac{\delta}{x_{\max}}\big)\Big) \cdot(2R_{\max}) \nonumber\\
&\qquad\le 2R_{\max} \sum_{j \neq i} \Big( \mathbb{P} \big( \|\betah(\cW_{j,l}, \lambda_{2,l})-\beta_{j}^{\text {true}}\|_1 > \frac{\delta}{2x_{\max}} \big) + \mathbb{P} \big( \|\betah(\cW_{i,l}, \lambda_{2,l})-\beta_{i}^{\text {true}}\|_1 > \frac{\delta}{2x_{\max}} \big) \Big). \label{apeg08}
\end{align}

Then, from Proposition \ref{prop5}, we have the following inequality:
\#\label{apeg09}
\mathbb{P}\Big(\| \betah(\cW_{k,l}, \lambda_{2,l}) - \beta_k^{\true}\|_1 \ge 32\sqrt{\frac{2(\log t_l +\log d)}{t_lp_{\ast}^2C_1(\phi_0)}}\Big)\le\frac{2}{t_l}+2\exp\Big(-\frac{t_lp_{\ast}^3C_2^2(\phi_0)}{64}\Big).
\#
By combining \eqref{apeg08} and \eqref{apeg09} and setting $\delta=64x_{\max}\sqrt{2(\log t_l+\log d)/(t_lp_{\ast}^3C_1(\phi_0))}$, we obtain:
\$
&2R_{\max} \sum_{j \neq i} \Big( \mathbb{P} \big( \|\betah(\cW_{j,l}, \lambda_{2,l})-\beta_{j}^{\text {true}}\|_1 > \frac{\delta}{2x_{\max}} \big) + \mathbb{P} \big( \|\betah(\cW_{i,l}, \lambda_{2,l})-\beta_{i}^{\text {true}}\|_1 > \frac{\delta}{2x_{\max}} \big) \Big)\\
&\qquad\le 2R_{\max}(K-1) \Big( \frac{4}{t_l}+4\exp\big(-\frac{t_lp_{\ast}^2C_2^2(\phi_0)}{64}\big) \Big).
\$
Then, the following result can be deduced:
\$
\begin{aligned}
\rg_t &\le  2R_{\max}K \Big( \frac{4}{t_l}+4\exp\big(-\frac{t_lp_{\ast}^2C_2^2(\phi_0)}{64}\big) \Big) + C_mR_{\max}K\delta^2\\
&\le R_{\max}K\Big(\frac{8}{t_l}+8\exp\big(-\frac{t_lp_{\ast}^2C_2^2(\phi_0)}{64}\big)\Big)+C_mR_{\max}K 64^2 x_{\max}^2\frac{2(\log t_l+\log d)}{t_lp_{\ast}^3C_1(\phi_0)}\\
&=R_{\max}K\Big(\frac{8}{t_l}+8\exp\big(-\frac{t_lp_{\ast}^2C_2^2(\phi_0)}{64}\big)+C_3^{'}\frac{\log t_l+\log d}{t_l}\Big),
\end{aligned}
\$
where $C_3^{'}=2(64x_{\max})^2C_m/(p_{\ast}^3C_1(\phi_0))$. 

We find that the upper bound of time in the same batch is equal. Therefore, we bound the third part of the regret as follows:
\$
\RG_3(T,L)&\le\sum_{l=l_0}^{L-1}R_{\max}K\Big(\frac{8}{t_l}+8\exp\big(-\frac{t_lp_{\ast}^2C_2^2(\phi_0)}{64}\big)+C_3^{'}\frac{\log t_l+\log d}{t_l}\Big)\cdot(t_{l+1}-t_l)\\
&\le \sum_{l=l_0}^{L-1}R_{\max}K\Big(\frac{8 + C_3^{'} \log d}{t_l}+8\exp\big(-\frac{t_lp_{\ast}^2C_2^2(\phi_0)}{64}\big)+C_3^{'}\frac{\log t_l}{t_l}\Big)\cdot\frac{at_l}{l \log t_l}\\
&=R_{\max}Ka\sum_{l=l_0}^{L-1}\Big(\frac{8+ C_3^{'} \log d + C_3^{'}}{l}+8\exp\big(-\frac{t_lp_{\ast}^2C_2^2(\phi_0)}{64}\big)\frac{t_l}{l}\Big).
\$
We continue to compute that
\#\label{eq:RG_3 linear}
\RG_3(T,L)
&\le R_{\max}Ka\sum_{l=l_0}^{L-1}\Big(\frac{8+ C_3^{'} \log d + C_3^{'}}{l}+\frac{512 }{p_{\ast}^2C_2^2l}\Big)\notag\\
&\le R_{\max}Ka\Big(C_3^{'}\log d + C_3^{'} + 8+\frac{512}{p_{\ast}^2C_2^2} \Big)\cdot (\log L+1)\notag\\
&=Ka \big( C_3(\phi_0,p_{\ast})\log d+C_4(\phi_0, p_{\ast}) \big) \cdot(\log L+1),
\#
where $C_3(\phi_0,p_{\ast})=2C_m R_{\max}(64x_{\max})^2/(p_{\ast}^3C_1(\phi_0))$ and $C_4(p_{\ast})= R_{\max}( 512/(p_{\ast}^2C_2^2(\phi_0)) + C_3^{'} + 8 )$.

Finally, the regret bound can be obtained by combining the bounds \eqref{eq:RG_1 linear}, \eqref{eq:RG_2 linear} and \eqref{eq:RG_3 linear} for these three groups and using $a = \mathcal O ((\log T)^2/(\log L))$:
$$
\begin{aligned}
\RG(T,L)=&\RG_1(T,L)+\RG_2(T,L)+\RG_3(T,L)\\
\le& R_{\max }K\left(2+6 C_{0} \log T\right)+R_{\max } t_{l_0}+7KR_{\max}a(\log L+1)\\
&\qquad+Ka \left( C_3(\phi_0,p_{\ast})\log d+C_4(\phi_0, p_{\ast}) \right) (\log L+1)\\
=&\mathcal{O}\left(Ks_0^2\log d (\log^2 T+\log d) \right).
\end{aligned}
$$
Ultimately, we obtain the desired result of Theorem \ref{th:linear}.
\end{proof}

\section{Proof of Theorem \ref{th:matirx}}\label{s:Proof_of_Theorem_2}
The following theorem is the restatement of Theorem \ref{th:matirx} with details about the constants.

\begin{theorem}[Restatement of Theorem \ref{th:matirx}]
Suppose that Assumptions \ref{as0501}-\ref{as0504} hold. Under the batched low-rank bandit, we take $\lambda_1 = \phi_0^2 p_{\ast}h/(48r x_{\max})$, $\lambda_{2,l}=2\phi_0^2/(3r)\cdot \sqrt{2(\log t_l + \log(d_1+d_2))/(p_{\ast} C_1 t_l)}$ and $\lambda_{2,0}=2 \phi_0^2/(3r) \sqrt{2/(p_{\ast} C_1)}$. When $K \ge 2$, $d_1,d_2 > 1$, $T \ge C_5$ and $L\ge 2$, the cumulative regret up to time $T$ is upper bounded:
\#
\RG(T,L)\le\mathcal{O}\left( K r^2 \log(d_1d_2) (\log^2 T+\log(d_1d_2))  \right),
\notag
\#
where the constants $C_1(\phi_0)$, $C_2(\phi_0)$, $C_3(\phi_0, p_{\ast})$, $C_4(\phi_0,p_{\ast})$ and $C_5$ are given by
\begin{gather*}
C_1(\phi_0) = \frac{\phi_0^4}{(96r\sigma \kappa_0)^2},\quad C_2(\phi_0) = \frac{\phi_0^2}{4\kappa_0\sqrt{2r(4r\kappa_0^2+\phi_0^2)}},\\
C_3(\phi_0, p_{\ast}) = \frac{6C_mR_{\max}(64\kappa_0)^2}{p_{\ast}^3 C_1}, \quad C_4(\phi_0,p_{\ast})=\frac{C_3(\phi_0,p_{\ast})}{3}+4\Big(1+\frac{128}{p_{\ast}^2C_2^2}\Big)R_{\max},\\
C_5=\min_{t\in\mathbb{Z}}\left\{ t\ge4K(2C_0+7a+7)\log(t/\tt_0), t\ge2\tt_0(a/\log\tt_0+1)\right\},
\end{gather*}
and we take $C_0 = \max\{10,8/p_{\ast}, 16\log (2d_1d_2)/(p_{\ast}^2C_2^2),1024\log(d_1+d_2)x_{\max}^2/(C_1p_{\ast}^2h^2) \}$, $\tt_0 = \max \left\{ 7K-1, 2C_0 K-1, \left\lceil (t_0+1)^2/e^2 \right\rceil-1 \right\}$ and $t_0 = 2C_0K$.
\end{theorem}

The proof of Theorem \ref{th:matirx} is composed of three subsections: in $\S$\ref{ss:Convergence of Random-Sample Estimators_2} and $\S$\ref{ss:Convergence of Whole-Sample Estimators_2}, we demonstrate the convergence of random-sample estimators and whole-sample estimators, respectively; in $\S$\ref{ss:Bounding the Cumulative Regret_2}, the cumulative regret is decomposed into several parts and upper-bounded seperately.

\vspace{4pt}
\noindent
{\bf Estimator from forced samples up to batch $l$:} By combining Lemma \ref{lem0501} and Lemma \ref{boundsrandom}, we will obtain the tail inequality for forced-sample estimator $\bThetah(\cR_{k,l}, \lambda_1)$. The proof is presented in Appendix \ref{ss:Convergence of Random-Sample Estimators_2}.
\begin{proposition}\label{propj01}
Under Assumptions \ref{as0501}, \ref{as0503} and \ref{as0504}, for all arms $k \in [K]$, the forced-sample estimator $\bThetah(\cR_{k,l}, \lambda_1)$ satisfies
$$
\P\Big( \big\| \bThetah(\cR_{k,l}, \lambda_1) - \bTheta_k^{\true} \big\|_N \ge \frac{h}{4x_{\max}} \Big) \le \frac{6}{t_l+1},
$$
where $t_l \ge (t_0 +1)^2/e^2 - 1$, $t_0 = 2C_0K$, $\lambda_1 = \phi_0^2 p_{\ast}h/(48 r x_{\max})$ and $C_0 = \max \{10,8/p_{\ast},\\ 16\log (2d_1d_2)/(p_{\ast}^2C_2^2),1024\log(d_1+d_2)x_{\max}^2/(C_1p_{\ast}^2h^2) \}$, and the parameter $h$ is defined in Assumption \ref{armoptimality}.
\end{proposition}

\vspace{4pt}
\noindent
{\bf Estimator from whole samples up to batch $l$:} To show the performance of the whole-sample estimator $\bThetah(\cW_{k,l}, \lambda_{2,l})$, we start by proving that enough number of i.i.d. data in whole samples is guaranteed. We also define the event $\mathcal{A}_l$ that the forced-sample estimator at batch $l$ is within a given distance of its true parameter:
$$
\mathcal{A}_l=\mathbbm{1}\Big(\big\| \bThetah(\cR_{k,l}, \lambda_1) - \bTheta_k^{\true} \big\|_N \le \frac{h}{4x_{\max}},~\forall k\in[K]\Big).
$$
According to Proposition \ref{propj01}, if $t_l\ge(t_0+1)^2/e^2-1$, we have $\mathbb{P}(\mathcal{A}_l)\ge1-7K/(t_l+1)$. Moreover, for $t\in\{1,\ldots,t_{l}\}$ and $k\in\cK_o$, we define
$$
M_k(t)=\mathbb{E}\Big(\sum_{m=1}^{l}\sum_{\tau=t_{m-1}+1}^{t_m}\mathbbm{1}(\bX_{\tau}\in \cU_k,\mathcal A_{m-1},\bX_{\tau}\notin \mathcal{R}_{k,m}) \,\Big|\, \mathcal{F}_t\Big),
$$
where $\mathcal{F}_t=\{(\bx_{\tau},r_{\tau}) ~\text{for}~ \tau\le t\}$. Now we know that $\{M_k(t_{l})\}_{l\in[L]}$ is a martingale, and the same proof of Proposition \ref{prop4} constructs its bound:
if $t_{l} \ge C_5$, $l\ge 2$ 
, then for $k\in\cK_o$, we have
\#\label{propj02}
\mathbb{P}\Big(M_k(t_{l})\ge\frac{p_{\ast}}{8}t_{l}\Big)\ge 1 - \exp\Big(-\frac{p_{\ast}^2}{128}(t_{l}-1)\Big).
\#
Now, we will show the convergence of the whole-sample estimators. The proof is presented in Appendix \ref{ss:Convergence of Whole-Sample Estimators_2}.
\begin{proposition}\label{propj03}
If Assumptions \ref{as0501}, \ref{as0503} and \ref{as0504} hold, when $l\ge 2$ and $t_l\ge C_5$, the whole-sample estimator $\bThetah(\cW_{k,l}, \lambda_{2,l})$ will satisfy the following inequality:
$$
\mathbb{P}\Big(\big\| \bThetah(\cW_{k,l}, \lambda_{2,l}) - \bTheta_k^{\true}\big\|_N \ge 32\sqrt{\frac{2(\log t_l+\log(d_1+d_2))}{ t_l p_{\ast}^3 C_1(\phi_0)} }\Big) \le\frac{1}{t_l}+2\exp\Big(-\frac{t_lp_{\ast}^2C_2^2(\phi_0)}{64}\Big),
$$
where
$\lambda_{2,l}=2\phi_0^2/(3r)\cdot\sqrt{2(\log t_l+\log(d_1+d_2))/(t_lp_{\ast}C_1(\phi_0))}$.
\end{proposition}

\vspace{4pt}
\noindent
{\bf Cumulative regret up to time $T$:} Eventually, we will divide all time steps $[T]$ into three groups in the same way as $\S$\ref{sub404}. We will bound the regrets in each group by applying convergence inequalities for forced-sample and whole-sample estimators in Proposition \ref{propj01} and \ref{propj03}. The full proof of Theorem \ref{th:matirx} refers to the process of proving Theorem \ref{th:linear} in Appendix \ref{ss:Bounding the Cumulative Regret}. The details are presented in Appendix \ref{ss:Convergence of Whole-Sample Estimators_2}.

\subsection{Convergence of Random-Sample Estimators}\label{ss:Convergence of Random-Sample Estimators_2}
\begin{proof}[Proof of Proposition \ref{propj01}]
We know from the proof of Proposition \ref{prop3} that $\mathcal{A}$ is a set of i.i.d. samples from $\mathcal{P}_{X|X\in \cV_k}$ for arm $k\in[K]$. Since $t_0 = 2C_0K$, $t_l \ge (t_0 +1)^2/e^2 - 1$, we can invoke Lemma \ref{ap004} to obtain
\#\label{eqj0101}
\mathbb{P}\big( |\cA| <\frac{1}{2}C_0\log(t_l+1)\big)\le\frac{2}{t_l+1}.
\#
Moreover, since $C_0 \ge \max \{10,8/p_{\ast}, 16\log (2d_1d_2)/(p_{\ast}^2C_2^2)\}$ and $16\log (2d_1d_2)/(p_{\ast}^2C_2^2)>$\\$4\log (2d_1d_2)/(p_{\ast}C_2^2)$, we have $n_k \ge C_0\log(t_l+1)/2 \ge C_0/2 > 2\log (2d_1d_2)/(p_{\ast}C_2^2)$. Then, by applying Lemma \ref{lem0501} and Lemma \ref{lem:ap005} to \eqref{eqj0101}, with $p=p_{\ast}$, $\chi = h/(4x_{\max})$, $ C_0 \ge \max \{ 4/(p_{\ast}C_2^2),1024\log(d_1+d_2)x_{\max}^2/(C_1p_{\ast}^2h^2) \}$ and $\lambda_1 = \phi_0^2 p_{\ast}h/(48rx_{\max})$, it follows that
\$
\mathbb{P}\Big( \big\| \bThetah(\cR_{k,l}, \lambda_1) - \bTheta_k^{\true} \big\|_N \ge \frac{h}{4x_{\max}}\Big)
&\le \exp\Big(-\frac{C_1(\phi_0)}{2} \big(\frac{p_{\ast}h}{16x_{\max}}\big)^2 C_0\log(t_l+1)+ \log(d_1+d_2)\Big)\\
&\qquad +\exp\Big(-\frac{1}{4}p_{\ast}C_2(\phi_0)^2C_0\log(t_l+1)\Big)+\frac{5}{t_l+1}\\
&\le \exp\big( -(d_1+d_2)\cdot(2\log(t_l+1)-1) \big) + \frac{4}{t_l+1} \le \frac{6}{t_l+1},
\$
where the last inequality is deduced since $d_1,d_1>1$ and $\log(t_l+1)\ge 1$. Hence, we finifsh the proof of Proposition \ref{propj01}.
\end{proof}

\subsection{Convergence of Whole-Sample Estimators}\label{ss:Convergence of Whole-Sample Estimators_2}
\begin{proof}[Proof of Proposition \ref{propj03}]
Take $\mathcal{A}=\mathcal{W}_{k,l}$, $\cA^{'}=\cup_{m=1}^{l}\mathcal{B}_m$, where \$
\mathcal{B}_m = \{\bX_{\tau}|\bX_{\tau}\in \cU_k,\mathcal A_{m-1},\bX_{\tau}\notin \mathcal{R}_{k,m},\tau\in\{t_{m-1}+1,\ldots,t_m\}\}.
\$
Then, it follows from \eqref{propj02} that $|\cA'|/|\cA|\ge p_*/8$ and $|\cA|\ge|\cA'|\ge p_*t_l/8$. Since $t_l\ge C_5\ge 2KC_0\ge 4C_0$ and $C_0\ge 16\log(2d_1d_2)/(p_*^2C_2^2)$, we know that $|\cA|\ge 2\log(2d_1d_2)/(pC_2^2)$, where $p=p_{\ast}/4$.
Hence, we invoke Lemma \ref{lem0501} with $p=p_{\ast}/4$ and $\lambda=\chi\phi_0^2p_{\ast}/(48r)$ to obtain the following result:
\$
\begin{aligned}
\mathbb{P}\Big(\| \bThetah(\cW_{k,l}, \lambda_{2,l}) - \bTheta_k^{\true}\|_N \ge \chi\Big) \le& \exp\Big(-C_1\big(\frac{\phi_0\sqrt{p_{\ast}/4}}{2}\big)\frac{p_{\ast}t_l}{8}\chi^2+ \log(d_1+d_2)\Big) \\
&\qquad + \exp \Big(-\frac{p_{\ast}^2C_2^2(\phi_0)t_l}{64}\Big)
+\exp \Big(-\frac{p_{\ast}^2(t_{l}-1)}{128}\Big)\\
\le& \exp\Big(-\frac{t_lp_{\ast}^3C_1(\phi_0)}{2048}\chi^2+\log(d_1+d_2)\Big) + 2\exp \Big(-\frac{p_{\ast}^2C_2^2(\phi_0)t_l}{64}\Big),
\end{aligned}
\$                                 
where the last inequality uses $C_2(\phi_0)\le 1/2$ and $t_{l}-1\ge t_l/2$. By taking 
$$
\chi=32\sqrt{\frac{2(\log t_l+\log(d_1+d_2))}{ t_l p_{\ast}^3 C_1(\phi_0)} },
$$
we get
$$
\mathbb{P}\Big(\| \bThetah(\cW_{k,l}, \lambda_{2,l}) - \bTheta_k^{\true}\|_N \ge 32\sqrt{\frac{2(\log t_l+\log(d_1+d_2))}{ t_l p_{\ast}^3 C_1(\phi_0)} }\Big) \le\frac{1}{t_l}+2\exp\Big(-\frac{t_lp_{\ast}^2C_2^2(\phi_0)}{64}\Big),
$$
which completes the proof of Proposition \ref{propj03}.
\end{proof}

\subsection{Bounding the Cumulative Regret}\label{ss:Bounding the Cumulative Regret_2}
\begin{proof}[Proof of Theorem \ref{th:matirx}]
We can divide the cumulative regret $\RG(T,L)$ into three parts in the same way as \eqref{eq:regret decompose 3 parts}. In the sequel, we bound the regrets in each group.

\vspace{4pt}
\noindent
{\bf Regret for the first group:}

From Lemma \ref{ap004}, we know that if $t_0 = 2C_0K$, $C_0 =\max\{10, 8/p_{\ast}\}$, and $T \ge (t_0 +1)^2/e^2 - 1$, the following inequality can be obtained:
\#\label{eqj0301}
\mathbb{P}\Big(\sum_{k=1}^{K}n_k>6C_0K\log T\Big)\le\sum_{k=1}^{K}\mathbb{P}\big(n_k>6C_0K\log T\big)\le\frac{2K}{T+1}.
\#
Thus, refering to \eqref{eq:RG_1 linear}, we can bound the regret of the first group as
\#\label{eq:RG_1 matrix}
\RG_1(T,L) \le R_{\max }K\left(2+6 C_{0} \log T\right)+R_{\max } t_{l_0}.
\#
\vspace{2pt}
\noindent
{\bf Regret for the second group:}

From Proposition \ref{propj01}, we know that
\$
\P\Big( \big\| \bThetah(\cR_{k,l}, \lambda_1) - \bTheta_k^{\true} \big\|_N \ge \frac{h}{4x_{\max}} \Big) \le \frac{6}{t_l+1},
\$
Therefore, $\RG_2(T,L)$ can be bounded as follows:
\#\label{eq:RG_2 matrix}
\RG_2(T,L)&\le\mathbb{E}\big(\sum_{l=l_0}^{L}(t_l-t_{l-1})\mathbbm{1}(\mathcal A_{l-1}^c)R_{\max}\big)\le \sum_{l=l_0}^{L}(t_l-t_{l-1})\frac{6KR_{\max}}{t_{l-1}+1}\notag\\
&\le 6KR_{\max}\sum_{l=l_0}^L\frac{(a t_{l-1})/((l-1)\log t_{l-1})}{t_{l-1}+1}\notag\\
&\le 6KR_{\max}a(\log L+1),
\#
where the third equality comes from the grid structure \eqref{eqgrid}.

\vspace{4pt}
\noindent
{\bf Regret for the third group:}\\
Define $\cE(t,\delta)_{1,k} = \{ \la \bX_{t}, \bTheta_{i}^{\text {true}} \ra > \la \bX_{t}, \bTheta_{k}^{\text {true}} \ra + \delta \}$. Then, similar to the decomposition in \eqref{apeg05} and \eqref{apeg06}, the regret at time $t$ can be bounded as
\#
\rg_t =& \mathbb{E}\Big(\sum_{j=1}^K\mathbbm{1}\big(j=\arg \max _{k \in [K]} \big(\la \bX_{t}, \bThetah(\cW_{k,l}, \lambda_{2,l}) \ra \big)\big)\big(\la \bX^{\mathrm{T}}_{t}, \bTheta_{i}^{\text {true}} - \bTheta_{j}^{\text {true}}\ra\big)\Big) \nonumber\\
\le& \mathbb{E}\Big(\sum_{j \neq i} \mathbbm{1}\big(\big\{ \la \bX_{t}, \bThetah(\cW_{j,l}, \lambda_{2,l}) \ra > \la \bX_{t}, \bThetah(\cW_{i,l}, \lambda_{2,l}) \ra \big\}\cap \mathcal{E}(t, \delta)_{1, j}\big) \Big) \cdot (2R_{\max}) \label{eqj03041}\\
&\qquad + \mathbb{E}\Big(\sum_{j \neq i} \mathbbm{1}\big(\big\{ \la \bX_{t}, \bThetah(\cW_{j,l}, \lambda_{2,l}) \ra > \la \bX_{t}, \bThetah(\cW_{i,l}, \lambda_{2,l}) \ra \big\}\cap \mathcal{E}(t, \delta)_{1, j}^c\big) \Big) \cdot \delta. \label{eqj03042}
\#
The term \eqref{eqj03042} can be bounded by Assumption \ref{as0502}:
\#\label{eq:bound1}
\mathbb{E}\Big(\sum_{j \neq i} \mathbbm{1}\big(\big\{ \la \bX_{t}, \bThetah(\cW_{j,l}, \lambda_{2,l}) \ra > \la \bX_{t}, \bThetah(\cW_{i,l}, \lambda_{2,l}) \ra \big\}\cap \mathcal{E}(t, \delta)_{1, j}^c\big) \Big) \cdot \delta \le C_mR_{\max}K\delta^2.
\#
Correspondingly, the term \eqref{eqj03041} is bounded by
\#\label{eq:bound2}
&\mathbb{E}\Big(\sum_{j \neq i} \mathbbm{1}\big(\big\{ \la \bX_{t}, \bThetah(\cW_{j,l}, \lambda_{2,l}) \ra > \la \bX_{t}, \bThetah(\cW_{i,l}, \lambda_{2,l}) \ra \big\}\cap \mathcal{E}(t, \delta)_{1, j}\big) \Big) \cdot (2R_{\max})\nonumber\\
&\qquad \le\mathbb{E}\Big(\sum_{j \neq i} \mathbbm{1}\big( \kappa_0\| \bThetah(\cW_{j,l}, \lambda_{2,l}) - \bTheta_j^{\true}\|_N >-\kappa_0\| \bThetah(\cW_{i,l}, \lambda_{2,l}) - \bTheta_i^{\true}\|_N +\delta \big)\Big) \cdot (2R_{\max})\nonumber\\
&\qquad \le 2R_{\max}\sum_{j\ne i}\Big( \mathbb{P}\big(\| \bThetah(\cW_{j,l}, \lambda_{2,l}) - \bTheta_j^{\true}\|_N \ge \frac{\delta}{2\kappa_0}\big) + \mathbb{P}\big(\| \bThetah(\cW_{i,l}, \lambda_{2,l}) - \bTheta_i^{\true}\|_N \ge \frac{\delta}{2\kappa_0}\big)\Big).
\#
Then, we derive from Proposition \ref{propj03} that
\$
\mathbb{P}\Big(\| \bThetah(\cW_{k,l}, \lambda_{2,l}) - \bTheta_k^{\true}\|_N \ge 32\sqrt{\frac{2(\log t_l+\log(d_1+d_2))}{ t_l p_{\ast}^3 C_1(\phi_0)} }\Big) \le\frac{1}{t_l}+2\exp\Big(-\frac{t_lp_{\ast}^2C_2^2(\phi_0)}{64}\Big),
\$
which is applied to \eqref{eq:bound2} by setting $\delta = 64\kappa_0\sqrt{2(\log t_l + \log(d_1+d_2))/(t_lp_{\ast}^3C_1(\phi_0))}$:
\#\label{eq:bound201}
&2R_{\max}\sum_{j\ne i}\Big( \mathbb{P}\big(\| \bThetah(\cW_{j,l}, \lambda_{2,l}) - \bTheta_j^{\true}\|_N \ge \frac{\delta}{2\kappa_0}\big) + \mathbb{P}\big(\| \bThetah(\cW_{i,l}, \lambda_{2,l}) - \bTheta_i^{\true}\|_N \ge \frac{\delta}{2\kappa_0}\big)\Big)\nonumber\\
&\qquad \le 2R_{\max}(K-1)\Big( \frac{2}{t_l}+4\exp\big(-\frac{t_lp_{\ast}^2C_2^2(\phi_0)}{64}\big)\Big).
\#
Combining \eqref{eq:bound1} and \eqref{eq:bound201}, we obtain that
\$
\rg_t &\le  2R_{\max}K \Big( \frac{2}{t_l}+4\exp\big(-\frac{t_lp_{\ast}^2C_2^2(\phi_0)}{64}\big)\Big) + C_mR_{\max}K\delta^2\nonumber\\
&\le R_{\max}K\Big(\frac{4}{t_l}+8\exp\big(-\frac{t_lp_{\ast}^2C_2^2(\phi_0)}{64}\big)\Big)+C_m R_{\max} K 64^2\kappa_0^2\frac{2(\log t_l + \log(d_1+d_2))}{t_lp_{\ast}^3C_1(\phi_0)}\nonumber\\
&=R_{\max}K\Big(\frac{4}{t_l}+8\exp\big(-\frac{t_lp_{\ast}^2C_2^2(\phi_0)}{64}\big)+C_3^{'}\frac{\log t_l+\log(d_1+d_2)}{t_l}\Big),
\$
where $C_3^{'} = 2(64\kappa_0)^2C_m/(p_{\ast}^3C_1(\phi_0))$. Then, we accumulate the regret in this group:
\#\label{eq:RG_3 matrix}
\RG_3(T,L)&\le\sum_{l=l_0}^{L-1}R_{\max}K\Big(\frac{4}{t_l}+8\exp\big(-\frac{t_l p_{\ast}^2C_2^2(\phi_0)}{64}\big)+C_3^{'}\frac{\log t_l+\log(d_1+d_2)}{t_l}\Big)(t_{l+1}-t_l)\notag\\
&\le \sum_{l=l_0}^{L-1}R_{\max}K\Big(\frac{4 + 3C_3^{'} \log(d_1+d_2)}{t_l}+8\exp\big(-\frac{t_lp_{\ast}^2C_2^2(\phi_0)}{64}\big)+\frac{C_3^{'}\log t_l}{t_l}\Big)\frac{at_l}{l \log t_l}\notag\\
&\le R_{\max}Ka\sum_{l=l_0}^{L-1}\Big(\frac{4+ 3C_3^{'}\log(d_1+d_2) + C_3^{'}}{l}+\frac{512 }{p_{\ast}^2C_2^2l}\Big)\notag\\
&\le Ka \bigl( C_3(\phi_0,p_{\ast})\log(d_1+d_2)+C_4(\phi_0, p_{\ast}) \bigr)(\log L+1),
\#
where $C_3(\phi_0,p_{\ast})=6C_mR_{\max}(64\kappa_0)^2/(p_{\ast}^3C_1(\phi_0))$ and $C_4(p_{\ast})= R_{\max}(C_3^{'} + 4 +512/(p_{\ast}^2C_2^2(\phi_0))$.

Finally, the regret bound can be obtained by combining the bounds \eqref{eq:RG_1 matrix}, \eqref{eq:RG_2 matrix} and \eqref{eq:RG_3 matrix} for the three groups:
\$
\RG(T,L)=&\RG_1(T,L)+\RG_2(T,L)+\RG_3(T,L)\\
\le& R_{\max }K\left(2+6 C_{0} \log T\right)+R_{\max } t_{l_0}+6KR_{\max}a(\log L+1)\\
&\qquad +Ka \big( C_3(\phi_0,p_{\ast})\log(d_1+d_2)+C_4(\phi_0, p_{\ast}) \big) (\log L+1)\\
=&\mathcal{O}\left( K r^2 \log(d_1d_2) (\log^2 T+\log(d_1d_2)) \right).
\$
Eventually, we conclude the proof of Theorem \ref{th:matirx}.
\end{proof}

\section{Proof of the Supporting Lemmas}

\subsection{A Tail Inequality for Low-Rank Matrix}\label{pflem0501}
This subsection provides a demonstration of a tail inequality for the low-rank matrix model. The inequality is stated in Lemma \ref{lem0501}. We summarize the key steps as follows. First, we demonstrate in Appendix \ref{subi01} that when the covariance matrix of $\cA$ satisfies the compatibility condition $\Sighat(\cA) \in \cC(\cM, \barcM^{\bot}, \phi)$ and $\cG(\lambda_n)=\{\| \nabla \cL_n(\bTheta^{\true})\|_{\op}\le \lambda_n/2\}$ holds, the nuclear norm of $\Deltahat$ is bounded. To establish these conditions, we prove in Appendix \ref{subi02} that $\Sighat(\cA)$ satisfies the compatibility condition with high probability by showing that $\| \Sighat(\cA') - \Sigma \|_{\op}$ is small with high probability and that $\Sighat(\cA')$ meets the compatibility condition. Finally, in Appendix \ref{subi03}, we show that $\cG(\lambda_n)$ occurs with high probability by using the matrix subgaussian series and the matrix Bernstein Concentration.

The novel concentration result for non-i.i.d. data here, which is obtained by adopting the matrix Bernstein inequality and deriving an inequality for matrix subgaussian series (Lemma \ref{lm:Matrix_Sub-Gaussian_Series}). These techniques differ from the prior low-rank bandit literature, particularly from \citet{li2022simple}. Although they proved a similar result for low-rank bandits, their work doesn't directly utilize inequalities for matrix series in Lemma C.4, leading to exploration time steps of $\Omega(d_1d_2)$ when $d_1d_2$ is large. Therefore, their regret bound has a dependence of $\cO(d_1d_2)$ instead of $\log(d_1+d_2)$. In our work, we consider a non-i.i.d. data set $\cA = { \bX_t \in \RR^{d_1 \times d_2} }_{t=1}^n$ with an unknown subset $\cA' \subset \cA$ consisting of i.i.d. samples. We assume that $\Sigma \in \cC(\cM, \barcM^{\bot}, \phi_1)$, where $\Sigma = \E ( \Vec{(\bX)\Vec{(\bX)}^{\mathrm{T}}} )$. Then, denoting $\Thetahat-\bTheta$ by $\Deltahat$, we present the following result.

\begin{lemma}[A Tail Inequality for Non-i.i.d. Data]\label{lem0501}
If $d_1,d_2 > 1$, for any $\chi > 0$, $\,|\, \mathcal{A}'\,|\,/\,|\, \mathcal{A}\,|\, \ge p/2$, $\,|\, \mathcal{A}\,|\, \ge 2\log (2d_1d_2)/(p C_2^2(\phi_1))$, and $\lambda_n =\lambda_n(\chi, \phi_1\sqrt{p}/2)=\chi \phi_1^2p/(12r)$, the following tail inequality holds:
\$
\P\big( \| \Deltahat \|_N \ge \chi \big) \le \exp\big( - C_1(\sqrt{p} \phi_1/2)\,|\, \cA \,|\,\chi^2 + \log(d_1+d_2) \big) + \exp\big( -p C_2^2(\phi_1) \,|\, \cA \,|\,/2 \big),
\$
where $C_1(\phi_1) = \phi_1^4/(96r\sigma \kappa_0)^2$, $C_2(\phi_1) = \phi_1^2/(4\kappa_0\sqrt{2r(4r\kappa_0^2+\phi_1^2)})$.
\end{lemma}

\subsubsection{A General Tail Inequality}\label{subi01}
Given a sample set (not necessarily i.i.d.) of size $n$, we first consider a general convex loss function $\cL_n(\bTheta)$ with the nuclear norm regularization. We aim to determine under what conditions the estimation error will be bounded.
\begin{lemma}\label{lem:generalresult}
Consider the true parameter $\bTheta^{\true}$ with rank $r$ and the induced pair of subspaces $(\cM,\barcM^{\bot})$ in \eqref{eq:subspacemuv} and \eqref{eq:subspacebmuv}. Then for some constant $\phi>0$, we have
\$
\P\left( \| \Deltahat \|_N \ge \frac{6r\lambda_n}{\phi^2} \right) \le \P\left(\cG^c(\lambda_n) \right) + \P\left( \Sighat \notin \cC(\cM, \barcM^{\bot}, \phi) \right),
\$
where $\cG^c(\lambda_0(\gamma))=\{\|1/n\sum_{t=1}^n\epsilon_t \bX_t \|_{\op} \ge \lambda_0(\gamma)/2\}$
\end{lemma}
\begin{proof}
To begin with, we use the optimality of $\bTheta^{\true}$:
\#\label{eqi101}
\cL_n(\Thetahat) + \lambda_n\| \Thetahat \|_N \le \cL_n(\ \bTheta^{\true} ) + \lambda_n\| \bTheta^{\true} \|_N.
\#
Taking the second order Tailor expansion of $\cL_n(\Thetahat)$ into \eqref{eqi101}, we deduce that
\$
\left\la \nabla \cL_n(\bTheta^{\true}), \Deltahat \right\ra + \frac{1}{2} \vec(\Deltahat)^{\mathrm{T}} \Sighat \vec(\Deltahat) \le \lambda_n\| \Deltahat \|_N.
\$
Applying H\"older's inequality for Schatten norms to the leftmost term of the equation above, it follows that
\#
\frac{1}{2} \vec(\Deltahat)^{\mathrm{T}} \Sighat_{\bX\bX} \vec(\Deltahat) \le \left\| \nabla \cL_n(\bTheta^{\true}) \right\|_{\op} \cdot \| \Deltahat \|_N + \lambda_n \| \Deltahat \|_N. \label{eqi102}
\#
If event $\cG(\lambda_n)=\{ \left\| \nabla \cL_n(\bTheta^{\true})\right\|_{\op}\le \lambda_n/2\}$ holds true, \eqref{eqi102} can be written as
\#\label{eqi1021}
\vec(\Deltahat)^{\mathrm{T}} \Sighat_{\bX\bX} \vec(\Deltahat) \le 3 \lambda_n \| \Deltahat \|_N.
\#
Additionally, by invoking Lemma \ref{lem:cone} we have
\#\label{eq:deltahatN}
\| \Deltahat \|_N \le\|\Deltahat_{\barcM}\|_N + \|\Deltahat_{\barcM^{\bot}}\|_N \le4\|\Deltahat_{\barcM}\|_N.
\#
Furthermore, if $\Sighat \in \cC(\cM,\barcM^{\bot}, \phi)$ in \eqref{eq:cC}, it follows that
\#
\vec(\Deltahat)^{\mathrm{T}} \Sighat \vec(\Deltahat) &\ge16\phi^2\| \Deltahat_{\barcM} \|_F^2\ge \frac{8\phi^2}{r}\| \Deltahat_{\barcM} \|_N^2 \ge \frac{\phi^2}{2r}\| \Deltahat \|_N^2,\label{eqi103}
\#
where the second inequality is derived since any matrix $\bDelta\in\barcM(\UU,\VV)$ has rank at most $2r$ \citep{wainwright2019high}, and the last inequality is by applying \eqref{eq:deltahatN}.
Then by combining \eqref{eqi103} and \eqref{eqi1021}, we finally get
\$
\frac{\phi^2}{2r}\| \Deltahat \|_N^2 \le3 \lambda_n \| \Deltahat \|_N
\$
thus we prove that
\$
\| \Deltahat \|_N \le \frac{6r\lambda_n}{\phi^2}.
\$
Thus, the proof is accomplished.
\end{proof}

Then, we demonstrate that by properly choosing the regularization weight $\lambda_n$, the error vector $\Deltahat=\Thetahat-\bTheta^{\true}$ is in special cone in the following lemma.
\begin{lemma}\label{lem:cone}
Suppose that $\cL_n$ is a convex and differentiable function and the following event holds:
\#\label{eq:cG}
\cG(\lambda_n)=\Big\{ \left\| \nabla \cL_n(\bTheta^{\true})\right\|_{\op}\le \frac{\lambda_n}{2}\Big\}
\#
Then for the pair $(\cM,\barcM^{\bot})$ defined in \eqref{eq:subspacemuv} and \eqref{eq:subspacebmuv}, the error $\Deltahat$ belongs to the set
\#\label{eq:cV}
\cV(\cM,\barcM^{\bot})=\left\{ \bDelta\in\RR^{d_1\times d_2} \mid \|\bDelta_{\barcM^{\bot}}\|_N\le 3\|\bDelta_{\barcM}\|_N\right\}.
\#
\end{lemma}
This lemma is adopted from Proposition $9.13$ in \citet{wainwright2019high}. For the completeness of the theory, we present the proof here.
\begin{proof}
The argument starts with the function $\cF(\Deltahat):\RR^{d_1\times d_2}\rightarrow\RR$ given by
\#\label{eq:cF}
\cF(\Deltahat)=\cL_n(\bTheta^{\true}+\Deltahat) - \cL_n(\bTheta^{\true}) + \lambda_n\left( \|\bTheta^{\true}+\Deltahat\|_N - \|\bTheta^{\true}\|_N\right).
\#
By the convexity of $\cL_n$, we know that \#\label{eq:convcL}
\cL_n(\bTheta^{\true}+\Deltahat) - \cL_n(\bTheta^{\true}) &\ge \left< \nabla\cL_n(\bTheta^{\true}), \Deltahat\right> \ge -\|\nabla\cL_n(\bTheta^{\true})\|_{\op}\cdot\|\Deltahat\|_N\nonumber\\
&\ge -\frac{\lambda_n}{2}(\|\Deltahat_{\barcM}\|_N + \|\Deltahat_{\barcM^{\bot}}\|_N),
\#
where the second inequality is by applying the matrix H\"older inequality for spectrum norms, and the last inequality is deduced by conditioning on $\cG(\lambda_n)$ in \eqref{eq:cG} and the triangle inequality.

Additionally, we have
\#\label{eq:bThetaN}
\|\bTheta^{\true}+\Deltahat\|_N - \|\bTheta^{\true}\|_N &= \|\bTheta^{\true}_{\cM}+\Deltahat_{\barcM}+\Deltahat_{\barcM^{\bot}}\|_N - \|\bTheta^{\true}_{\cM}\|_N\nonumber\\
&\ge \|\bTheta^{\true}_{\cM}+\Deltahat_{\barcM^{\bot}}\|_N - \|\Deltahat_{\barcM}\|_N - \|\bTheta^{\true}_{\cM}\|_N\nonumber\\
&= \|\Deltahat_{\barcM^{\bot}}\|_N - \|\Deltahat_{\barcM}\|_N,
\#
where the second inequality is by the triangle inequality, and the last is because of the decomposability of subspaces $(\cM,\barcM^{\bot})$. Taking \eqref{eq:bThetaN} and \eqref{eq:convcL} back into \eqref{eq:cF}, we obtain that
\#\label{eq:cF2}
\cF(\Deltahat) &\ge -\frac{\lambda_n}{2}(\|\Deltahat_{\barcM}\|_N + \|\Deltahat_{\barcM^{\bot}}\|_N) + \frac{\lambda_n}{2} \left( \|\Deltahat_{\barcM^{\bot}}\|_N - \|\Deltahat_{\barcM}\|_N\right)\nonumber\\
&= \frac{\lambda_n}{2}\left( \|\Deltahat_{\barcM^{\bot}}\|_N - 3\|\Deltahat_{\barcM}\|_N\right).
\#
Since the optimality of $\Thetahat$ implies that $\cF(\Deltahat) \le0$, we obtain from \eqref{eq:cF2} that
\$
\|\Deltahat_{\barcM^{\bot}}\|_N - 3\|\Deltahat_{\barcM}\|_N\le0,
\$
which completes the proof.
\end{proof}

In the next two sections, we will prove that $\Sighat \in \cC(\cM, \barcM^{\bot}, \phi)$ and event $\cG(\lambda_n)$ holds with a high probability separately.

\subsubsection{Compatibility Condition for Non-i.i.d. Samples} \label{subi02}
Recall the whole set $\cA$, the i.i.d. subset $\cA'$ and the corresponding covariance matrices $\Sighat(\cA)$ and $\Sighat(\cA')$.
In this part, we will show that $\| \Sighat - \Sigma \|_{\op}$ is small with high probability in the beginning so that we can prove $\Sighat(\cA') \in \cC(\cM, \barcM^{\bot}, \phi_1)$ with high probability. Afterwards, we finally prove that $ \Sighat(\cA) \in \cC\left( \cM, \barcM^{\bot}, \frac{\phi_1}{\sqrt{2}}\sqrt{\frac{\,|\, \cA' \,|\,}{\,|\, \cA \,|\,}} \right)$ with high probability.
\begin{lemma}\label{lemi02}
Given i.i.d. samples $\bX_1, \ldots, \bX_n \in \RR^{d_1 \times d_2}$ such that $\bX_t$ satisfy $\| \bX_t \|_{F} \le \kappa_0$ for all $t \in [n]$, then when $n\ge \log(2d_1d_2)/C_2^2(\phi_1)$, we have
\$
\P\left( \| \Sighat - \Sigma \|_{\op} \ge \frac{\phi_1^2}{4r} \right) \le \exp\left(-C_2^2(\phi_1)n\right),
\$
where $\Sigma = \E\left( \vec(\bX_t)\vec(\bX_t)^{\mathrm{T}} \right)$, $\Sighat = \frac{1}{n}\sum_{t=1}^n \vec(\bX_t)\vec(\bX_t)^{\mathrm{T}}$, and $C_2(\phi_1)=\phi_1^2/(4\kappa_0\sqrt{2r(4r\kappa_0^2+\phi_1^2)})$.
\end{lemma}
\begin{proof}
Since 
$$
\left\|\vec(\bX_t)\vec(\bX_t)^{\top}\right\|_{\op} \le \left\|\vec(\bX_t)\vec(\bX_t)^{\top}\right\|_{F} \le \|\bX_t\|_F^2 \le \kappa_0^2,
$$
we know that
$$
\left\|\vec(\bX_t)\vec(\bX_t)^{\top} - \E\vec(\bX_t)\vec(\bX_t)^{\top}\right\|_{\op} \le 2\kappa_0^2.
$$
Moreover, we obtain the following bound
\$
\left\|\E\vec(\bX_t)\vec(\bX_t)^{\top}\vec(\bX_t)\vec(\bX_t)^{\top}\right\|_{\op} &\le \E\left\|\vec(\bX_t)\vec(\bX_t)^{\top}\vec(\bX_t)\vec(\bX_t)^{\top}\right\|_{\op}\\
&= \E\left\|\vec(\bX_t)\vec(\bX_t)^{\top}\right\|_{\op}^2 \le \kappa_0^4,
\$
where the first inequality uses Jenssen inequality, and the equality is deduced since matrix $\vec(\bX_t)\vec(\bX_t)^{\top}$ is positive semi-definite and for any positive semi-definite matrix $A$, we have $\sigma_{\max}(A^2) = (\sigma_{\max}(A))^2$. Hence, it follows that
\$
&\left\|\E\left(\vec(\bX_t)\vec(\bX_t)^{\top} - \E\vec(\bX_t)\vec(\bX_t)^{\top}\right)\cdot\left(\vec(\bX_t)\vec(\bX_t)^{\top} - \E\vec(\bX_t)\vec(\bX_t)^{\top}\right)\right\|_{\op}\\
&\qquad= \left\|\E\vec(\bX_t)\vec(\bX_t)^{\top}\vec(\bX_t)\vec(\bX_t)^{\top} - \left(\E\vec(\bX_t)\vec(\bX_t)^{\top}\right)^2\right\|_{\op}\\
&\qquad\le \left\|\E\vec(\bX_t)\vec(\bX_t)^{\top}\vec(\bX_t)\vec(\bX_t)^{\top}\right\|_{\op} + \left(\E\left\|\vec(\bX_t)\vec(\bX_t)^{\top}\right\|_{\op}\right)^2 \le 2\kappa_0^4.
\$
Then, by invoking Lemma \ref{lm:Matrix_Bernstein} with $\rho^2=2n\kappa_0^4$ and $M=2\kappa_0^2$, we get
\$
\P\big(\| \Sighat - \Sigma \|_{\op} \ge \tau\big) &= \P\Big(\big\| \frac{1}{n}\sum_{t=1}^n \left(\vec(\bX_t)\vec(\bX_t)^{\mathrm{T}} - \E\vec(\bX_t)\vec(\bX_t)^{\mathrm{T}}\right) \big\|_{\op} \ge \tau\Big)\\
&\le d_1d_2\exp\Big(\frac{-n\tau^2/2}{2\kappa_0^4 + 2\kappa_0^2\tau/3}\Big) = \exp\Big(\frac{-n\tau^2}{4\kappa_0^2(\kappa_0^2 + \tau/3)} + \log(2d_1d_2)\Big).
\$
Taking $\tau=\phi_1^2/(4r)$, we can find $-n\tau^2/4\kappa_0^2(\kappa_0^2 + \tau/3)\le -2C_2^2(\phi_1)n$ and $\log(2d_1d_2)\le C_2^2(\phi_1)n$, which completes the proof.
\end{proof}

\begin{lemma}\label{lemi03}
There exists an i.i.d. sample set $\cA'$ in a non-i.i.d. set $\cA = \{ \bX_t \}_{t=1}^n$. Suppose $\Sigma \in \cC(\cM, \barcM^{\bot}, \phi_1)$ for some $\phi_1>0$ and $\| \Sighat(\cA')-\Sigma \|_{\op} \le \phi_1^2/(4r)$, then $\Sighat(\cA') \in \cC(\cM, \barcM^{\bot}, \phi_1/\sqrt{2})$.
\end{lemma}
\begin{proof}
Since $\Sigma \in \cC(\cM, \barcM^{\bot}, \phi_1)$, for any matrix $\bDelta \in \cV(\cM,\barcM^{\bot})$ in \eqref{eq:cV}, we have $\| \bDelta_{\barcM}\|_F^2 \le  \vec(\bDelta)^{\mathrm{T}} \Sigma \vec(\bDelta)/(4\phi_1)^2$. Then, using $\| \Sighat(\cA')-\Sigma \|_{\op} \le \phi_1^2/(4r)$, we deduce that
\#
\big| \vec(\bDelta)^{\mathrm{T}} \Sighat(\cA') \vec(\bDelta) - \vec(\bDelta)^{\mathrm{T}} \Sigma \vec(\bDelta)\big| &= \| \bDelta \|_{F}^2 \,|\, \vec(\tbv)^{\mathrm{T}} \left(\Sighat(\cA')- \Sigma\right) \vec(\tbv) \,|\, \nonumber\\
&\le \| \Sighat(\cA')-\Sigma \|_{\op} \| \bDelta \|_{F}^2  \le \frac{\phi_1^2}{4r} \| \bDelta \|_F^2, \label{eqi301}
\#
where we denote $\vec(\bDelta)/\| \bDelta \|_{F}$ as $\vec(\tbv)$.
In addition, by the matrix inequalities, it is deduced that
\#
\| \bDelta \|_F^2\le\| \bDelta \|_N^2 \le (4\| \bDelta_{\barcM} \|_N)^2 \le (4\sqrt{2r}\| \bDelta_{\barcM} \|_F)^2, \label{eqi302}
\#
where the second inequality is by applying \eqref{eq:deltahatN} and the reason for the last inequality is that $\rank(\bDelta_{\barcM})\le 2r$.
We obtain from \eqref{eqi301} and \eqref{eqi302} that
\#
\Big| \frac{\vec(\bDelta)^{\mathrm{T}} \Sighat(\cA') \vec(\bDelta)}{\vec(\bDelta)^{\mathrm{T}} \Sigma \vec(\bDelta)} -1 \Big| &\le \frac{ \phi_1^2/(4r)\cdot(4\sqrt{2r}\| \bDelta_{\barcM} \|_F)^2}{(4\phi_1\| \bDelta_{\barcM} \|_F)^2} =\frac{1}{2}. \label{eqi303}
\#
Thus the following result is obvious:
\#
\| \bDelta_{\barcM} \|_F^2 \le \frac{\vec(\bDelta)^{\mathrm{T}} \Sigma \vec(\bDelta)}{(4\phi_1)^2} \le \frac{2\vec(\bDelta)^{\mathrm{T}} \Sighat(\cA') \vec(\bDelta)}{(4\phi_1)^2},
\#
which completes the proof.
\end{proof}

\begin{lemma}\label{lemi04}
Suppose that $\Sigma \in \cC(\cM, \barcM^{\bot}, \phi_1)$ for some $\phi_1>0$. If $\,|\, \cA' \,|\, \ge \log(2d_1d_2)/C_2^2(\phi_1)$, we have
\$
\P\Big( \Sighat(\cA) \in \cC\Big( \cM, \barcM^{\bot}, \frac{\phi_1}{\sqrt{2}}\sqrt{\frac{\,|\, \cA' \,|\,}{\,|\, \cA \,|\,}} \Big) \Big) \ge 1 - \exp\left( -C_2^2(\phi_1) \,|\, \cA' \,|\, \right),
\$
where $C_2(\phi_1) = \phi_1^2/(4\kappa_0\sqrt{2r(4r\kappa_0^2+\phi_1^2)})$.
\end{lemma}
\begin{proof}
Since $\,|\, \cA' \,|\, \ge \log(2d_1d_2)/C_2^2(\phi_1)$, we can combine Lemma \ref{lemi02} and Lemma \ref{lemi03} to get
\$
\P\Big( \Sighat(\cA') \in \cC\big( \cM, \barcM^{\bot}, \frac{\phi_1}{\sqrt{2}} \big) \Big) \ge 1-\exp\left( -C_2^2(\phi_1) \,|\, \cA' \,|\, \right).
\$
Furthermore, in order to prove the compatibility of $\Sighat(\cA)$, we write it as:
\$
\Sighat(\cA) = \frac{\,|\, \cA' \,|\,}{\,|\, \cA \,|\,}\Sighat( \cA' ) + \frac{\,|\, \cA \backslash \cA' \,|\,}{\,|\, \cA \,|\,}\Sighat(\cA \backslash \cA') \ge \frac{\,|\, \cA' \,|\,}{\,|\, \cA \,|\,}\Sighat(\cA').
\$
If $\Sighat(\cA') \in \cC\left( \cM, \barcM^{\bot}, \phi_1/\sqrt{2} \right)$, for any $\bDelta \in \cV$, we have
\$
\| \bDelta_{\barcM} \|_F^2 &\le \frac{2\vec(\bDelta)^{\mathrm{T}} \Sighat(\cA') \vec(\bDelta)}{(4\phi_1)^2}\nonumber\\
&\le \frac{2\,|\, \cA \,|\,/\,|\, \cA' \,|\, \vec(\bDelta)^{\mathrm{T}} \Sighat(\cA) \vec(\bDelta)}{(4\phi_1)^2}
\$
Therefore, under the condition that $\Sighat(\cA') \in \cC\left( \cM, \barcM^{\bot}, \frac{\phi_1}{\sqrt{2}} \right)$, we prove that $ \Sighat(\cA) \in \cC\left( \cM, \barcM^{\bot}, \frac{\phi_1}{\sqrt{2}}\sqrt{\frac{\,|\, \cA' \,|\,}{\,|\, \cA \,|\,}} \right)$.
Therefore, the desired result is obtained.
\end{proof}

\subsubsection{Probability of \texorpdfstring{$\cG(\lambda_n)$}{cG}}\label{subi03}

Given the observations $(r,\mathfrak{X}_n)$ from the matrix regression model in $\S$\ref{s:LinearModels}, recall the loss function $\cL_n(\bTheta)=\|r-\mathfrak{X}_n(\bTheta)\|_{\op}^2/n$. Then the event $\cG(\lambda_n)$ is given by
\$
\cG(\lambda_n) = \Big\{ \big\| \frac{1}{n}\sum_{t=1}^n\epsilon_t \bX_t \big\|_{\op} \le \frac{\lambda_n}{2} \Big\}.
\$
\begin{lemma}\label{lemi06}
Suppose Assumption \ref{as0501} holds and define $\lambda_0(\gamma) = 8\sigma\kappa_0\sqrt{(\gamma^2+\log(d_1+d_2))/n}$. Then, we have
\$
\P\left( \cG(\lambda_0(\gamma)) \right) \ge 1-\exp( -\gamma^2).
\$
\end{lemma}
\begin{proof}
Since $\|X_t\|_F\le 1$ for all $t\in[n]$, we get the following bounds
$$
\max\{\left\|X_tX_t^{\top}\right\|_{op},\left\|X_t^{\top}X_t\right\|_{op}\}\le \kappa_0^2.
$$
Hence, we can invoke Lemma \ref{lm:Matrix_Sub-Gaussian_Series} with $v_t=\kappa_0^2$ and $u=\lambda_n/2$ to derive that for all $u>0$,
\$
\P\Big(\big\| \frac{1}{n}\sum_{t=1}^n\epsilon_t \bX_t \big\|_{\op} \ge \frac{\lambda_n}{2}\Big) &\le (d_1+d_2)\exp\Big(-\frac{n(\lambda_n/2)^2}{16\sigma^2\kappa_0^2}\Big)\\
&= \exp\Big(-\frac{n\lambda_n^2}{64\sigma^2\kappa_0^2} + \log(d_1+d_2)\Big) = \exp( -\gamma^2),
\$
which finishes the proof.
\end{proof}

\subsubsection{Proof of Lemma \ref{lem0501}}
\begin{proof}
Eventually, we present the final result by combining Lemma \ref{lem:generalresult}, \ref{lemi04} and \ref{lemi06} as follows.

Firstly, since $\,|\, \cA' \,|\,/\,|\, \cA \,|\, \ge p/2$ and $\,|\, \cA \,|\, \ge 2\log(2d_1d_2)/(pC_2^2(\phi_1))$, we get $\,|\, \cA' \,|\, \ge \log(2d_1d_2)/(C_2^2(\phi_1))$. Using Lemma \ref{lemi04}, it holds that
\#
\P\left( \Sighat(\cA) \notin \cC\left( \cM, \barcM^{\bot}, \phi \right) \right) \le \exp\left( -p C_2^2(\phi_1) \,|\, \cA \,|\,/2 \right), \label{eqi1}
\#
where $\phi = \sqrt{p} \phi_1/2$. Then from Lemma \ref{lemi06}, 
\#\label{eqi2}
\P\left( \cG^c(\lambda_0(\gamma)) \right) \le\exp( -\gamma^2),
\#
where $\lambda_0(\gamma) = 8\sigma \kappa_0\sqrt{(\gamma^2+\log(d_1+d_2)/n}$. Take \eqref{eqi1} and \eqref{eqi2} back into Lemma \ref{lem:generalresult} and let $\lambda_n\ge2\lambda_0(\gamma)$ and $\phi=\phi_1\sqrt{p}/2$. Then, it follows that
\#\label{eq:DeltahatN1}
\P\Big( \| \Deltahat \|_N \ge \frac{6r\lambda_n}{\phi^2} \Big) &\le \P\left( \cG^c(\lambda_0(\gamma)) \right) + \P\big( \Sighat_{\cA} \notin \cC(\cM, \barcM^{\bot}, \phi) \big) \nonumber\\
& \le \exp( -\gamma^2) + \exp\left( -p C_2^2(\phi_1) \,|\, \cA \,|\,/2 \right).
\#
Now let $\chi=6r\lambda_n(\chi,\phi)/\phi^2$ and choose $\gamma(\chi)=\sqrt{nC_1(\phi)\chi^2-\log(d_1+d_2)}$, where $C_1(\phi)=\phi^4/(96r\sigma\kappa_0)^2$. This choice guarantees that $\lambda_n(\chi,\phi)\ge2\lambda_0(\gamma)$. Then taking the selections back into \eqref{eq:DeltahatN1} yields the final result:
\$
\P\big( \| \Deltahat \|_N \ge \chi \big) \le \exp\left( - C_1(\sqrt{p} \phi_1/2)\,|\, \cA \,|\,\chi^2 + \log(d_1+d_2) \right) + \exp\left( -p C_2^2(\phi_1) \,|\, \cA \,|\,/2 \right). 
\$
Finally, we finish the proof of Lemma \ref{lem0501}.
\end{proof}

\subsection{Bounds about the Forced Samples }\label{proofboundsrandom}
In this section, we focus on bounding the size of the forced sample. Lemma \ref{boundsrandom} and Lemma \ref{ap004} are adopted from Proposition 2 and Lemma EC.8 of \citet{wang2018online}. For the sake of completeness and clarity, the full proof is presented here.
\begin{lemma}[Proposition 2 of \citet{wang2018online}]\label{boundsrandom} Let $C_0 \ge 10$, $t_0 = 2C_0 K$ and $t \ge (t_0+1)^2/e^2-1$. Then, for arm $k \in [K]$, under the $\varepsilon$-decay forced sampling method, with probability at least $1-2/(t+1)$, the size of forced sample $n_k$ up to time $t$ is bounded by
\$
C_0(1+\text{log}(t+1)-\text{log}(t_0+1)) \le n_k \le3C_0(1+\text{log}(t)-\log(t_0)).
\$
\end{lemma}
\begin{proof}
Via the $\epsilon$-decay forced sampling method, the system randomly
draws arm $k$ at time $t$ with the probability of $\min\{1, t_0/t\}/K$, where $K$ denotes the number of arms. Therefore, at time $t$, the expected total number of time steps at which arm $k$ were randomly sampled is
\$
\mathbb{E}[n_k]=\frac{1}{K}\sum_{\tau=1}^{t}\min\left\{ 1, \frac{t_0}{\tau}\right\}.
\$
When $t>t_0$, it follows that
\#\label{ecb2}
\mathbb{E}[n_k]=\frac{1}{K}\Big( t_0+\sum_{\tau=t_0+1}^{t}\frac{t_0}{\tau}\Big) =\frac{t_0}{K}\Big( 1+\sum_{\tau=t_0+1}^{t}\frac{1}{\tau}\Big).
\#
The function $f(\tau) = 1/\tau$ is decreasing in $\tau$ , so $f(\tau)$ can be bounded from both sides for any $t \geq 2$:
\$
\int_{t}^{t+1} \frac{1}{t} \md t<\frac{1}{t}<\int_{t-1}^{t} \frac{1}{t} \md t.
\$
Since $t_0 \ge 1$, we plug $t$ from $t_0 +1$ to $t$ into the above inequality respectively and sum them up:
\#\label{ecb4}
\log(t+1)-\log(t_0+1) < \sum_{\tau=t_0+1}^{\mathrm{T}} \frac{1}{\tau} < \log (t)-\log(t_0).
\#
From \eqref{ecb2} and \eqref{ecb4}, the $\mathbb{E}(n_k)$ can be bounded as follows:
\#\label{ecb5}
\frac{t_0}{K}(1+\log(t+1)-\log(t_0+1)) < \mathbb{E}(n_k) < \frac{t_0}{K}(1+\log(t)-\log(t_0)).
\#
Because $n_k=\sum_{\tau=1}^{t} \mathbbm{1}\{ \text{forced sampling for arm } k \text{ at } \tau\}$, it is regarded as the summation of bounded i.i.d. random variables. We can connect $n_k$ and $\mathbb{E}(n_k)$ with the Chernoff bound:
\#\label{ecb6}
\mathbb{P}\big( \frac{1}{2}\mathbb{E}(n_k) \le n_k \le \frac{3}{2}\mathbb{E}(n_k)\big) > 1-2\exp\big( -\frac{1}{10}\mathbb{E}(n_k)\big).
\#
Using the upper and lower bounds provided in \eqref{ecb5}, the $\mathbb{E}(n_k)$ in \eqref{ecb6} can be relaxed:
\#\label{ecb7}
\mathbb{P}\Big(\frac{t_{0}\left(1+\log (t+1)-\log \left(t_{0}+1\right)\right)}{2K} \leq n_{k} \leq \frac{3 t_{0}\left(1+\log (t)-\log \left(t_{0}\right)\right)}{2K}\Big)  \geq 1-2\Big(\frac{t_{0}+1}{e(t+1)}\Big)^{\frac{t_{0}}{10K}}.
\#
When $t_0 = 2C_0K, C_0 \ge 10$, and $t\ge(t_0+1)^2/e^2-1$, the right-hand size of \eqref{ecb7} can be simplified as
\$
1-2\left(\frac{t_{0}+1}{e(t+1)}\right)^{\frac{t_{0}}{10K}} \geq 1-2\left(\frac{e \sqrt{t+1}}{e(t+1)}\right)^{C_{0} / 5} \geq 1-\frac{2}{t+1}.
\$
Therefore, we conclude the proof.
\end{proof}

\begin{lemma}[Lemma EC.8 of \citet{wang2018online}]\label{ap004}
Let $t_0 = 2C_0K$, $C_0 =\max\{10, 8/p_{\ast}\}$, and $t_l \ge (t_0 +1)^2/e^2 - 1$ (for $l \in [L]$). Under Assumptions \ref{armoptimality} and \ref{compatibilitycondition}, the following inequalities stay valid:
\begin{enumerate}
\item $\mathbb{P}\left(n_k<\frac{1}{2}C_0\log(t_l+1)~\text{or}~n_k>6C_0\log(t_l+1)\right)\le\frac{2}{t_l+1}$;
\item $\mathbb{P}\left(|\cA^{'}|<\frac{1}{2}p_{\ast}C_0\log(t_l+1)\right)\le\frac{1}{t_l+1}$;
\item  $\mathbb{P}\left(|\cA^{'}|/n_k<\frac{1}{12}p_{\ast}\right) \le\frac{3}{t_l+1}$.
\end{enumerate} 
\end{lemma}
\begin{proof}
From Lemma \ref{boundsrandom}, we have
\#\label{aped01}
\mathbb{P}\left(C_0(1+\log(t_l+1)-\log(t_0+1))\le n_k \le 3C_0(1+\log(t_l)-\log(t_0))\right)\ge 1-\frac{2}{t_l+1}.
\#
According to Assumption \ref{armoptimality}, for $k \in [K]$, there exists a region $\cV_k$ such that $\mathbb{P}(x\in \cV_k)\ge p_{\ast}$. Therefore, the expected number of  $\cA^{'}$ will be lower bounded by:
\#\label{aped02}
\mathbb{E}\big(\sum_{i=1}^{n_k}\mathbbm{1}(x_i\in \cV_k)\big) \ge p_{\ast} \EE(n_k) \ge 2 p_{\ast}C_0(1+\log(t_l+1)-\log(t_0+1)),
\#
where the last inequality comes from \eqref{ecb5}. Because $t_l\ge(t_0+1)^2/e^2-1$ indicates $\frac{1}{2}\log(t_l+1)\ge\log(t_0+1)-1$, we simplify \eqref{aped02} as follows:
\#\label{aped03}
\mathbb{E}\big(\sum_{i=1}^{n_k}\mathbbm{1}(x_i\in \cV_k)\big)\ge p_{\ast}C_0\log(t_l+1).
\#
Applying the Chernoff inequality to $\sum_{i=1}^{n_k}\mathbbm{1}(x_i\in \cV_k)$, we obtain:
\$
\mathbb{P}\Big(\sum_{i=1}^{n_k} \mathbbm{1}\left(\boldsymbol{x}_{i} \in \cV_k\right)<\frac{1}{2} \mathbb{E}\big(\sum_{i=1}^{n_k} \mathbbm{1}\left(\boldsymbol{x}_{i} \in \cV_k\right)\big)\Big) \leq \exp \Big(-\frac{1}{8} \mathbb{E}\big(\sum_{i=1}^{n} \mathbbm{1}\left(\boldsymbol{x}_{i} \in \cV_k\right)\big)\Big) .
\$
Using \eqref{aped03}, we deduce that:
\#\label{aped05}
\mathbb{P}\big(\sum_{i=1}^{n_k} \mathbbm{1}\left(\boldsymbol{x}_{i} \in \cV_k\right)<\frac{1}{2} p_{\ast} C_{0} \log (t_l+1)\big) \leq \exp \big(-\frac{1}{8} p^{*} C_{0} \log (t_l+1)\big).
\#
Because $C_0\ge 10$, we have $t_l \ge (t_0 +1)^2/e^2 - 1\ge e$. It follows that
\#\label{aped06}
\begin{aligned}
3C_0(1+\log(t_l)-\log(t_0))&\le3C_0(\log(t_l)+\log(t_l)-0)\\
&\le6C_0\log(t_l+1).
\end{aligned}
\#
From $\frac{1}{2}\log(t_l+1)\ge\log(t_0+1)-1$, we obtain:
\#\label{aped07}
C_0(1+\log(t_l+1)-\log(t_0+1))&=C_0(1+\frac{1}{2}\log(t_l+1)+\frac{1}{2}\log(t_l+1)-\log(t_0+1))\nonumber\\
&\ge C_0(1+\frac{1}{2}\log(t_l+1)-1)\nonumber\\
&=\frac{1}{2}C_0\log(t_l+1).
\#
Furthermore, from \eqref{aped06} and \eqref{aped07}, the following result is deduced:
\#\label{aped08}
&\big\{n_k<\frac{1}{2} C_{0} \log (t_l+1) \text { or } n_k>6 C_{0} \log (t_l+1)\big\} \nonumber\\
&\qquad=\big\{n_k<\frac{1}{2} C_{0} \log (t_l+1)\big\} \cup \big\{n_k>6 C_{0} \log (t_l+1)\big\} \nonumber\\
&\qquad \subseteq \left\{n_k<C_{0}\left(1+\log (t_l+1)-\log \left(t_{0}+1\right)\right)\right\} \cup\left\{n_k>3 C_{0}\left(1+\log (t_l)-\log \left(t_{0}\right)\right)\right\} \nonumber\\
&\qquad =\left\{n_k<C_{0}\left(1+\log (t_l+1)-\log \left(t_{0}+1\right)\right) \text { or } n_k>3 C_{0}\left(1+\log (t_l)-\log \left(t_{0}\right)\right)\right\}.
\#
By combining \eqref{aped08} and \eqref{aped01}, we obtain the following inequality:
\#\label{aped09}
&\mathbb{P}\big\{n_k<\frac{1}{2} C_{0} \log (t_l+1) \text { or } n_k>6 C_{0} \log (t_l+1)\big\} \nonumber\\
&\qquad \le\mathbb{P}\big\{n_k<C_{0}\left(1+\log (t_l+1)-\log \left(t_{0}+1\right)\right) \text { or } n_k>3 C_{0}\left(1+\log (t_l)-\log \left(t_{0}\right)\right)\big\}\nonumber\\
&\qquad \le\frac{2}{t_l+1}.
\#
Now, we prove the second part of Lemma \ref{ap004}. We only assume that for $x_k\in \cV_k, \Sigma_k=\mathbb{E}(XX^{\mathrm{T}} \,|\, X \in \cV_k) \in \mathcal{C}(\supp(\beta_k), \phi_0)$ in assumption \ref{compatibilitycondition}. By \eqref{aped05} and $C_0\ge 8/p_{\ast}$, it is easy to verify that:
\#
& \mathbb{P}\Bigl(|\cA^{'}|=\sum_{i=1}^{n_k} \mathbbm{1}\left(\boldsymbol{x}_{i} \in \cV_k\right)<\frac{1}{2} p^{*} C_{0} \log (t_l+1) \Bigr) \notag \\
& \qquad  \leq  \exp \big(-\frac{1}{8} p^{*} C_{0} \log (t_l+1)\big)\le\exp \left(-\log (t_l+1)\right) = \frac{1}{t_l+1}.\label{aped10}
\#
Finally, we will show the third part. Notice that the following result holds:
\#\label{aped11}
\Big\{|\cA^{'}|/n_k\ge\frac{1}{12}p_{\ast}\Big\} \supseteq& \Big\{|\cA^{'}|\ge\frac{1}{2}C_0p_{\ast}\log(t_l+1)\Big\}\cap\Big\{n_k\le6 C_{0} \log (t_l+1)\Big\}\nonumber\\
=&\Big(\Big\{|\cA^{'}|<\frac{1}{2}C_0p_{\ast}\log(t_l+1)\Big\}\cup\Big\{n_k >6 C_{0} \log (t_l+1)\Big\}\Big)^c.
\#
Thus we deduce that
\#\label{aped111}
\mathbb{P}\Big(|\cA^{'}|/n_k\ge\frac{1}{12}p_{\ast}\Big) &\ge \mathbb{P}\Big(\Big(\Big\{|\cA^{'}|<\frac{1}{2}C_0p_{\ast}\log(t_l+1)\Big\}\cup\Big\{n_k >6 C_{0} \log (t_l+1)\Big\}\Big)^c\Big)\nonumber\\
&=1-\mathbb P\Big(\Big\{|\cA^{'}|<\frac{1}{2}C_0p_{\ast}\log(t_l+1)\Big\}\cup\Big\{n_k >6 C_{0} \log (t_l+1)\Big\}\Big)\nonumber\\
&\ge1-\mathbb{P}\Big\{|\cA^{'}|<\frac{1}{2}C_0p_{\ast}\log(t_l+1)\Big\}-\mathbb{P}\left\{n_k>6 C_{0} \log (t_l+1)\right\}.
\#
Combining \eqref{aped09}, \eqref{aped10} and \eqref{aped111}, the following result is obtained:
\$
\mathbb{P}\Big(|\cA^{'}|/n_k<\frac{1}{12}p_{\ast}\Big) &\le ~\mathbb{P}\Big(|\cA^{'}|<\frac{1}{2}C_0p_{\ast}\log(t_l+1)\Big)+\mathbb{P}\Big(n>6 C_{0} \log (t_l+1)\Big)\\
&\le \frac{1}{t_l+1}+\frac{2}{t_l+1} =\frac{3}{t_l+1},
\$
which leads to the desired result.
\end{proof}
\begin{lemma}\label{lem:ap005}
For $\cA$ and $\cA^{'}$ defined as Lemma \ref{ap004}, if $|\cA|\ge C_0\log(t_l+1)/2$, the following inequality holds
\$
\PP\Big( \frac{|\cA^{'}|}{|\cA|}\ge\frac{p_{\ast}}{2} \Big)\ge1-\frac{2}{t_l+1}.
\$
\end{lemma}
\begin{proof}
The proof is adapted from Lemma EC.10 in \citet{bastani2015online}. By invoking Corollary A.1.14 in page 268 of \citet{alon2016probabilistic} for $\epsilon=1/2$ and $c_{\epsilon}$ (a version of Chernoff inequality), we let $y=\sum_{\tau\in \cA}\mathbbm{1}(\tau\in\cA^{'})$ and $\mu$ be its mean $\EE(y)$. Then, we get
\$
\PP\left( |y-\mu|>\mu/2\right)<2\exp(-0.1\mu).
\$
Combining this and
\$
\mu=\EE\big( \mathbbm{1}(\tau\in\cA^{'})\big)\ge p_{\ast}|\cA|,
\$
we have
\$
\PP\Big( |\cA^{'}|<\frac{p_{\ast}}{2}|\cA|\Big)< 2\exp\Big( \frac{p_{\ast}}{10}|\cA|\Big).
\$
Then, by using $|\cA|\ge C_0\log(t_l+1)/2$ and $C_0\ge 8/p_{\ast}$, we obtain that
\$
\PP\Big( |\cA^{'}|<\frac{p_{\ast}}{2}|\cA|\Big)< 2\exp\Big(\frac{p_{\ast}C_0\log(t_l+1)}{20}\Big)\le \frac{2}{t_l+1}.
\$
Therefore, we conclude the proof.
\end{proof}

\section{Proofs of Auxiliary Results}

\subsection{Margin Condition Implies Arm Optimality Condition}\label{pfoptimality}
Recall that when the margin condition (Assumption \ref{margincondition}) holds, there exists a constant $C_m\ge 0$ satisfying that $\mathbb{P}\left( 0 \le \,|\, X^{\mathrm{T}}(\beta_k^{\true} - \beta_j^{\true})\,|\, \le \gamma \right) \le C_m R_{\max}\gamma$ for any $k \ne j$. For readability, the definition of the sub-optimal set and the optimal set in Assumption \ref{armoptimality} are shown as follows: 
\$
&\mathcal{K}_s = \{k \in [K] \,|\, x^{\mathrm{T}}\beta_k < \max_{j \ne k} x^{\mathrm{T}}\beta_j - h ,~\text{a.e.}~ \bx \in \mathcal{X} \},\\
&\mathcal{K}_o = \{ k \in [K] \,|\, \exists ~\cU_k \subseteq \mathcal X ~\text{s.t.}~\P(x \in \cU_k) > p_{\ast}~\text{and}~ x^{\mathrm{T}}\beta_k > \max_{j \ne k} x^{\mathrm{T}}\beta_j + h ~\text{for}~ x \in \cU_k \}.
\$
It is obvious from the formulation that $\cK_s\cap\cK_o=\emptyset$. Thus, if we can prove that when the margin condition holds and $C_m$ is small enough, $\cK_o \cup \cK_s = \cK$, we can claim that the margin condition implies the arm optimality condition in some cases.
\begin{lemma}\label{lm:pfoptimality}
Under Assumption \ref{margincondition} with $0< C_m \le p_1/(h R_{\max})$, we have $\cK_o \cup \cK_s = [K]$.
\end{lemma}
\begin{proof}
For any arm $k\in[K]\setminus \cK_s$, there exists a region $\cB_k$ such that $x^{\mathrm{T}}\beta_k \ge \max_{j \ne k} x^{\mathrm{T}}\beta_j - h$ for any $x \in \cB_k$ and $\PP(\bx \in \cB_k) > 0$. Then, define the positive constant $p_1 = \min_{k \notin \cK_s} \mathbb P(\bx \in \cB_k)$. Now it suffices to prove that there exists a positive measure subset of $\cB_k$ such that $x^{\mathrm{T}}\beta_k > \max_{j \ne k} x^{\mathrm{T}}\beta_j+h$ for $x$ in this subset. If this does not hold, there exists an arm $k_0\in[K]\setminus \cK_s$ such that $x^{\mathrm{T}}\beta_{k_0} \le \max_{j \ne k_0} x^{\mathrm{T}}\beta_j + h$ for a.e. $x \in \cB_{k_0}$. Therefore, there always exists an arm $j_0 = \argmax_{j\ne k_0}x^{\mathrm{T}}\beta_j$ such that $\,|\, x^{\mathrm{T}} \beta_{k_0} - x^{\mathrm{T}}\beta_{j_0} \,|\, \le h$. Now, we let $C_m \le p_1/(h R_{\max})$ and deduce that
$$
\P \left( \{ 0 < \,|\, x^{\mathrm{T}}(\beta_{k_0} - \beta_{j_0})\,|\, \le h \} \right) \ge \P(x \in \cB_{k_0}) \ge p_1 > C_m R_{\max}h,
$$
which contradicts Assumption \ref{margincondition}, since for any $x \in \cX$, we have 
\$
\mathbb{P}\left( \{ 0 \le \,|\, x^{\mathrm{T}}(\beta_{k_0} - \beta_{j_0})\,|\, \le h \} \right) &\le C_m R_{\max}h.
\$
Therefore, we conclude the proof.
\end{proof}

\section{Experimental Details}\label{s:Experimental_Details}
This section contains the details about the experiment: dynamic pricing in retail data, including the description of the dataset and the pseudocode for the dynamic pricing algorithm.

\vspace{4pt}
\noindent
\textbf{Dataset.} We choose the fulfillment centers in the largest cities with sample size of $35084$. At each time step $t$, an order arrives with order-specific contexts $X_t$ of dimension $22$, including dummy variables characterizing the meat category and cusine, the indicator of promotions, the base price (normalized) and an intercept. The decision made by the centers was the checkout price.

\vspace{4pt}
\noindent
\textbf{Algorithm.} To propose the BIASEX algorithm in Algorithm \ref{algo:price}, we combined the batched sparse bandit with the ILSX algorithmic design \citep{ban2021personalized}. To commit to exploration, we fixed two experimental prices ($200$ and $600$), and at the forced-sampling time step, the price was uniformly sampled from the two values. During other time steps, the price was charged greedily to optimize the revenue as follows:
\#\label{eq:price_optimal}
\hat p_t = \Big[\frac{X_t^\top \betah_0}{-2X_t^\top \betah_1}\Big]_{[p_{\min},p_{\max}]},
\#
where we set the minimum and maximum prices $p_{\min}=0$ and $p_{\max}=1000$. The expression is truncated at $p_{\min}$ and $p_{\max}$ if the result exceeds the boundaries.

\begin{algorithm}[th]
\caption{Batched High-dimensional Sparse Bandit with Price Experimentation (BIASEX)}\label{algo:price}
\begin{spacing}{1.35}
\begin{algorithmic}[1]
\STATE \textbf{Require:} Input parameters time $T$, number of batches $L$, $t_0$, $\lambda_{0}$ and $a$ 
\STATE \hspace{0.15in} Choose grid $\mathcal{T}=\{ t_1,\ldots,t_L\}$ in \eqref{eqgrid}
\STATE \hspace{0.15in} Initialize $\betah = \mathbf{0}$
\STATE \hspace{0.15in} \textbf{For} batch $l=1,\ldots,L$,~ \textbf{do}
\STATE \hspace{0.3in} \textbf{For} time $t=t_{l-1}+1,\ldots,t_l$ $\textbf{do}$
\STATE \hspace{0.45in} Observe $x_t$
\STATE \hspace{0.45in} Draw a binary random variable $\mathcal{D}_t$, where $\mathcal{D}_t=1$ with probability $\min\{1, t_0/t\}$
\STATE \hspace{0.45in} $\mathbf{If}$ $\mathcal{D}_t=1$
\STATE \hspace{0.6in} Assign price $p_t$ to a random decision in $\{200, 600\}$ with equal probability
\STATE \hspace{0.45in} $\mathbf{Else}$
\STATE \hspace{0.6in} Charge price $\hat p_t$ in \eqref{eq:price_optimal}
\STATE \hspace{0.45in} $\mathbf{End~If}$
\STATE \hspace{0.3in} $\mathbf{End~For}$
\STATE \hspace{0.3in} Compute regularization parameter $\lambda_{l}=\lambda_{0} t_l^{1/4} \sqrt{\log t_l+\log d}$
\STATE \hspace{0.3in} Observe demands $\{Y_t = X_t^\top \beta_0 + p_t\cdot X_t^\top \beta_1 + \epsilon_t\}_{t=t_{l-1}+1}^{t_l}$ and update LASSO estimator $\betah(\lambda_{l})$
\STATE \hspace{0.15in} $\mathbf{End~For}$
\end{algorithmic}
\end{spacing}
\end{algorithm}

\section{Technical Lemmas}
To begin with, we introduce the following Bernstein Concentration for non-i.i.d. random vectors.
\begin{lemma}[Vector Bernstein]\label{lemi05}
Let $\{ D_t, \mfS_t \}_{t=1}^{\infty}$ be a martingale difference sequence, and suppose that $D_t$ is $\sigma$-subgaussian in an adapted sense, i.e., for all $\alpha \in \RR$, $\E\left( e^{\alpha D_t} \mid \mfS_{t-1} \right) \le e^{\alpha^2\sigma^2/2}$ almost surely. Then, for all $s>0$, $\P\left( \,|\, \sum_{t=1}^n D_t \,|\, \ge s \right) \le 2\exp\left( -s^2/(2n\sigma^2) \right)$.
\end{lemma}
This lemma is from Theorem 2.19 of \cite{wainwright2019high} when $\alpha_* = \alpha_k = 0$ and $\nu_k = \sigma$ for all $k$.

Then, we present the Bernstein Inequality for i.i.d. random matrices as follows.
\begin{lemma}[Matrix Bernstein]\label{lm:Matrix_Bernstein}
Consider an independent zero-mean sequence of random matrices $\{X_t\}_{t=1}^n$ of dimension $d_1\times d_2$. If $\rho^2=\max\{\|\sum_{t=1}^n\E(\bX_t\bX_t^{\top})\|_{\op}, \|\sum_{t=1}^n\E(\bX_t^{\top}\bX_t)\|_{\op}\}$ and $\|\bX_k\|_{\op}\le M$ almost surely for all $t\in[n]$, then for any $\tau>0$, we have
$$
\P\bigg( \Big\|\sum_{t=1}^n\bX_t\Big\|_{\op} > \tau\bigg) \le (d_1+d_2)\cdot\exp\Big(\frac{-\tau^2/2}{\rho^2 + M\tau/3}\Big).
$$
\end{lemma}
The proof is provided in Theorem 1.6 of \citet{tropp2011freedman}.

\begin{lemma}[Matrix Sub-Gaussian Series]\label{lm:Matrix_Sub-Gaussian_Series}
Consider a sequence of matrix $\{\bA_t\}_{t=1}^{\infty}$ with dimension $d_1\times d_2$ and a martingale difference sequnece $\{\gamma_t,\mfS_t\}_{t=1}^{\infty}$, where $\gamma_t$ is conditional $\sigma$-subgaussian (i.e., $\E(e^{\alpha\gamma_t} \,|\, \bA_t, \mfS_t) \le e^{\alpha^2\sigma^2/2}$ almost surely for all $\alpha\in\RR$), and $\mfS_t=\{\bA_s,\gamma_s\}_{s=1}^{t-1}$. Define the matrix sub-Gaussian series
$\bZ = \sum_{t=1}^n \gamma_t \bA_t$ with bounded matrix variance statistic:
$$
\max\left\{\left\|\bA_t\bA_t^{\top}\right\|_{op},\left\|\bA_t^{\top}\bA_t\right\|_{op}\right\}\le v_t.
$$
Then, for all $u>0$, we have 
$$
\P\left(\|\bZ\|_{op} \ge u\right) \le (d_1+d_2)\exp\Big(-\frac{u^2}{16\sigma^2\sum_{t=1}^n v_t}\Big).
$$
\end{lemma}
\begin{proof}
Define a sequence of matrix with dimension $(d_1+d_2)\times(d_1+d_2)$:
$$
\bB_t = \left[
\begin{array}{cc}
    0 & \bA_t \\
    \bA_t^{\top} & 0
\end{array}
\right],\quad t=1,\ldots,n.
$$
Then, it is obvious that the new series $Y = \sum_{t=1}^n\gamma_t\bB_t$ has the same bounded matrix variance statistic as $Z$:
$$
\left\|\bB_t^2\right\|_{op} = \max\left\{\left\|\bA_t\bA_t^{\top}\right\|_{op},\left\|\bA_t^{\top}\bA_t\right\|_{op}\right\}\le v_t.
$$
Additionally, $\bB_t$ preserves spectral information:
$$
\lambda_{\max}(\bY) = \lambda_{\max}(\bZ) = \|\bZ\|_{op}.
$$

By invoking Proposition 3.2.1 in \citet{tropp2011freedman}, we get
\#\label{eq:prop_3.2.1}
\P\left(\lambda_{\max}(\bY)\ge u\right) \le \inf_{\theta>0}e^{-\theta u}\E\tr \exp\left(\theta \bY\right).
\#
Then, applying iterated expectations due to the tower property of conditional expectation, we obtain
\$
\E\tr \exp\left(\theta \bY\right) &= \E\tr \exp\big(\sum_{t=1}^n \theta \gamma_t \bB_t\big)\\
&= \E\Big[\E\Big(\tr \exp\big(\sum_{t=1}^{n-1} \theta \gamma_t \bB_t + \theta \gamma_n \bB_n\big)\,\bigg|\, \bA_n, \mfS_n\Big)\Big]\\
&\le \E\Big[\tr \Big(\exp\big(\sum_{t=1}^{n-1} \theta \gamma_t \bB_t\big)\cdot \E\left(e^{\theta \gamma_n \bB_n}\,|\,A_n,\mfS_n\right)\Big)\Big],
\$
where the inequality is deduced since the trace function is concave. We then use Lemma \ref{lm:subgaussian_moment} to get
\$
\E\tr \exp\left(\theta \bY\right) 
&\le \E\Big[\tr \Big(\exp\big(\sum_{t=1}^{n-1} \theta \gamma_t \bB_t\big)\cdot e^{4\sigma^2\theta^2 \bB_n^2}\Big)\Big]\\
&\le \E\Big[\tr \Big(\exp\big(\sum_{t=1}^{n-1} \theta \gamma_t \bB_t\big)\Big)\Big]\cdot e^{4\sigma^2\theta^2 v_n},
\$
where and the last inequality is because of the Cauchy-Schwartz inequality:
$$
\tr \bigg(\exp\Big(\sum_{t=1}^{n-1} \theta \gamma_t \bB_t\Big)\cdot e^{4\sigma^2\theta^2 \bB_n^2}\bigg) \le \tr \bigg(\exp\Big(\sum_{t=1}^{n-1} \theta \gamma_t \bB_t\Big)\bigg)\cdot \|e^{4\sigma^2\theta^2\bB_n^2}\|_{op}.
$$
Repeating this step leads to the following result:
$$
\E\tr \exp\left(\theta \bY\right) \le (d_1+d_2)\exp\Big(4\sigma^2\theta^2\sum_{t=1}^nv_t\Big).
$$
The result above can be taken back into \eqref{eq:prop_3.2.1} to derive that
\$
\P\left(\lambda_{\max}(\bY)\ge u\right) &\le \inf_{\theta>0}e^{-\theta u}(d_1+d_2)\exp\Big(4\sigma^2\theta^2\sum_{t=1}^nv_t\Big)\\
&= (d_1+d_2)\inf_{\theta>0}\exp\Big(-\theta u + 4\sigma^2\theta^2\sum_{t=1}^nv_t\Big)\\
&= (d_1+d_2)\exp\Big(-\frac{u^2}{16\sigma^2\sum_{t=1}^nv_t}\Big),
\$
which implies that
$$
\Pb\left(\lambda_{\max}(\bZ)\ge u\right) \le (d_1+d_2)\exp\Big(-\frac{u^2}{16\sigma^2\sum_{t=1}^nv_t}\Big).
$$
\end{proof}

\begin{lemma}\label{lm:subgaussian_moment}
Given a fixed symmetric matrix $A$ and a $\sigma$-subgaussian variable $\gamma$, we have for any $\theta>0$,
$$
\E e^{\gamma\theta \bA} \preceq e^{4\sigma^2\theta^2 \bA^2}.
$$
\end{lemma}
\begin{proof}
By invoking Proposition 3.2 in \citet{rivasplata2012subgaussian}, we know that a $\sigma$-subgaussian variable $\gamma$ has bounded moments:
\#\label{eq:gamma_moment}
\E|\gamma|^k \le (2\sigma^2)^{k/2}k\Gamma(k/2).
\#
We may assume $\theta=1$ by absorbing $\theta$ into the matrix $\bA$. It follows that
\$
\E e^{\gamma \bA} &= I + \sum_{k=1}^{\infty}\frac{\E(\gamma^k)}{k!}\bA^k \preceq I + \sum_{k=2}^{\infty}\frac{(2\sigma^2\bA^2)^{k/2}k\Gamma(k/2)}{k!}\\
&= I + \sum_{k=1}^{\infty} \frac{(2\sigma^2\bA^2)^k2k\Gamma(k)}{(2k)!} + \sum_{k=1}^{\infty} \frac{(2\sigma^2\bA^2)^{k+1/2}(2k+1)\Gamma(k+1/2)}{(2k+1)!}.
\$
By using 
$$
\frac{\Gamma(k+1/2)}{k!} = \frac{(2k)!\sqrt{\pi}}{(k!)^24^k} \le 1,
$$
we obtain that
\$
\E e^{\gamma \bA} 
&\preceq I + \left(2I + (2\sigma^2\bA^2)^{1/2}\right) \sum_{k=1}^{\infty} \frac{(2\sigma^2\bA^2)^kk!}{(2k)!}\\
&\preceq I + \left(I + (\sigma^2\bA^2/2)^{1/2}\right) \sum_{k=1}^{\infty} \frac{(2\sigma^2\bA^2)^k}{k!}\\
&\preceq e^{2\sigma^2\bA^2} + (\sigma^2\bA^2/2)^{1/2}\cdot(e^{2\sigma^2\bA^2} - 1) \preceq e^{4\sigma^2\bA^2},
\$
where the second inequality is due to $(2k)!\ge 2(k!)^2$.
\end{proof}

\begin{lemma}\label{lem:monint}
If function $f$ is monotone increasing and integrable on the interception $[s_1-1,s_2]$, then
\$
\sum_{t=s_1}^{s_2}f(t)\ge\int_{s_1-1}^{s_2}f(t)\md t.
\$
\end{lemma}
\begin{proof}
By direct calculation, we have 
\$
\sum_{t=s_1}^{s_2}f(t)\ge\sum_{t=s_1}^{s_2}\int_{t-1}^{t}f(\tau)\md\tau=\int_{s_1-1}^{s_2}f(t)\md t.
\$
\end{proof}

\end{appendix}
\end{document}